\documentclass[opre]{informs3}
\OneAndAHalfSpacedXI



\usepackage{amsmath,amssymb,enumerate,bm,xcolor,multirow,pbox}
\usepackage{graphicx,color,framed,tikz,caption,subcaption}
\usepackage{enumitem}
\usepackage{algorithm}
\usepackage[noend]{algpseudocode}
\usepackage{natbib}
\usepackage[colorlinks=true,linkcolor=blue,citecolor=blue,urlcolor=blue]{hyperref}
\usepackage{booktabs}
\usepackage{tabularray}
\usepackage{pgfplots}

\usepackage{microtype}
\usepackage{mathtools,booktabs}
\usepackage{xcolor}
\usepackage{tabularx,threeparttable}
\usepackage{comment}

\usepackage{natbib}
 \bibpunct[, ]{(}{)}{,}{a}{}{,}%
 \def\bibfont{\small}%
\EquationsNumberedThrough    

\TheoremsNumberedThrough     
\ECRepeatTheorems  %

\newcommand{\E}{\mathbb{E}}
\newcommand{\Prob}{\mathbb{P}}
\newcommand{\R}{\mathbb{R}}

\newcommand{\1}{\mathbf{1}}
\newcommand{\F}{\mathcal{F}}
\newcommand{\norm}[1]{\left\lVert #1 \right\rVert}
\newcommand{\abs}[1]{\left| #1 \right|}

\newcommand{\gr}{\mathrm{gr}}
\newcommand{\ex}{\mathrm{ex}}
\newcommand{\pop}{\mathrm{pop}}

\newcommand{\Var}{\mbox{Var}}

\renewcommand{\P}{\mathbb P}
\newcommand{\cF}{\mathcal F}

\newcommand{\bp}{\bm p}

\renewcommand{\qquad}{\quad}
\renewcommand{\subseteq}{\subset}
\renewcommand{\star}{\ast}

\begin{document}


\RUNAUTHOR{Mei, Fan, Leng and Lv}

\RUNTITLE{
CART-ROSA}

\TITLE{CART Random Forests as Sequential Allocation over Random Opportunity Sets: A Stochastic-Control Theory of Ensemble Risk}

\ARTICLEAUTHORS{%
\AUTHOR{Tianxing Mei}
\AFF{Faculty of Business, Lingnan University, \EMAIL{tianxingmei@ln.edu.hk}}

\AUTHOR{Yingying Fan}
\AFF{Data Sciences and Operations Department, University of Southern California, \EMAIL{fanyingy@marshall.usc.edu}}

\AUTHOR{Mingming Leng}
\AFF{Faculty of Business, Lingnan University, \EMAIL{mmleng@ln.edu.hk}}

\AUTHOR{Jinchi Lv}
\AFF{Data Sciences and Operations Department, University of Southern California, \EMAIL{jinchilv@marshall.usc.edu}}
} 


\ABSTRACT{
CART random forests are among the most widely used modern predictive methods, with well-documented empirical success. Yet, at the mechanistic level, the algorithm is often treated as a black box because of its complexity. 
In this paper, we develop a stochastic-control perspective on feature-subsampled CART random forests, named CART random opportunity-set allocation (CART-ROSA). At each node, the random subset of features is interpreted as a random feasible action set, and the CART split rule as a masked-action allocation policy. 
This policy induces a controlled stochastic process over informative split-count states, 
whose terminal law determines both single-tree error and cross-tree interaction terms in the forest mean squared error (MSE). 
Such representation opens the black box of CART-forests by separating two design levers: the informative-opportunity rate induced by feature subsampling, and the contraction strength from the within-mask split policy. 
We establish that the CART policy is locally stabilizing: it contracts imbalances in informative split allocations and concentrates terminal tree geometry. 
At the system level, however, it can be globally suboptimal for the forest objective. 
Specializing to the linear model, we derive the MSE risk expansion explicitly. Our results show how an operations-research perspective makes tractable a theoretical gap difficult to access from the standard algorithmic description of CART forests.}





\maketitle


\section{Introduction} \label{new.sec.intro}
Random forests are among the most successful nonparametric regression algorithms in supervised learning.
Since Breiman's original formulation in \cite{breiman2001random}, forests built from randomized CART (classification and regression trees) have become standard tools in scientific, business, and engineering applications for classification and regression, owing to their accuracy, robustness, ease of tuning, natural parallelizability, and effectiveness in high-dimensional problems. 
Yet the empirical success of CART random forests has \textit{not} been matched by a complete \textit{mechanism-level} theory within Breiman's framework.
The difficulty is intrinsic to the structure of CART random forests: a forest is not a single estimator with a fixed analytical form, but a randomized sequential construction.
At each node of each tree, a random subset of features is sampled.
From these features, CART chooses greedily a split according to a local impurity criterion. Such process changes recursively the geometry of all descendant cells, while the predictive performance is evaluated only after many such randomized trees are aggregated. As a result, the forest mean squared error (MSE) risk is shaped not only by the quality of individual splits, but also by the distribution over terminal tree geometries induced by the entire randomized construction rule. 

For these reasons, CART random forests remains largely a black box.
In particular, a mechanism-level theory of the split-allocation process itself is still \textit{missing}: how random feature subsampling and greedy CART decisions jointly generate the terminal tree geometry, and how that terminal geometry determines the ensemble-level MSE risk.
Existing analyses have made major progress on the consistency, asymptotic behavior, splitting rules, forest weighting, and modified or simplified forest procedures; see Section \ref{new.sec.relawork} for the related literature and references. These perspectives, however, do not directly characterize the mechanism-level interaction between feature subsampling, CART split choices, recursive state evolution, and the forest MSE objective.

In this paper, we take an initial step toward filling such theoretical gap from an \textit{operations-research} perspective.
We de-blackbox CART random forests by identifying the state, action mask, policy, transition, terminal law, and system objective that are implicit in the standard algorithmic description.
Specifically, we formulate the feature-subsampled CART random forests construction as a finite-horizon stochastic control problem with random feasible action sets, a perspective named as CART random opportunity-set allocation (CART-ROSA). At each node, feature subsampling generates a mask that determines the feasible actions.
The split rule is a policy that selects one feasible action.
The state records the accumulated allocation of splits across informative coordinates (i.e., those with nonzero impurity decrease).
The transition updates this state according to the selected split, and the terminal law of the resulting controlled process determines the statistical risk of the forest.
Under such representation, CART is no longer merely a heuristic split criterion; it is a masked-action policy whose induced terminal distribution can be studied, compared, and improved.
Our \textit{stochastic-control} viewpoint makes the mechanism analyzable and allows us to distinguish local split improvement from global ensemble-level MSE performance. 

Our theoretical analysis focuses on a sparse nonparametric regression model with independent and uniformly distributed covariates, a model that has been used widely in theoretical studies of CART random forests \citep{biau12a, klusowski2019analyzing, klusowski21b}.
We consider a simplified setting in which, once CART chooses a covariate, the split location is fixed at the midpoint; we call this the \textit{midpoint split model}.
Together, this model and simplification allow us to isolate a central mechanism in CART random forests: how informative split opportunities are allocated across informative coordinates under feature subsampling. 

Our proposed CART-ROSA formulation reveals two \textit{distinct} design levers.
The first one is the opportunity.
Feature subsampling controls the probability
$
q 
$
that at least one informative coordinate is available at a node.
This lever determines how often the policy can make progress along informative directions that yield impurity decrease.
The second one is the exploitation.
Conditional on the appearance of an informative opportunity, the within-mask split policy controls how aggressively the process allocates splits to currently under-split informative coordinates.
We summarize this second lever through a contraction parameter, denoted as $\kappa$, that captures the stabilizing force of the split rule.
This separation of opportunity and exploitation is difficult to discern from the standard random forests algorithmic description, but becomes natural in our stochastic-control representation.

Our first main result reveals that the population CART has a local stabilizing effect.
In the sparse nonparametric regression midpoint model, when the informative coordinates are exposed, the CART policy favors coordinates with smaller current split counts, up to deterministic signal-strength shifts.
Hence, greedy CART acts as a balancing policy for informative split allocation.
We prove that this policy contracts imbalance in the split-count process and concentrates the terminal geometry of individual trees.
Such result gives a mechanism-level explanation for one beneficial role of CART: 
beyond preferring large impurity decreases, 
CART also stabilizes the allocation of informative splits across the tree-building process.

The same formulation, however, also reveals a fundamental limitation. The ensemble should optimize a global objective, such as the prediction MSE, which does not need to align with the one-step CART impurity decrease.
Forest MSE depends on the entire terminal law of the controlled process.
In particular, it contains both single-tree terms and cross-tree interaction terms, as confirmed by our MSE decomposition analysis under the specialized linear midpoint models.
A policy that stabilizes individual-tree geometry may concentrate the terminal law in a way that reduces diversity or creates unfavorable cross-tree interactions. Thus, local CART stabilization does not necessarily imply global optimality for the forest.
Our second main result formalizes such distinction.
We derive a Bellman-type marginal-cost certificate for the forest terminal-law objective.
This certificate evaluates whether a feasible one-step policy perturbation decreases the induced forest MSE.
Using such certificate, we construct an explicit \textit{counterexample} in which greedy CART selects the locally stabilizing split, whereas an alternative feasible action has strictly smaller marginal terminal cost.
A small policy \textit{perturbation} thus strictly improves the forest MSE. Consequently, greedy CART is not a local minimizer, and thus not globally optimal for the forest-level objective.

Our findings should \textit{not} be interpreted as a criticism of CART random forests as a practical method.
Instead, they provide a mechanism-level de-blackboxing of CART random forests.
Greedy CART has a stabilizing effect, which helps explain why the method is effective. The same local logic, however, does not fully align with the ensemble objective.
The stochastic-control view makes this distinction explicit: CART is locally reasonable as a split-allocation policy, yet globally improvable as a policy for the terminal forest objective.
In summary, our paper makes three major contributions:
\begin{itemize}
    \item We introduce a masked-action stochastic-control formulation of feature-subsampled CART random forests (CART-ROSA), 
which provides a mechanism-level representation of the method. 

\item We use this formulation to prove a local stabilization theorem for the population CART.

\item We show that local stabilization does not imply ensemble-level optimality. 
\end{itemize}

Together, these results demonstrate that an operations-research perspective can fill a theoretical gap in understanding the mechanism of CART random forests.

\subsection{Related works} \label{new.sec.relawork}


A substantial portion of the literature on de-blackboxing random forests focuses on simplified variations of CART random forests that yield analytically tractable stochastic partitioning dynamics.
Representative examples include the purely random forests
\citep{biau12a,klusowski21b},
median forests \citep{Scornet2016},
Mondrian forests \citep{Mourtada2020,cattaneo2024inferencemondrianrandomforests},
and honest forests \citep{Wager2018,Athey2019}.
At a conceptual level, these works, however, focus primarily on the statistical accuracy of the resulting methods and do \textit{not} provide a systematic mechanism-level understanding of the roles of feature subsampling, the split rule, and the ensemble.   

Our stochastic-control formulation is related in spirit to the stochastic approximation literature on recursive controlled stochastic systems \citep{RobbinsSiegmund1971,kushneryin2003,Borkar2023}.
Such viewpoint connects naturally to the literature on the purely random forests and random partitioning estimators \citep{biau12a,klusowski21b}. 
This connection, however, has not been explored in full depth, even for these stylized variations of CART random forests.
Our state-action formulation is also related in spirit to sequential stochastic decision systems \citep{powell2007approximate,bertsekas2012dynamic}.
From the viewpoint of the directional counting dynamics, the subsampling ratio $\gamma$ in CART random forests controls simultaneously the informative-opportunity probability and the stabilization strength of the induced allocation dynamics. 
This leads naturally to an \textit{exploration--exploitation} interpretation of the sequential split-allocation mechanism, a distinction that also arises broadly in sequential stochastic decision systems and adaptive allocation problems \citep{powell2007approximate}. 
Despite these connections, our formulation and the derived results, particularly the self-stabilizing fluctuation behavior and terminal geometry, are \textit{new} to the random forests literature.
{\color{blue}
Our proposed \textit{$\alpha$--mixture extension}}
further separates the informative-opportunity generation from the contraction strength, yielding a \textit{continuous} interpolation between fully greedy stabilization and purely random splitting.

Another major direction, which is less directly related to our work, focuses on empirical understanding of the roles of feature subsampling and other randomization mechanisms in the success of tree ensembles by studying their predictive and regularization effects.
Existing works have primarily interpreted feature subsampling through its statistical regularization and decorrelation effects \citep{MentchZhou2020,liu2024randomizationreducebiasvariance}, or analyzed random forests through adaptive nearest-neighbor and smoothing perspectives \citep{Lin2006,curth2024randomforestsworkunderstanding}.
The effects of additional tuning parameters, including the tree depth, forest size, and noise level, have also been extensively investigated \citep{delgado14a,Probst2018,le2023survey,bernard2009influence,zhou2022random,zhou2023trees}.

On the theory side, recent work has studied the statistical estimation and prediction accuracy of random forests estimators under structured assumptions, including one-split stumps \citep{Buehlmann2002}, sparse additive regression models \citep{Scornet2015,Klusowski2024}, binary-feature settings \citep{syrgkanis20a}, and high-dimensional nonlinear models satisfying the sufficient impurity decrease (SID) property \citep{Fan22AOS} or merged staircase property \citep{tan2024statcomp}.
Recently, \cite{MFL2024} developed a non-asymptotic MSE framework for partitioning estimation ensembles under exogenous partition randomness, clarifying how exogenous randomness influences ensemble performance through the induced bias--variance structure in statistical prediction accuracy.

The remainder of the paper is organized as follows.
Section~\ref{sec:model} introduces the model and the masked-action policy formulation.
Section~\ref{sec:results} establishes the local stabilization results for the population CART.
Section~\ref{sec:policyfamilies} studies benchmark and mixture policies on the same masked environment.
Section~\ref{sec:stat} translates the count-process analysis into exact risk comparisons for single trees and forests.
Section~\ref{sec:practical} develops the resulting design map and its implications for randomized tree ensembles. Section \ref{new.sec.discu} concludes with some discussions. All the proofs and additional technical details are provided in the Supplementary Material.

\section{Model and masked-action allocation formulation} \label{sec:model}

We consider a sparse regression problem with response $Y\in\mathbb R$ and covariate vector $
X=(X_1,\dots,X_d)^\top\in C_0:=[0,1]^d.
$
Write
$
\mu(x):=\E[Y\mid X=x]
$
for the mean regression function. Assume that there exists a unique smallest \textit{informative covariate set}
$$
S\subseteq [d]:=\{1,\dots,d\},
\qquad s:=|S|
$$
such that $\mu(X)=\mu(X_S)$ almost surely. The complement $S^c$ represents the noninformative set. At each split, feature subsampling selects randomly 
$
m:=\lceil \gamma d\rceil
$
candidate coordinates, where $\gamma\in(0,1]$ is the feature subsampling rate, $s<m\leq d$, and $\lceil x\rceil$ stands for the smallest integer greater than or equal to $x>0$. We study tree construction along a single root-to-leaf branch of depth $\ell$ as a finite-horizon \textit{masked-action allocation} problem. 
Such formulation separates the opportunity generation from the opportunity allocation, and makes explicit the decision process that is implicit in the CART random forests (RF) algorithm. Table~\ref{tab:or-dictionary} summarizes this dictionary.

\begin{table}[ht]
\TABLE
{Dictionary for feature-subsampled CART random forests (RF).
\label{tab:or-dictionary}}
{
\small
\begin{tabularx}{\textwidth}{
>{\raggedright\arraybackslash}p{0.2\textwidth}
>{\raggedright\arraybackslash}p{0.32\textwidth}
>{\raggedright\arraybackslash}X}
\hline
\up
CART--RF component
& OR object
& Role in this paper \\
\hline
\up
Feature subsampling
& Random opportunity set / action mask $U_t$
& Determines which split coordinates are feasible at depth $t$ \\
Coordinate split choice
& Allocation action $J_t\in U_t$
& The policy decision made after the opportunity set is revealed \\
Recursive branch growth
& Controlled count-state process $N_t$
& Encodes the evolving geometry of a root-to-leaf branch \\
Population CART rule
& Greedy masked-action policy $\pi_{\gr}$
& Maximizes the local impurity-decrease among feasible actions \\
Forest prediction risk
& Terminal-law ensemble objective
& Evaluates the distribution of terminal cell geometries induced by the policy \down\\
\hline
\end{tabularx}
}
{}
\end{table}

\subsection{Feature-subsampled population CART as masked-action decision process} \label{new.sec.cart}

We now formalize the first four rows of Table~\ref{tab:or-dictionary}. 
Consider a single branch in a tree grown recursively by the population CART splitting rule from the root cell $C_0=[0,1]^d$. 
At tree depth $t$, feature subsampling selects randomly a candidate set
$$
U_t\subseteq [d],
\qquad
|U_t|=m
$$
drawn uniformly without replacement and independently across depths. Such random candidate set is the action mask. Let $C_{t-1}\subset C_0$ be the current cell at depth $t-1$, and 
$
J_t\in U_t
$
the coordinate selected for the split at depth $t$. Conditional on the current cell $C_{t-1}=\prod_{k=1}^d(a_{t-1,k},b_{t-1,k}) $ and  realized mask $U_t$, the population CART impurity decrease by splitting coordinate $j$ is
$$
\mathcal G(C_{t-1},j)
:=
\sup_{c\in(a_{t-1,j},b_{t-1,j})}\mathrm{IMD}_{j,c}(C_{t-1}),
$$
where for a current cell $C=\prod_{k=1}^d(a_k,b_k)$, the impurity decrease for splitting coordinate $j$ at threshold $c \in (a_j,b_j)$ is given by 
$$
\begin{aligned}
\mathrm{IMD}_{j,c}(C)
:={}&
\Var(Y\mid X\in C)
-\Var(Y\mid X\in C_{j,c;L})\Prob(X\in C_{j,c;L}\mid X\in C)\\
&-\Var(Y\mid X\in C_{j,c;R})\Prob(X\in C_{j,c;R}\mid X\in C).
\end{aligned}
$$
Here, $C_{j,c;L}$ and $C_{j,c;R}$ represent the left and right children obtained by splitting $C$ along coordinate $j$ at threshold $c$. 

Applying the law of total variance to the binary partition $C = C_{j,c; L} \cup C_{j,c; R}$, 
the impurity decrease can be rewritten as
$$
\mathrm{IMD}_{j,c}(C)
=
\Var\!\big(\E[Y\mid X\in C_{j,c;\cdot}]\mid X\in C\big).
$$
That is, the impurity decrease measures the variation between the conditional means of the two child cells.
Hence, the CART splitting criterion 
chooses the split that maximizes the amount of signal
revealed at the current partition. Since $\mu(X)=\mu(X_S)$ and $X$ is uniformly distributed on $[0,1]^d$, splitting a noninformative coordinate does not reduce impurity; that is,
$
\mathrm{IMD}_{j,c}(C)=0$,
$
j\in S^c.
$

Let us focus on the coordinate-selection process $\{J_t\}$ induced by the population CART split rule
\begin{equation}\label{eq:cart-pop-rule-general}
J_t\in A_{\gr}(C_{t-1},U_t)
:=
\argmax_{j\in U_t}\mathcal G(C_{t-1},j),
\end{equation}
where ties are broken randomly. 
The split-count vector is defined as
$$
N_t = (N_{t,1},\ldots, N_{t,d})^\top \in \mathbb{N}_0^d, \qquad 
N_{t,j} = \sum_{r = 1}^t \mathbf{1}\{J_r = j\}.
$$
Then it holds that $N_t=N_{t-1}+e_{J_t}$, where $e_j$ is the $j$th canonical basis vector. This process $\{J_t\}$ can be formulated as a masked-action sequential decision process: the state is the current cell $C_{t-1}$, the random mask is the candidate set $U_t$, the action is the selected coordinate $J_t\in U_t$, the one-step split score is $\mathcal G(C_{t-1},J_t)$, and the state transition is the deterministic update $N_t=N_{t-1}+e_{J_t}$.

We define the filtration generated by the process as
$
\mathcal H_t:=\sigma(U_1,J_1,\dots,U_t,J_t)$, 
$t\ge 0.
$ A split policy is a stochastic kernel
$
\pi(\cdot\mid C_{t-1},U_t)
$
on $U_t$ satisfying that 
$
\pi(j\mid C_{t-1},U_t)\ge 0, \ 
\sum_{j\in U_t}\pi(j\mid C_{t-1},U_t)=1.
$
Under such a policy, we have that 
$$
\Prob_\pi(J_t=j\mid \mathcal H_{t-1},U_t)=\pi(j\mid C_{t-1},U_t),
\qquad
N_t=N_{t-1}+e_{J_t}.
$$
Consequently, the feature-subsampled population CART policy is the greedy policy $\pi_{\gr}$ induced by \eqref{eq:cart-pop-rule-general} with
\begin{equation}\label{eq:greedy-policy-dist}
\pi_{\gr}(j\mid C_{t-1},U_t)
=
\frac{\1\{j\in A_{\gr}(C_{t-1},U_t)\}}{|A_{\gr}(C_{t-1},U_t)|}, \qquad j \in [d].
\end{equation}

We restrict our attention to \emph{informative-respecting} policies, namely policies satisfying that 
\begin{equation}\label{eq:informative-respecting}
\pi(j\mid C_{t-1},U_t)=0, \text{ for all } j\in U_t\cap S^c
\end{equation}
 whenever $U_t\cap S\neq\varnothing$.
In other words, these policies prefer to split along informative covariates whenever it is possible.
Such class contains the population CART, the exploratory benchmark policy, and the mixture policies to be introduced later in Section \ref{sec:policyfamilies}.

The informative-respecting condition is a modeling restriction instead of a claim about the finite-sample CART. It focuses the analysis on the policy choice among informative coordinates after the mask has created an informative opportunity. This keeps the subsequent comparison centered on allocation and stabilization, rather than on whether the policy can recognize informative variables when they are available.


\subsection{Count-state reduction and shifted canonical structure} \label{new.sec.state}

We next introduce the two structural assumptions that connect the tree process from the previous subsection to the later split-count-based analysis. For technical simplicity, let us consider midpoint splits, so the split locations are deterministic after a split direction is chosen. Since each midpoint split halves the side length in the chosen coordinate, the side length of $C_t$ along coordinate $j$ is equal to $2^{-N_{t,j}}$. The count vector $N_t$ is thus the natural geometric state variable.

\begin{assumption}\label{ass:count-state}
There exist measurable functions
$
g_j:\mathbb N_0^d\to [0,\infty),
\ j=1,\dots,d,
$
such that the population CART impurity decrease depends on the current branch only through the split-count state
$
\mathcal G(C_{t-1},j)=g_j(N_{t-1})$,
$j\in[d].
$
That is, once the candidate mask $U_t$ is revealed, the split rule is Markovian in the count state $N_{t-1}$.
\end{assumption}

Assumption~\ref{ass:count-state} above provides the first dimension-reduction step: after midpoint splitting, the relevant branch geometry is summarized by split counts instead of by the full cell history. Then recursive tree growth can be viewed as a finite-dimensional Markov allocation problem, while retaining precisely the information required by the population split rule. In particular, under Assumption~\ref{ass:count-state}, policies can be written in the count-state form $\pi(\cdot \mid N_{t-1},U_t)$,
and the greedy action set is $A_{\gr}(N_{t-1},U_t)
:=
\argmax_{j\in U_t} g_j(N_{t-1})$. 
For a realized mask $u\subseteq[d]$, let us write
$
u^{\inf}:=u\cap S
$
for its informative part. 

\begin{assumption}\label{ass:shiftedcanonical}
There exist constants $\theta_j\in\mathbb R$, $j\in S$, such that for each state $\mathbf n = (n_1,\ldots, n_d)\in\mathbb N_0^d$ and 
candidate set $u\subseteq[d]$ with $|u|=m$,
\begin{equation}\label{eq:shiftedcanonical-order}
A_{\gr}(\mathbf n,u)
=
\begin{cases}
\argmin_{j\in u^{\inf}}(n_j-\theta_j),
& u^{\inf}\neq\varnothing,\\[0.75ex]
u,
& u^{\inf}=\varnothing.
\end{cases}
\end{equation}
Thus, the greedy rule is informative-respecting and prefers the smallest shifted count.
\end{assumption}

Assumption~\ref{ass:shiftedcanonical} above isolates the ordering property that drives the stabilization theory. The constants $\theta_j$ absorb persistent differences in the signal strength across informative coordinates, so that on the informative block, the greedy CART rule is equivalent to selecting exposed informative coordinates with the smallest shifted count $n_j-\theta_j$. As a result, Assumption~\ref{ass:shiftedcanonical} forms the bridge between impurity maximization and the balancing dynamics analyzed below. When the shifts are all equal, this reduces to the unshifted smallest-count rule on the exposed informative block; when they differ, the same rule applies after a deterministic coordinatewise shift.

Combining Assumptions~\ref{ass:count-state} and~\ref{ass:shiftedcanonical}, the greedy policy \eqref{eq:greedy-policy-dist} admits the explicit count-state representation
\begin{equation}\label{eq:pi-gr-explicit}
\pi_{\gr}(j\mid \mathbf n,u)
=
\begin{cases}
\dfrac{\1\{j\in \argmin_{k\in u^{\inf}}(n_k-\theta_k)\}}
{\left|\argmin_{k\in u^{\inf}}(n_k-\theta_k)\right|},
& u^{\inf}\neq\varnothing,\\[2ex]
\dfrac{1}{|u|}\1\{j\in u\},
& u^{\inf}=\varnothing.
\end{cases}
\end{equation}
Equation~\eqref{eq:pi-gr-explicit} is the policy-level representation of the population CART: the environment, through $U_t$, determines which coordinates are exposed, while the policy allocates informative opportunities across the exposed informative coordinates in $U_t\cap S$.

\subsection{Informative opportunities and informative time} \label{new.sec.oppotime}


Denote by 
$
K_t:=|U_t\cap S|$,
$
I_t:=\1\{K_t\ge 1\}$, and
$
M_t:=\sum_{r=1}^t I_r.
$
Here, $K_t$ is the number of informative coordinates in the candidate set at depth $t$, $I_t$ records whether depth $t$ presents an \textit{informative opportunity} (i.e., the availability of an informative covariate in $S$), and $M_t$ counts the total number of informative opportunities up to depth $t$. Since $U_t$ is sampled uniformly among the $m$-subsets of $[d]$, we have that 
$
K_t\sim \mathrm{Hypergeometric}(d,s,m).
$
Hence, the informative-opportunity rate is
\begin{equation}\label{eq:qdef}
q:=\Prob(I_t=1)=1-\frac{\binom{d-s}{m}}{\binom{d}{m}}.
\end{equation}
Since the masks are independent and identically distributed (i.i.d.) across depths, the indicators $\{I_t\}$ are i.i.d. Bernoulli$(q)$, so we have that 
$$
M_t\sim \mathrm{Binomial}(t,q),
\qquad
\frac{M_t}{t}\to q
\qquad\text{almost surely.}
$$
The parameter $q$ is the informative-variable opportunity rate induced by environment $U_t$.

To isolate the policy effect, we define the \textit{informative depths}
$$
T_n:=\inf\{t\ge 1:M_t=n\}
\qquad n\ge 1,
$$
with $T_0:=0$. The process $\{T_n\}$ records the tree depths at which informative opportunities occur. Under informative-respecting policies, these are also the depths at which an informative coordinate is selected for splitting. Thus, $M_t$ serves as the \textit{informative clock}, and we reindex the selected informative coordinates as 
$$
J_n^{\inf}:=J_{T_n}\in S,
\qquad n\ge 1.
$$
The \textit{informative-time count process} is then defined as
$$
Z_n=(Z_{n,j})_{j\in S} \in \mathbb{N}^s,
\qquad
Z_0=0,
\qquad
Z_n=\sum_{u=1}^n e_{J_u^{\inf}},
\qquad
\sum_{j\in S} Z_{n,j}=n.
$$
Equivalently, it holds that 
$
Z_{n+1}=Z_n+e_{J_{n+1}^{\inf}}.
$ 
For each $j\in S$, the raw count process $N_{t,j}$ can be reexpressed in terms of the informative clock $M_t$ and informative-time count process $Z_n$ as
\begin{equation}\label{eq:clock-link}
N_{t,j}=Z_{M_t,j}.
\end{equation}
These two identities play an important role in the later analysis: $M_t$ depends only on feature subsampling, whereas the law of $Z_n$ depends only on the policy's allocation rule on informative covariates in the masked set $U_t$.

We now translate the greedy action set in Assumption \ref{ass:shiftedcanonical} from raw-time counts $N_t$ to informative-time counts $Z_n$. Observe that
$
Z_n=(N_{T_{n+1}-,j})_{j\in S},
$
where $N_{T_{n+1}-}$ represents the split-count vector immediately before the
decision at time $T_{n+1}$. In light of Assumption \ref{ass:shiftedcanonical}, we have the equivalent representation
\begin{equation}\label{eq:greedy-act-infor}
A^{\inf}_{\gr}(N_{T_{n+1}-},U_{T_{n+1}}) =
A^{\inf}_{\gr}(Z_n,U_{T_{n+1}}\cap S) =\argmin_{j\in U_{T_{n+1}}\cap S}(Z_{n,j}-\theta_j).
\end{equation}

Since there are $s=|S|$ informative coordinates, we define the \textit{informative imbalance} at informative time $n$ as the deviation of the informative counts from perfect equalization (i.e., allocating equal counts to each informative coordinate) 
\begin{equation}\label{eq: inform-imba}
\Delta_n:=Z_n-\frac{n}{s}\mathbf 1_s, 
\end{equation}
where $\mathbf 1_s$ is a vector of ones with length $s$. By construction, it holds that 
$
\sum_{j\in S}\Delta_{n,j}=0.
$
We further define the associated \textit{Lyapunov function}
\begin{equation}\label{eq:VWdef}
V_n:=\|\Delta_n\|_2^2,
\qquad
W_n:=\sqrt{V_n}.
\end{equation}
They measure how unbalanced the split counts are across informative coordinates.
The variables $(\Delta_n,V_n,W_n)$ will drive the local-optimality and stabilization analysis presented in later sections.

Finally, let us define the informative-time filtration as 
$
\mathcal F_n:=\mathcal H_{T_n}
$, $ n\ge 0$.
All policies compared in this paper share the same feature subsampling scheme, and thus the same mask process, informative clock $M_t$, and opportunity rate $q$. They differ \textit{only} through the conditional law of $J_n^{\inf}$, equivalently through the law of $Z_n$. In particular, policies differ only on event $\{K_t\ge 2\}$: if $K_t=0$ no informative split is possible, and if $K_t=1$ the informative action is forced for all informative-respecting policies. The benchmark and mixture policies introduced later exploit exactly this common masked-action environment.

\subsection{Verification of Assumptions \ref{ass:count-state} and \ref{ass:shiftedcanonical} under the sparse linear midpoint model} \label{subsec:linearspecialization}

We now verify the preceding assumptions in the sparse linear midpoint model that underlies the explicit calculations in Section \ref{sec:stat} later. Consider the linear model
\begin{equation}\label{eq:model}
Y=X_1\beta_1+\cdots+X_d\beta_d+\varepsilon,
\end{equation}
where covariate vector $X=(X_1,\dots,X_d)^\top$ has i.i.d. coordinates uniformly distributed on $[0,1]$, and noise $\varepsilon$ is independent of $X$ with 
$
\E[\varepsilon]=0$ and 
$
\Var(\varepsilon)=\sigma_0^2<\infty.
$ 
We assume the sparsity, i.e., 
$
\beta_j\neq 0$ for $ j\in S$
and
$ \beta_j=0$ for $ j\in S^c.
$

For a cell
$
C=\prod_{k=1}^d (a_k,b_k),
$ since the covariates are i.i.d. uniformly distributed, 
the population CART rule always makes the midpoint split in the informative coordinates in $S$. Then the population CART impurity decrease has the explicit form
\begin{equation}\label{eq:reward}
\mathcal G(C,j) = g_j(N_{t-1})
=
\begin{cases}
(\beta_j^2\,2^{-2N_{t-1,j}})/12,
& j\in S,\\[1ex]
0,
& j\in S^c.
\end{cases}
\end{equation}
Hence, the greedy action set is
$$
A_{\gr}(N_{t-1},U_t)
=
\begin{cases}
\argmax_{j\in U_t\cap S}\beta_j^2\,2^{-2N_{t-1,j}},
& U_t\cap S\neq\varnothing,\\[1ex]
U_t,
& U_t\cap S=\varnothing.
\end{cases}
$$
Equivalently, defining the shifted informative counts
$\widetilde N_{t,j}
:=
N_{t,j}-\frac{1}{2}\log_2(\beta_j^2)
$, $j\in S$,
we have that 
$
\beta_j^2\,2^{-2N_{t-1,j}}
=
2^{-2\widetilde N_{t-1,j}},
$
and thus 
\begin{equation}\label{eq:greedy-policy-shifted}
A_{\gr}(N_{t-1},U_t)
=
\begin{cases}
\argmin_{j\in U_t\cap S}\widetilde N_{t-1,j},
& U_t\cap S\neq\varnothing,\\[1ex]
U_t,
& U_t\cap S=\varnothing.
\end{cases}
\end{equation}
Consequently, the population CART rule is exactly a smallest-shifted-count rule on the informative block. This verifies Assumptions~\ref{ass:count-state} and~\ref{ass:shiftedcanonical} with
$\theta_j=(1/2)\log_2(\beta_j^2)$, $j\in S$. 

Recall the definition of the informative-time count process $Z_n$.
In view of the preceding expression, we also have the following representation at informative time points
\begin{equation}\label{eq:pi-gr-informative-unshift}
\Prob_{\pi_{\gr}}(J_{n+1}=j\mid \F_n,U_{T_{n+1}})
=
\frac{\1\!\left\{j\in \argmin_{k\in U_{T_{n+1}}\cap S} \{Z_{n,k}-\frac{1}{2}\log_2(\beta_k^2)\}\right\}}
{\left|\argmin_{k\in U_{T_{n+1}}\cap S} \{Z_{n,k}-\frac{1}{2}\log_2(\beta_k^2)\}\right|},
\qquad j\in S\cap U_{T_{n+1}}
\end{equation}
with ties broken uniformly.


\subsection{Terminal-law formulation of the forest ensemble-design problem}
\label{subsec:terminal-law-design}

We fix a depth $\ell$. Let $\Pi_\ell$ be the class of nonanticipative split policies
$\pi=(\pi_t)_{t=1}^{\ell}$ such that at each depth $t$ in a tree branch,
$$
\pi_t(j\mid \mathcal H_{t-1},U_t)\ge 0,\qquad
\sum_{j\in U_t}\pi_t(j\mid \mathcal H_{t-1},U_t)=1,
\qquad
\pi_t(j\mid \mathcal H_{t-1},U_t)=0 \quad \text{for } j\notin U_t.
$$
We restrict attention to the informative-respecting class \eqref{eq:informative-respecting}. Under Assumption~\ref{ass:count-state}, the Markov count-state
policies $\pi_t(\cdot\mid N_{t-1},U_t)$ form a subclass of $\Pi_\ell$. For a policy $\pi\in\Pi_\ell$, denote by $\mathbb P_\pi$ the law of the masked-action process
generated by i.i.d. masks $(U_t)_{t=1}^{\ell}$, policy $\pi$, and transition
$
N_t=N_{t-1}+e_{J_t}$ with $N_0=\mathbf 0.
$
Define the terminal split-count law
$
\nu_{\pi,\ell}:=\mathcal L_\pi(N_\ell)
$
on the finite state space
$
\mathcal N_{\ell,d}
:=
\Bigl\{\mathbf n=(n_1,\cdots, n_d)^\top\in\mathbb N^d:\sum_{j=1}^d n_j=\ell\Bigr\}.
$


A $B$-tree ensemble grown under the same policy $\pi$ aggregates independent
trees discussed above. Then the terminal count states along the test branch satisfy that 
$
N_\ell^{(1)},\ldots,N_\ell^{(B)}
\stackrel{\mathrm{i.i.d.}}{\sim}
\nu_{\pi,\ell}.
$
This independence is with respect to the partition-generating randomization, conditional on the training data. 
For any functions
$
\Phi:\mathcal N_{\ell,d}\to\mathbb R$ and $
\Psi:\mathcal N_{\ell,d}\times \mathcal N_{\ell,d}\to\mathbb R,
$
let us define the terminal-law ensemble objective function
\begin{equation}
 \mathcal J_{\ell,B}^{\Phi,\Psi}(\pi):=
\frac{1}{B}\,
\mathbb E_{\nu_{\pi,\ell}}\!\left[\Phi(N_\ell)\right]+
\frac{B-1}{B}\,
\mathbb E_{\nu_{\pi,\ell}\otimes\nu_{\pi,\ell}}\!\left[
\Psi(N_\ell,N'_\ell)
\right],
\label{eq:terminal-law-objective} 
\end{equation}
where $N_\ell$ and $N'_\ell$ are independent random variables with law
$\nu_{\pi,\ell}$. Since $\mathcal N_{\ell,d}$ is finite, the expectations in
\eqref{eq:terminal-law-objective} are finite whenever $\Phi$ and $\Psi$ are finite-valued. Such criterion is motivated by the MSE risk decomposition of CART random forests to be presented in Section~\ref{sec:msemapping}, which consists of both single-tree and cross-tree terms. 

The feasible set of terminal laws induced by the policy class is
$
\mathfrak V_\ell
:=
\{\nu_{\pi,\ell}:\pi\in\Pi_\ell\}.
$
For any fixed $(\Phi,\Psi)$, optimizing
\eqref{eq:terminal-law-objective} over split policies is equivalently the terminal-law optimization
problem 
\begin{equation}
\inf_{\pi\in\Pi_\ell}\mathcal J_{\ell,B}^{\Phi,\Psi}(\pi) = \inf_{\nu\in\mathfrak V_\ell}\Bigl\{
\frac{1}{B}\int_{\mathcal N_{\ell,d}}\Phi(\mathbf n)\,d\nu(\mathbf n)
+
\frac{B-1}{B}
\int_{\mathcal N_{\ell,d}}\int_{\mathcal N_{\ell,d}}
\Psi(\mathbf n,\mathbf n')\,d\nu(\mathbf n)\,d\nu(\mathbf n')\Bigr\}.
\label{eq:ensemble-design-problem}
\end{equation}
In what follows, we will use 
$
\mathcal J_{\ell,B}^{\Phi,\Psi}(\pi)$
and $
\mathcal J_{\ell,B}^{\Phi,\Psi}(\nu_{\pi,\ell})$ interchangeably, where
the argument makes clear whether we
are evaluating the criterion at a policy $\pi$ or 
its induced terminal law
$\nu_{\pi,\ell}$.
Such formulation separates the local CART split problem from the ensemble design problem.
At each split, CART chooses an action from the realized mask $U_t$ to maximize the immediate impurity decrease $\mathcal G(C_{t-1},j)$, which is a one-step local objective. In contrast, \eqref{eq:ensemble-design-problem} is a terminal-law objective: the policy affects performance through the distribution of the depth-$\ell$ count state $N_\ell$ and 
the product law of two independently generated terminal states $N_\ell$ and $N_\ell'$. 

Section~\ref{sec:stat} later identifies the functions $\Phi$ and $\Psi$ that arise in the MSE risk decomposition in the special case of the sparse linear midpoint model. The
 term involving both $N_t$ and $N_t'$ is the source of the local/global tension: it depends on how two independent
trees interact and is \textit{not} visible to a one-step split criterion. Proposition \ref{prop:counterexample} will provide a counterexample that the locally stabilizing CART policy needs
not solve the induced terminal-law ensemble design problem.

\section{Greedy policy stabilization and its local optimality} \label{sec:results}

In this section, we analyze the subsampled population CART policy at the local level. We will focus on a local stabilization criterion based on one-step changes in the informative imbalance $\Delta_n$. The main result is that population CART is locally optimal for a natural balancing objective on the informative covariate block, and 
this local optimality induces contractive drift, first-order equalization, and second-order exponential stabilization. 
We begin by analyzing the unshifted case of Assumption \ref{ass:shiftedcanonical}, namely $\theta_j =0$ for all $j \in S$.
The extension to general setting in Assumption \ref{ass:shiftedcanonical} is then made at the end of Section \ref{new.sec.negdrift} through a deterministic centering argument.

\subsection{Negative drift and stabilization on the informative covariate block} \label{new.sec.negdrift}

At informative time
$n+1$, by definition the mask process satisfies that $U_{T_{n+1}}\cap S\neq \varnothing$, so the
relevant feasible set is the \textit{realized informative block} defined as
$
\mathcal A_{n+1}:=U_{T_{n+1}}\cap S$ with 
$|\mathcal A_{n+1}|\ge 1$. Conditional on the current informative-time state and 
realized set
$\mathcal A_{n+1}$, the decision is to choose the next informative coordinate
$
J_{n+1}\in \mathcal A_{n+1}.
$
Under the unshifted case $\theta_j = 0$, the informative-time transition under the population CART policy takes the form 
\begin{equation}\label{eq:pi-gr-informative}
    \P_{\pi_{\gr}}(J_{n+1} = j\mid \mathcal{F}_n, U_{T_{n+1}}) = 
    \frac{\1\left\{j \in \argmin_{k\in U_{T_{n+1}}\cap S} Z_{n,k}\right\}}{\left|\argmin_{k \in U_{T_{n+1}}\cap S} Z_{n,k}\right|},
    \qquad j \in S\cap U_{T_{n+1}}
\end{equation}
with tie broken randomly.

We set
$
a:=\sqrt{1-(1/s)}$ and $ 
b:=1-(1/s)=a^2.
$ The proposition below holds for any informative-time policy since each informative step selects only one informative covariate, and thus increases exactly one informative coordinate by one.

\begin{proposition}[One-step imbalance geometry] \label{prop:cart-onestep}
For the informative-time count process, we have that 
\begin{equation}\label{eq:one-step}
V_{n+1}=V_n+2\Delta_{n,J_{n+1}}+b,
\end{equation}
and almost surely for each $n\ge 0$,
\begin{equation}\label{eq:bounded-jump}
\abs{W_{n+1}-W_n}\le a.
\end{equation}
\end{proposition}

Proposition~\ref{prop:cart-onestep} above provides the basic 
identity for informative-time allocation. Each informative split adds one unit to a single coordinate and subtracts the common average increment from the centered state. Hence, the jump size of $W_n$ is uniformly bounded, and the policy enters the Lyapunov drift only through the coordinate selected at the next informative time. In particular, the entire one-step Lyapunov evolution
$\E[V_{n+1}-V_n\mid \F_n]$
is governed by the conditional term
$\E\bigl[\Delta_{n,J_{n+1}}\mid \F_n\bigr]$.
As a result, the population CART dynamics reduce, at the drift level, to how the greedy policy $\pi_{\gr}$ selects the next informative coordinate $J_{n+1}$ as a function of the current imbalance state. 

We impose the following nondegeneracy condition throughout the analysis.

\begin{assumption}[Informative competition]\label{ass:nondeg}
It holds that $
\Prob\bigl(|U_t\cap S|\ge 2\mid |U_t\cap S|\ge 1\bigr)>0.
$
\end{assumption}

Assumption~\ref{ass:nondeg} above is the minimal competition condition for within-mask allocation: it rules out the degenerate situation when each informative opportunity contains exactly one informative coordinate, i.e., $|U_t\cap S|=1$ whenever $U_t\cap S\neq\varnothing$. In that case, the informative action set would always be a singleton and the split would be forced by feature subsampling alone, so no adaptive rule could generate negative drift toward balance. By ensuring that the policy sometimes has a genuine choice among informative coordinates, Assumption~\ref{ass:nondeg} creates precisely the setting where policy-induced stabilization can occur.

Whenever the informative block is imbalanced and the mask exposes at least two informative coordinates (i.e., $|\mathcal A_{n+1}|\geq 2$), the greedy rule tends to select an under-split direction. The lemma below quantifies the restoring force.

\begin{lemma}[Negative drift by population CART] \label{lem:negdrift}
Assume that Assumptions~\ref{ass:shiftedcanonical} and~\ref{ass:nondeg} hold with $\theta_j = 0$ for all $j\in S$. 
Then there exists a constant $c_\star>0$ such that for each $n\ge 0$,
\begin{equation}\label{eq:negdrift}
\E\bigl[\Delta_{n,J_{n+1}}\mid \F_n\bigr]
\le -c_\star W_n
\qquad\text{almost surely.}
\end{equation}
One admissible choice is
$c_\star
=
s^{-1}(s-1)^{-3/2}
\E[(K-1)_+\,\Big|\, K\ge 1]$,
where $K\sim \mathrm{Hypergeometric}(d,s,m)$ is the number of informative coordinates in a candidate set of size $m$.
\end{lemma}

Lemma~\ref{lem:negdrift} above identifies the central stabilizing mechanism of population CART: when the informative counts are unbalanced, the greedy rule selects preferentially an under-split coordinate among those exposed by the mask, and constant $c_\star$ quantifies how strongly such corrective effect survives feature subsampling. Combining Proposition~\ref{prop:cart-onestep} and Lemma~\ref{lem:negdrift} leads to the Lyapunov drift bound
\begin{equation}\label{eq:V-drift}
\E\bigl[V_{n+1}-V_n\mid \F_n\bigr]
=
2\E\bigl[\Delta_{n,J_{n+1}}\mid \F_n\bigr]+a^2
\le -2c_\star W_n+a^2.
\end{equation}
The factor $c_\star$ separates the two ingredients of the effect: the hypergeometric term $K$ measures how often feature subsampling exposes multiple informative covariates, while the linear dependence on $W_n$ reflects that the greedy rule selects a coordinate whose corrective impact grows as the informative block moves farther from balance (i.e., larger $W_n$). Hence, population CART acts as a contractive policy in the masked-action environment.


The theorem below shows that sublinear growth of $W_n$ forces the informative coordinates to equalize at the first order.

\begin{theorem}[First-order equalization] \label{thm:firstorder}
Assume that Assumptions~\ref{ass:shiftedcanonical} and~\ref{ass:nondeg} hold with $\theta_j = 0$ for all $j\in S$. 
Then,
$$
\frac{Z_n}{n}\to \frac{1}{s}\1_s \ 
\text{ almost surely as } n\to\infty.
$$ 
Consequently, it holds that for each $j\in S$,
$$
\frac{N_{t,j}}{t}
=
\frac{Z_{M_t,j}}{M_t}\cdot \frac{M_t}{t}
\to \frac{q}{s}
\qquad\text{almost surely,}
$$
where $q$ is the informative opportunity rate from \eqref{eq:qdef}. 
In addition, if the policy treats noninformative coordinates symmetrically on the event $\{U_t\cap S=\varnothing\}$, 
for each $j\in S^c$,
$$
\frac{N_{t,j}}{t}\to \frac{1-q}{d-s}
\qquad\text{almost surely.}
$$
In particular,
$
N_t/t \to
\pi(\gamma)
:=
\Bigl(
\underbrace{q/s,\dots,q/s}_{s\text{ times}},
\underbrace{(1-q)/(d-s),\dots,(1-q)/(d-s)}_{d-s\text{ times}}
\Bigr)$
almost surely.
\end{theorem}

Theorem~\ref{thm:firstorder} above identifies the law-of-large-numbers limit, or the first-order equilibrium, of the greedy CART split-allocation process. It reveals that greedy local balancing equalizes allocation across the informative coordinates at leading order: on the informative block $S$, population CART assigns asymptotically equal mass to each informative coordinate. In the original count scale, this yields the limiting split-allocation vector $\pi(\gamma)$, whose informative component depends on the policy only through the total informative opportunity rate $q$. Such rate is determined entirely by feature subsampling, so the environment controls how often informative coordinates are split, while the greedy CART policy balances those splits evenly across the informative block.

We present the second-order stabilization result in the form of a uniform fluctuation control of the informative block imbalance measured by $W_n$.

\begin{theorem}[Compression of informative imbalance] \label{thm:expcomp}
Assume that Assumptions~\ref{ass:shiftedcanonical} and~\ref{ass:nondeg} hold with $\theta_j = 0$ for all $j\in S$. 
Then, for each 
$0<\eta<(4c_\star)/a^2$,
\begin{equation}\label{eq:expcomp}
\sup_{n\ge 0}\E\exp\{\eta W_n\}<\infty.
\end{equation}
\end{theorem}

Theorem~\ref{thm:expcomp} above is the second-order counterpart of Theorem~\ref{thm:firstorder}: it shows that population CART does more than attain the correct limiting allocation proportions. 
In addition, it keeps the informative imbalance uniformly exponentially tight around the balanced state. 
Such second-order control is what later permits equilibrium replacement in nonlinear MSE functionals discussed in Section \ref{sec:msemapping}. 


The \textit{distinction} between Theorems~\ref{thm:firstorder} and~\ref{thm:expcomp} is \textit{central} for the rest of the paper. The first theorem says that population CART and some other more exploratory policies can share the same first-order result, as long as they have the same feature subsampling step (and thus the same $q$). The second theorem captures their difference: greedy policy creates a restoring force strong enough to stabilize the informative block exponentially. Section~\ref{sec:policyfamilies} keeps the masked-action environment fixed and varies only the policy, where we will see the \textit{same} first-order limit but \textit{different} second-order behavior.

We now return to the general shifted formulation in Assumption \ref{ass:shiftedcanonical}.
The 
shifted dynamics can be reduced to the unshifted form through a deterministic centering argument. 
Define the shifted informative count
process $(\widetilde Z_n)$ as 
$$
\widetilde Z_n:=(\widetilde Z_{n,1},\dots,\widetilde Z_{n,s})^\top \ \text{ with } \ 
\widetilde Z_{n,j}:=Z_{n,j}-(\theta_j-\bar\theta) \text{ and } \bar\theta:=\frac1s\sum_{k=1}^s \theta_k.
$$ 
Then it holds that 
$
(1/s)\sum_{k\in S}\widetilde Z_{n,k}
=
n/s,
$
and let us recall \eqref{eq:greedy-act-infor}.
Using the notation above, the informative-time transition under the population CART policy $\pi_{\gr}$ takes the form
\begin{equation}\label{eq:pi-gr-informative-shifted}
\Prob_{\pi_{\gr}}(J_{n+1}=j\mid \F_n,U_{T_{n+1}})
=
\frac{\1\!\left\{j\in \argmin_{k\in U_{T_{n+1}}\cap S} \widetilde Z_{n,k}\right\}}
{\left|\argmin_{k\in U_{T_{n+1}}\cap S} \widetilde Z_{n,k}\right|},
\qquad j\in S\cap U_{T_{n+1}}
\end{equation}
with ties broken uniformly. Thus, the shifted process $\widetilde Z_n$ satisfies the same smallest-count dynamics as the unshifted process mentioned above. 
\begin{corollary}\label{cor:shifted-stable}
    Assume that Assumptions \ref{ass:shiftedcanonical}
    and \ref{ass:nondeg} hold.
    Then, the conclusions of Theorems \ref{thm:firstorder}
    and \ref{thm:expcomp} continue to hold under the shifted case with general $\theta_j$'s.
\end{corollary}
The proof is omitted, since $\widetilde Z_n$ and $Z_n$ differ only by a deterministic translation, and the same holds for their centered imbalance $\widetilde\Delta_n$ and $\Delta_n$. Consequently, the first-order limit and the exponential moment conclusions transfer immediately from the unshifted representation $(\widetilde Z_n, \widetilde \Delta_n)$ to the original process $(Z_n, \Delta_n)$.

\subsection{Local optimality on the informative covariate block} \label{new.sec.locoptim}

In this subsection, we identify the
precise \emph{local} objective optimized by the population CART policy. Motivated by the one-step analysis in the last subsection, we consider here the following one-step and pathwise question: among the feasible informative coordinates in $\mathcal A_{n+1}$, which action yields the smallest
post-decision imbalance? Under Assumption~\ref{ass:shiftedcanonical}, 
this question is naturally stated in terms of the shifted informative-time count
process $\widetilde Z_n$. 

Recall the centered shifted imbalance vector
$
\widetilde\Delta_n
:=
\widetilde Z_n- (n/s) \mathbf 1_s.
$ If at informative time $n+1$, action $j\in\mathcal A_{n+1}$ is selected,
the post-decision centered shifted imbalance is
$$
\widetilde\Delta_n^{(j,+)}
:=
\widetilde\Delta_n + e_j - \frac{1}{s}\mathbf 1_s.
$$
This motivates the one-step local Lyapunov potential
\begin{equation}\label{eq:local-potential}
 \mathfrak L_n(j)
:=
\bigl\|\widetilde\Delta_n^{(j,+)}\bigr\|_2^2
=
\left\|
\widetilde\Delta_n + e_j - \frac{1}{s}\mathbf 1_s
\right\|_2^2 
=\left\|
\widetilde Z_n+e_j-\frac{n+1}{s}\mathbf 1_s
\right\|_2^2,
\qquad j\in\mathcal A_{n+1},
\end{equation}
which records the post-decision imbalance after selecting coordinate $j$.
We show below that on each realized informative block $\mathcal A_{n+1}$, the greedy
population CART policy is exactly the policy that minimizes this local
potential $\mathfrak L_n(j)$.

The local potential in \eqref{eq:local-potential} is a one example of a class of dispersion functionals characterized using \textit{majorization}. 
For vectors $x,y\in\mathbb R^s$ with equal coordinate sum, 
we
write $x\prec y$ if $x$ is majorized by $y$. That is,
$$
\sum_{r=1}^k x_{[r]} \le \sum_{r=1}^k y_{[r]},
\qquad k=1,\dots,s-1,
\qquad
\sum_{r=1}^s x_r=\sum_{r=1}^s y_r,
$$
where $x_{[1]}\ge \cdots \ge x_{[s]}$ represents the decreasing rearrangement of the coordinates of $x$. Further, 
$\Psi:\mathbb R^s\to\mathbb R$ is called \textit{symmetric Schur-convex} if
$\Psi(Px)=\Psi(x)$ for each permutation matrix $P$, and
$
x\prec y
\Longrightarrow
\Psi(x)\le \Psi(y).
$
It is called \textit{symmetric strictly Schur-convex} if in addition,
$
x\prec y$ and
$x$ is not a permutation of $ y$
implies $
\Psi(x)<\Psi(y).
$ 
See \cite{marshall2011inequalities} for background on majorization and Schur-convexity.

For $j\in\mathcal A_{n+1}$, let us define the post-decision centered state
$$
X_n(j)
:=
\widetilde Z_n+e_j-\frac{n+1}{s}\mathbf 1_s
=
\widetilde\Delta_n+e_j-\frac1s\mathbf 1_s.
$$

\begin{proposition}[Schur-convexity local minimization]
\label{prop:majorization-local}
Fix informative time $n$ and a realized informative block
$\mathcal A_{n+1}\neq\varnothing$. 
Then, for each symmetric strictly Schur-convex function
$\Psi:\mathbb R^s\to\mathbb R$,
$$
\argmin_{j\in\mathcal A_{n+1}} \Psi(X_n(j))
=
\argmin_{j\in\mathcal A_{n+1}} \widetilde Z_{n,j}.
$$
\end{proposition}

Proposition~\ref{prop:majorization-local} above shows that: 
among the available informative coordinates in $\mathcal A_{n+1}$, 
the greedy rule minimizes every symmetric strictly Schur-convex function of the post-decision state. 
In particular,
choosing a coordinate with the smallest shifted count $\widetilde Z_{n,j}$ yields a minimal balanced post-decision state under the majorization order. 
Hence, the quadratic Lyapunov criterion is only one convenient representative of a broader local-balancing principle. 

Taking $\Psi(x)=\|x\|_2^2$, which is symmetric strictly Schur-convex, we have $\argmin_{j\in\mathcal A_{n+1}}\mathfrak L_n(j)=\argmin_{j\in\mathcal A_{n+1}}\widetilde Z_{n,j}$. 
By \eqref{eq:greedy-act-infor}, we also have $A^{\inf}_{\gr}(Z_n,\mathcal A_{n+1})=\argmin_{j\in\mathcal A_{n+1}}\widetilde Z_{n,j}$, which suggests that the greedy CART action minimizes the quadratic post-decision imbalance on the informative block.

\section{Exploration, exploitation, and policy families} \label{sec:policyfamilies}

Section~\ref{sec:results} has showed that population CART is a contractive policy on the informative covariate block. We now enlarge the policy space while holding the masked-action environment fixed.

\subsection{An exploratory benchmark policy} \label{new.sec.benchmark}

We compare the population CART policy $\pi_{\gr}$ to the pure exploratory benchmark, which employs the same feature-subsampling environment $U_t$ but removes the greedy within-mask exploitation. Concretely, if $u\subseteq\{1,\dots,d\}$ is the realized candidate covariate set and $u^{\inf}:=u\cap S$, we define
\begin{equation}\label{eq:pi-ex}
\pi_{\ex}(j\mid n,u)
=
\begin{cases}
\dfrac{1}{|u^{\inf}|}\1\{j\in u^{\inf}\}, & u^{\inf}\neq\varnothing,\\[2ex]
\dfrac{1}{|u|}\1\{j\in u\}, & u^{\inf}=\varnothing.
\end{cases}
\end{equation}
This policy is closely related to the purely random split rules considered in \cite{biau12a, klusowski21b}.

Since the candidate masks are sampled symmetrically over the informative coordinates, the informative-time dynamics under $\pi_{\ex}$ collapse to uniform random allocation on $S$. In particular, conditional on the informative-time filtration $\F_n$, it holds that 
$$
\P_{\pi_{\ex}}(J_{n+1}=j\mid \F_n)=\frac{1}{s},
\qquad j\in S.
$$
Equivalently, we have that 
$$
Z_n\sim \mathrm{Multinomial}\!\left(n,\frac1s,\dots,\frac1s\right)
\qquad\text{on informative time.}
$$
At each time $t$, the selected coordinate is i.i.d. with marginal law $\pi = (\pi_1,\cdots, \pi_d)$ given by 
$$
\pi_j=\frac{q}{s},\quad j\in S,
\qquad
\pi_j=\frac{1-q}{d-s},\quad j\in S^c;
$$
that is, $
N_t\sim \mathrm{Multinomial}(t,\pi).
$

\begin{proposition}[Exploratory benchmark] \label{prop:benchmark}
Under the exploratory policy $\pi_{\ex}$, we have that 
$Z_n/n\to 1/s \1_s
$ 
almost surely.
However, it holds that $\E_{\pi_{\ex}}\bigl[\Delta_{n,J_{n+1}}\mid \F_n\bigr]=0$ almost surely for all $n\ge 0$,
and for each $\eta>0$,
\begin{equation}\label{eq:no-exp-compression}
\sup_{n\ge 0}\E_{\pi_{\ex}} e^{\eta W_n}=\infty.
\end{equation}
\end{proposition}

Proposition~\ref{prop:benchmark} above isolates the role of within-mask exploitation by comparing population CART to an exploratory benchmark that faces the same informative opportunities. Since both policies exploit the same feature-subsampling mechanism to generate the mask process, they have the same first-order allocation limit on the informative block. What distinguishes them is the second-order geometry: the exploratory benchmark has \textit{no} restoring force and, once informative time is separated from the mask process, behaves like unconstrained multinomial allocation over the informative coordinates. Its lack of negative drift therefore leaves the informative counts in a diffusive multinomial regime as opposed to the uniformly exponentially compressed regime achieved by greedy population CART in Theorem~\ref{thm:expcomp}.

\subsection{Policy-level contraction coefficient and mixture policies} \label{new.sec.contrcoeff}

The comparison between the greedy CART policy and the exploratory benchmark in the last subsection suggests that a key quantity is the strength of the negative drift on informative imbalance. This motivates the following definition of the contraction coefficient.

\begin{definition}[Contraction coefficient]\label{def:kappa}
Let $\pi$ be an informative-time policy on the same masked-action environment as the feature-subsampled population CART policy. Its contraction coefficient is defined as 
\begin{equation}\label{eq:kappa-def}
\kappa(\pi)
:=
\sup\Bigl\{\kappa\ge 0:
\E_{\pi}\bigl[\Delta_{n,J_{n+1}}\mid \F_n\bigr]
\le -\kappa W_n \text{ almost surely for all }n\ge 0
\Bigr\}.
\end{equation}
When $\kappa(\pi)>0$, we call $\pi$ a \emph{contractive} policy on the informative block.
\end{definition}

Definition~\ref{def:kappa} above turns stabilization into a policy-level parameter by measuring how much linear negative drift a policy can guarantee per unit of current imbalance. The coefficient $\kappa(\pi)$ thus provides a common scale on which greedy, exploratory, and mixed policies can be compared. In such language, the key contrast between population CART and the exploratory benchmark is that the greedy policy has a strictly positive stabilizing coefficient $\kappa(\pi_{\gr})\ge c_\star$, whereas the exploratory policy has no policy-induced restoring force with $\kappa(\pi_{\ex})=0$.
Strictly speaking, the lower bound $\kappa(\pi_\gr) \ge c_\star$ was derived through a deterministic centering argument of the shifted dynamics as discussed in Corollary \ref{cor:shifted-stable} at the end of Section \ref{new.sec.negdrift}. Since the shifted formulation differs only by a deterministic translation,
the same conclusions carry over immediately.

The next corollary shows that once a policy induces linear negative drift in $W_n$, the same dynamical consequences follow regardless of whether that drift came specifically from population CART or some other more exploratory policy.

\begin{corollary}
 \label{cor:abstract}
Consider an informative-time policy $\pi$ and its induced informative-time count process $Z_n$ satisfying that 
$
Z_{n+1}=Z_n+e_{J_{n+1}},
\ 
\sum_{j=1}^s Z_{n,j}=n.
$
If for some $\kappa>0$,
\begin{equation}\label{eq:abstract-drift}
\E\bigl[\Delta_{n,J_{n+1}}\mid \F_n\bigr]\le -\kappa W_n
\quad\text{almost surely for all }n\ge 0,
\end{equation}
we have that 
$
Z_n/n\to 1/s \1_s
$ almost surely,
and for each 
$0<\eta<(2\kappa)/a^2$,
it holds that $
\sup_{n\ge 0}\E e^{\eta W_n}<\infty$.
\end{corollary}

Corollary~\ref{cor:abstract} above abstracts the stabilization argument away from CART by showing that once an informative-time policy has a positive contraction coefficient, the same first-order equalization and second-order exponential tightness conclusions follow. 
In this sense, the contraction coefficient $\kappa$ characterizes the policy's \textit{exploitation strength}: population CART is one example of a contractive policy, but the conclusions of Corollary~\ref{cor:abstract} are \textit{not} specific to impurity maximization or 
CART itself. 
The proof follows by the same arguments as in Theorems~\ref{thm:firstorder} and~\ref{thm:expcomp}, and is therefore omitted for simplicity.


The contraction coefficient gives a clean way to analyze the \textit{exploration--exploitation interpolation} on the same masked environment. With $\alpha\in[0,1]$, let us define the \textit{$\alpha$-mixture policy} family
\begin{equation}\label{eq:alpha-policy}
\pi_\alpha(\cdot\mid n,u)
=
\alpha\,\pi_{\gr}(\cdot\mid n,u)
+
(1-\alpha)\,\pi_{\ex}(\cdot\mid n,u).
\end{equation}
Equivalently, at each informative step one flips an independent Bernoulli$(\alpha)$ coin, and employs the greedy rule if the outcome is one and the exploratory rule otherwise. The parameter $\alpha$ is thus the within-mask exploitation probability.
{
As a consequence of Corollary \ref{cor:abstract}, we have the following conclusion for the $\alpha$-mixture policy.}

\begin{corollary}[$\alpha$-mixture policy] \label{cor:alpha}
Assume that the greedy policy $\pi_{\gr}$ has a contraction coefficient at least $c_{\gr}>0$, and the exploratory policy $\pi_{\ex}$ satisfies
that $
\E_{\pi_{\ex}}\bigl[\Delta_{n,J_{n+1}}\mid \F_n\bigr]\le 0$
almost surely for all $n\ge 0$.
Then, for the mixed policy $\pi_{\alpha}$, we have
$\kappa(\pi_\alpha)\ge \alpha \kappa(\pi_{\gr})$.
Consequently, it holds that 
$Z_n /n \to 1/s \1_s$ almost surely under $\pi_\alpha$,
and for each 
$0<\eta<(2\alpha \kappa(\pi_{\gr}))/a^2$,
we have that 
$
\sup_{n\ge 0}\E_{\pi_\alpha} e^{\eta W_n}<\infty$.
\end{corollary}

Corollary~\ref{cor:alpha} above shows how parameter $\alpha$ controls the exploitation-exploration tradeoff within the masked-action evironment. 
The mixed policy weakens the contraction coefficient linearly in the \textit{exploitation} probability $\alpha$, so the strength of exponential stabilization is reduced in proportion to the amount of \textit{exploratory} behavior, while the leading-order allocation limit is unchanged.

\begin{remark}[Opportunity rate and contraction strength]
The discussion above \textit{separates} the two levers of the masked-action dynamics.
The mask process determines the informative opportunity rate $q$: 
since $M_t/t\rightarrow q$ and
$
N_{t,j}=Z_{M_t,j}$ for all $ j\in S,
$
\textit{any} policy satisfying the contraction condition \eqref{eq:abstract-drift} obeys 
$$
\frac{N_{t,j}}{t}
=
\frac{Z_{M_t,j}}{M_t}\frac{M_t}{t}
\to
\frac{q}{s}, \qquad j\in S.
$$
The within-mask policy, in turn, determines the contraction coefficient $\kappa$, which controls how strongly the policy stabilizes around the first-order equilibrium on the informative block.
\end{remark}

\section{Statistical consequences of policy-induced stabilization under linear models} \label{sec:stat}

In this section, we will explain why the masked-action terminal-law formulation matters for statistical performance. The count-state process induced by a split policy determines the terminal cell geometry of each tree, and the forest MSE depends on both the one-tree terminal law and the product law of two independent trees. We exploit such representation to connect the local stabilization, cross-tree dependence, and ensemble-level optimality. Throughout the section, we will work under the sparse linear midpoint model
\eqref{eq:model}. 

\subsection{Equilibrium replacement and exact risk evaluation for population CART} \label{sec:msemapping}

We begin by expressing the MSE
risk of
as an explicit functional of the terminal split-count state.
Let $\hat{\mu}_{\star}(\cdot)$ with $\star \in \{\text{tree}, \text{RF}\}$ be either the tree estimator or the random forests estimator formed using the population CART split criterion and an i.i.d. training sample in $\mathcal{D} = \{(Y_i, X_i)\}$ with sample size $|\mathcal{D}|$. The explicit definition is given in Appendix \ref{app:defntreerf}. 
The MSE risk is defined as 
$$
\operatorname{MSE}_{\star} = \E_{{X'}, \mathcal D}\left[(\mu({X'}) - \hat{\mu}_{\star}({X'}))^2\right],
$$
where ${X'}=(X'_1,\ldots,X'_d)^\top$ is a test point that is independent of and identically distributed as the data in the training sample $\mathcal D$.


The following MSE decomposition result is policy-agnostic: it holds for
any split policy through its induced split-count distribution $N_\ell$.
Policy-specific behavior enters only through the law of $N_\ell$, which is
analyzed later. 
Throughout the paper, $a_n \lesssim b_n$ means that
$a_n \le C b_n$ for some universal constant $C>0$ independent of $n$,
and $a_n \asymp b_n$ means that
$a_n \lesssim b_n$ and $b_n \lesssim a_n$.

\begin{theorem}[MSE decomposition]
\label{thm:mse}
Under model \eqref{eq:model}, the depth-$\ell$ single-tree estimator satisfies that 
$$
\operatorname{MSE}_{\mathrm{tree}}
\asymp
\Bigl(\sum_{j\in S}\beta_j^2\E\big[2^{-2N_{\ell,j}}\big]\Bigr)
\Bigl(1+\frac{2^\ell}{n_0}\Bigr)
+
\sigma_0^2\frac{2^\ell}{n_0}
+
R_{\mathrm{tree}}
$$
with remainder $R_{\mathrm{tree}}\lesssim (1-2^{-\ell})^{n_0}$. For a forest of $B$ independent trees grown with the same splitting rule, we have that 
\begin{align}
\operatorname{MSE}_{\mathrm{RF}}
\asymp\;
&\frac{1}{B}
\underbrace{\left(\sum_{j\in S}\beta_j^2\E\big[2^{-2N_{\ell,j}}\big]\right)}_{\text{single-tree bias}}
\left(1+\frac{2^\ell}{n_0}\right)
+\frac{1}{B}\sigma_0^2\frac{2^\ell}{n_0} \notag\\
&+\frac{B-1}{B}
\underbrace{\left(\sum_{j\in S}\beta_j^2\E\big[2^{-2\max\{N_{\ell,j},N_{\ell,j}'\}}\big]\right)}_{\text{cross-tree bias}}
\left(1+\frac{2^\ell}{n_0}\right) \notag\\
&+\frac{B-1}{B}\underbrace{\sigma_0^2\frac{2^\ell}{n_0}
\E\left[2^{-\norm{N_\ell-N_\ell'}_1/2}\right]}_{\text{cross-tree variance}}
+R_{\mathrm{RF}},\label{eq:rf-mse}
\end{align}
where $R_{\mathrm{RF}}\lesssim (1-2^{-\ell})^{n_0}$, and $N_\ell$ and $N_\ell'$ represent independent copies of depth-$\ell$ split-count vector from two trees in the forest. 
\end{theorem}

Theorem~\ref{thm:mse} above provides the state-to-risk map for the midpoint model by expressing the single-tree and forest MSE through exponential functionals of the terminal split-count state and an independent copy. This is the point at which the controlled allocation dynamics become \textit{statistically relevant}: different policies affect the prediction risk through the terminal laws they induce for $N_\ell$ and $N_\ell'$. In particular, Theorem~\ref{thm:mse} identifies \textit{three} statistically relevant state functionals: a single-tree bias functional involving only $N_{\ell,j}$; a cross-tree bias functional involving $\max\{N_{\ell,j},N_{\ell,j}'\}$; and a cross-tree decorrelation functional involving $\|N_\ell-N_\ell'\|_1$. Hence, the MSE risk evaluation is reduced to the joint law of the terminal count states $N_\ell$ and $N_\ell'$. Equation~\eqref{eq:rf-mse} is a special case of the general terminal-law ensemble criterion \eqref{eq:terminal-law-objective}, specialized to the linear model setting. 
The proof of Theorem~\ref{thm:mse} follows the same general strategy as for Corollary~20 in \cite{MFL2024}, adapted to the present partitioning scheme, and is therefore omitted for simplicity. 

We now introduce a key technical innovation of this paper: an \emph{equilibrium-replacement principle}. 
It \textit{separates} prediction-risk analysis into \textit{two layers}.
The first layer shows that, under the stabilization results established earlier, the three exponential functionals appearing in Theorem \ref{thm:mse} can be replaced, up to universal constants, by their corresponding equilibrium proxies with multinomial structures. 
The second layer evaluates these equilibrium-proxy functionals explicitly through the Poisson-kernel representations developed later in Proposition \ref{prop:poisson-functionals}, finally yielding the non-asymptotic MSE bound for population CART forests in Theorem \ref{thm:cart-mse}. 

We detail the first layer of equilibrium replacement as follows. 
Let us introduce the $(d-s+1)$-dimensional probability vector 
$$
\pi^{\mathrm{pop}}
:=
\Bigl(
q,\ 
\underbrace{\tfrac{1-q}{d-s},\dots,\tfrac{1-q}{d-s}}_{d-s}
\Bigr).
$$
Recall that the total number of informative opportunities $M_{\ell}$.
The corresponding equilibrium replacement process $\bar N_\ell$ is defined by
$$
\bar N_{\ell,j} = \begin{cases}
M_{\ell}/s, & j \in S,\\
N_{\ell,j}, & j \in S^c.
\end{cases}
$$ 
Consequently,
$$
\left(M_\ell, (\bar N_{\ell, j})_{j \in S^c} \right) 
=
\left(M_\ell, (N_{\ell, j})_{j \in S^c} \right) \sim
\mathrm{Multinomial}(\ell, \pi^{{\mathrm{pop}}}).
$$

The key idea is that once the informative block $S$ is exponentially stabilized, the informative coordinates can be replaced by their balanced equilibrium proxy $M_{\ell}/s$ inside the three exponential functionals. 
In other words, 
the original terminal count processes $(N_\ell, N_{\ell}')$ are replaced by the equilibrium replacement processes $(\bar N_\ell, \bar N'_\ell)$, whose 
distribution structures are explicitly tractable.

{
The validity of equilibrium replacement is ensured by the following assumption. 
\begin{assumption}
\label{assm:etareq}
Let
$
\Delta_n:=Z_n-(n/s)\mathbf{1}_s$,
$
W_n:=\|\Delta_n\|_2,
$
and define $\eta_{\mathrm{req}}
:=
\max\{
2\log 2,\,
(\sqrt{s}\log 2)/2
\}$.
Assume that
$
K_*:=\sup_{n\ge 0}\E\exp\{\eta_{\mathrm{req}}W_n\}<\infty,
$
or equivalently, the contraction coefficient $\kappa(\pi)$ in Definition \ref{def:kappa} satisfies that 
\begin{equation*}\label{eq:ass-eta-req}
\kappa(\pi_{\gr}) > \frac{1-(1/s)}{4}\eta_{\rm req}.
\end{equation*}
\end{assumption}

Assumption~\ref{assm:etareq} above   ensures that fluctuations on the informative block remain uniformly controlled over depth, and thus the replacement errors can be absorbed into constants when evaluating exponential functionals. 
Since $\kappa(\pi_{\gr})\ge c_\star$ for population CART and drift constant $c_\star$ increases with feature subsampling ratio $\gamma$, Assumption~\ref{assm:etareq} can be interpreted as requiring sufficiently large $\gamma$; concrete verification for $s=2$ and $s=3$ is provided in Appendix~\ref{subsec:eta-s2-s3}.
}

We now turn to the second layer of equilibrium replacement, namely, the explicit evaluatation of the exponential functionals for equilibrium proxies.
Since $(\bar N_\ell, \bar N'_\ell)$ admit multinomial structures, these functionals can be analyzed through exact Poisson-kernel representations.
For $\rho\in(0,1)$, denote by 
$$
P_\rho(\theta):=\frac{1-\rho^2}{1-2\rho\cos\theta+\rho^2},
\qquad \theta\in[-\pi,\pi],
$$
and $\mu_\rho$ the probability measure on $[-\pi,\pi]$ with density
$P_\rho(\theta)/(2\pi)$. For $r\in(0,1)$ and $\alpha \in [0,1]$, define the
\emph{Poisson-kernel attenuation functional}
\begin{equation}\label{eq:Flr}
F_{\ell,r}(\alpha):=
\E_{\Xi\sim\mu_{\sqrt r}}
\Big[\abs{(1-\alpha)+\alpha e^{i\Xi}}^{2\ell}\Big].
\end{equation}
Further, for $\boldsymbol{p}=(p_1,\dots,p_d)$ on the simplex
$
\mathfrak S_{d-1}:=\left\{ \boldsymbol{p}\in[0,1]^d:\sum_{j=1}^d p_j=1\right\},
$ 
define the \emph{multivariate Poisson functional}
\begin{equation}\label{eq:Lldr}
L_{\ell,d,r}(\boldsymbol{p}):=
\E_{\Theta_1,\dots,\Theta_d}
\left[\abs{\sum_{j=1}^d p_j e^{i\Theta_j}}^{2\ell}\right],
\end{equation}
where $\Theta_1,\dots,\Theta_d$ are i.i.d.\ with law $\mu_{\sqrt r}$.

\begin{proposition}[Exact Poisson-kernel representations] \label{prop:poisson-functionals}
Let $r\in(0,1)$. 
\begin{enumerate}[label=(\roman*)]
\item
If $N_\ell,N_\ell' \stackrel{\mathrm{i.i.d.}}{\sim}\mathrm{Binomial}(\ell,p)$
with $p\in(0,1)$, it holds that 
\begin{equation}\label{eq:max-binomial-functional}
\E\big[r^{\max\{N_\ell,N_\ell'\}}\big]
=
\bigl(1-p+p\sqrt r\bigr)^{2\ell}
F_{\ell,r}\Bigl(\frac{p\sqrt r}{1-p+p\sqrt r}\Bigr),
\end{equation}
where the Poisson-kernel attenuation functional $F_{\ell,r}(\alpha)\asymp \ell^{-1/2}$, as $\ell\to\infty$,
for any fixed $r\in(0,1)$ and $\alpha\in(0,1)$.

\item
If $N_\ell,N_\ell' \stackrel{\mathrm{i.i.d.}}{\sim}\mathrm{Multinomial}(\ell,\boldsymbol{p})$
with $\boldsymbol{p}\in\mathfrak S_{d-1}$, it holds that 
\begin{equation}\label{eq:min-multinomial-functional}
\E\left[r^{\sum_{j=1}^d \min\{N_{\ell,j},N_{\ell,j}'\}}\right]
=
r^\ell L_{\ell,d,r}(\boldsymbol{p}),
\end{equation}
where the Poisson functional $
L_{\ell,d,r}(\boldsymbol{p})\asymp \ell^{-(d-1)/2}$, as $\ell\to\infty$,
for any fixed $d\ge 2$, and $\boldsymbol{p}$ lies in the interior of $\mathfrak S_{d-1}$.
\end{enumerate}
\end{proposition}

In Proposition~\ref{prop:poisson-functionals}, 
the one-dimensional functional $F_{\ell,r}$ captures the max-type interaction governing the cross-tree bias term, while the multivariate functional $L_{\ell,d,r}$ captures the overlap geometry governing the cross-tree decorrelation term. 
In addition, the polynomial prefactor in $L_{\ell, d, r}$ also reveals how the terminal-count dimension affects ensemble variance scaling. 
This completes the second layer of equilibrium replacement and yields the following explicit MSE risk result.

\begin{theorem}[Population CART MSE under equilibrium replacement] \label{thm:cart-mse}
Assume that the linear model \eqref{eq:model},
Assumptions \ref{ass:nondeg} and~\ref{assm:etareq} hold. Then the depth-$\ell$, $B$-tree population-CART forest considered in Theorem~\ref{thm:mse} satisfies the non-asymptotic bound
\begin{align}
\operatorname{MSE}_{\mathrm{RF}}
\lesssim\;
&\frac{1}{B}\Big(\sum_{j\in S}\beta_j^2\Big)
\bigl(1-q+q4^{-1/s}\bigr)^\ell
\left(1+\frac{2^\ell}{n_0}\right)
+\frac{1}{B}\sigma_0^2\frac{2^\ell}{n_0}
\notag\\
&+\frac{B-1}{B}\Big(\sum_{j\in S}\beta_j^2\Big)
\bigl(1-q+q2^{-1/s}\bigr)^{2\ell}
F_{\ell,2^{-2/s}}
\!\left(
\frac{q2^{-1/s}}{1-q+q2^{-1/s}}
\right)
\left(1+\frac{2^\ell}{n_0}\right)
\notag\\
&+\frac{B-1}{B}\sigma_0^2\frac{2^\ell}{n_0}
L_{\ell,d-s+1,2^{-1/2}}(\pi^{\pop})
+R_{\mathrm{RF}},
\label{eq:cart-mse}
\end{align}
where $R_{\mathrm{RF}}\lesssim (1-2^{-\ell})^{n_0}$.
\end{theorem}

Theorem~\ref{thm:cart-mse} above translates the policy-induced stabilization into an explicit MSE risk bound for population CART. The first-order informative opportunity rate $q$ controls the exponential bias factors, appearing through $(1-q+q4^{-1/s})^\ell$ and $(1-q+q2^{-1/s})^{2\ell}$, while the second-order exponential compression of the informative block enters through the collapsed variance proxy $L_{\ell,d-s+1,2^{-1/2}}(\pi^{\pop})$. This \textit{dimension collapse} is the main statistical signature of greedy stabilization in the MSE formula: once the informative counts are exponentially compressed around balance, the informative block behaves statistically like a single aggregated coordinate inside the cross-tree overlap functional, so the relevant dimension of $\pi^{\pop}$ is $d-s+1$ instead of $d$. Hence, Theorem~\ref{thm:cart-mse} is the statistical form of the stabilization theory developed above, with $q$ governing the leading bias decay and exponential tightness governing the effective dimension in the variance term. The proof of Theorem~\ref{thm:cart-mse} is presented in Appendix \ref{proof.thm:cart-mse}, and Assumption~\ref{assm:etareq} is verified for $s=2$ and $s=3$ under sufficiently large $\gamma$ in Appendix~\ref{subsec:eta-s2-s3}.

\subsection{Marginal Bellman certificate and global nonoptimality under ensemble MSE
}

Theorem~\ref{thm:mse} shows that the leading forest MSE terms can be written as a
terminal-law ensemble objective of the form \eqref{eq:ensemble-design-problem}.
We now use this formulation to study whether the locally stabilizing
CART policy optimizes the ensemble objective. 

The key object is the marginal terminal cost of a
terminal law defined in \eqref{eq:marginal-terminal-cost}. 
If a candidate policy is locally
optimal for the ensemble objective, it must satisfy the \textit{Bellman optimality condition} induced by the
marginal terminal cost generated by its own terminal law. 
The general Bellman optimality system for \eqref{eq:ensemble-design-problem} is developed in Theorem \ref{thm:marginal-value-bellman} of Appendix \ref{sec:general-bellman}. 
Here, we only state a specialized implication: 
if greedy CART is locally optimal for the terminal-law ensemble objective, at its own induced terminal law it must choose only marginal-Bellman-minimizing actions. 
Therefore, any violation of this Bellman optimality condition implies that greedy CART is \textit{not} locally optimal, and thus \textit{not} globally optimal.

In addition to the notation in Section \ref{subsec:terminal-law-design}, let us introduce some additional notation to facilitate the presentation. For
$\nu\in\mathfrak V_\ell$, we define 
\begin{equation}
\Gamma_{\nu}^{\Phi,\Psi}(\mathbf n)
:=
\frac{1}{B}\Phi(\mathbf n)
+
\frac{B-1}{B}
\int_{\mathcal N_{\ell,d}}
\{\Psi(\mathbf n,\mathbf n')+\Psi(\mathbf n',\mathbf n)\}\,d\nu(\mathbf n'),
\qquad \mathbf n\in\mathcal N_{\ell,d}.
\label{eq:marginal-terminal-cost}
\end{equation}
According to Theorem \ref{thm:marginal-value-bellman}, if greedy CART \textit{were} locally optimal, it must solve the linearized terminal-cost problem
$
\inf_{\pi\in\Pi_\ell}
\mathbb E_\pi
\left[
\Gamma_{\nu^{\rm gr}}^{\Phi,\Psi}(N_\ell)
\right].
$
Thus,
$\Gamma_{\nu^{\rm gr}}^{\Phi,\Psi}$ is the marginal terminal cost of the original ensemble objective at
the terminal law generated by greedy CART.

Denote by 
 $
 \mathcal U_m:=\{u\subseteq[d]: |u|=m\}$ and
 $p_m(u):=\binom{d}{m}^{-1}$ for $ u\in\mathcal U_m.
 $
 Define the admissible action set under mask $u$ as $\mathcal A(u)$, where $\mathcal A(u)=u\cap S$ if $u\cap S\neq\varnothing$, and $\mathcal A(u) = u$ if $ u\cap S=\varnothing$. 
For a terminal law $\nu\in\mathfrak V_\ell$, we define the marginal
continuation values, associated with marginal terminal
cost $\Gamma_{\nu}^{\Phi,\Psi}$, by the backward recursion
$$
\mathcal V_{\ell}^{\Gamma_{\nu}^{\Phi,\Psi}}(\mathbf n)
:=
\Gamma_{\nu}^{\Phi,\Psi}(\mathbf n),
\qquad \mathbf n\in\mathcal N_{\ell,d},
$$
and for $r=\ell-1,\ldots,0$,
$$
\mathcal V_r^{\Gamma_{\nu}^{\Phi,\Psi}}(\mathbf n)
:=
\sum_{u\in\mathcal U_m}p_m(u)
\min_{a\in\mathcal A(u)}
\mathcal V_{r+1}^{\Gamma_{\nu}^{\Phi,\Psi}}(\mathbf n+e_a),
\qquad \mathbf n\in\mathcal N_{r,d}.
$$
According to Lemma~\ref{lem:cost-to-go-bellman} in Appendix \ref{sec:general-bellman}, for any current state $N_t=\mathbf n$ and
feasible action $j\in\mathcal A(U_{t+1})$,
$$
\mathcal V_{t+1}^{\Gamma_{\nu}^{\Phi,\Psi}}(\mathbf n+e_j)
=
\inf_{\pi_{(t+2):\ell}}
\mathbb E_\pi
\left[
\Gamma_\nu^{\Phi,\Psi}(N_\ell)
\,\middle|\,
N_{t+1}=\mathbf n+e_j
\right].
$$
Consequently, $\mathcal V_{t+1}^{\Gamma_{\nu}^{\Phi,\Psi}}(\mathbf n+e_j)$ is the minimum expected marginal
terminal cost achievable after choosing action $j$ at state $\mathbf n$, with all
future actions $\pi_{(t+2):\ell}$ chosen optimally for the marginal terminal cost
$\Gamma_\nu^{\Phi,\Psi}$.

If $\pi^{\rm gr}$ \textit{were} locally optimal for the
ensemble-design problem \eqref{eq:ensemble-design-problem}, by Theorem \ref{thm:marginal-value-bellman}(ii), at each depth $t$, each action used by greedy CART with
positive probability must minimize the marginal continuation value
$
j\mapsto
\mathcal V_{t+1}^{\Gamma_{\nu^{\rm gr}}^{\Phi,\Psi}}(N_t+e_j)
$
over the currently feasible action set $\mathcal A(U_{t+1})$. Hence, greedy CART is \textit{nonoptimal} whenever there exists a
reachable state-mask event on which CART assigns positive probability to an
action $j$, but some other feasible action $k\in\mathcal A(U_{t+1})$
satisfies
\begin{equation}\label{eq:002}
\mathcal V_{t+1}^{\Gamma_{\nu^{\rm gr}}^{\Phi,\Psi}}(N_t+e_k)
<
\mathcal V_{t+1}^{\Gamma_{\nu^{\rm gr}}^{\Phi,\Psi}}(N_t+e_j).
\end{equation}
If the realized mask contains only one
informative coordinate, every policy is forced to select it and no violation is possible.
Nonoptimality can arise \textit{only} on events where the mask exposes
multiple feasible informative actions. On such events, greedy CART chooses the
locally best impurity-reduction action, whereas the \textit{Bellman policy}
chooses the action with smallest marginal continuation value $\mathcal V_{t+1}^{\Gamma_{\nu^{\rm gr}}^{\Phi,\Psi}}(N_t+e_j)$. We formalize these results in the theorem below.


\begin{theorem}[Bellman certificate for nonoptimality of greedy CART]
\label{thm:greedy-bellman-certificate}
Assume that there exist $t\in\{0,\ldots,\ell-1\}$, an event
$E\in \mathcal H_t\vee\sigma(U_{t+1})$ with
$\P_{\pi^{\rm gr}}(E)>0$, and feasible actions
$j,k\in\mathcal A(U_{t+1})$ on $E$ such that on event $E$,
$
\pi^{\rm gr}_{t+1}(j\mid H_t,U_{t+1})>0
$
and \eqref{eq:002} holds.
Then $\nu^{\rm gr}$ is not a local minimizer of
$\mathcal J_{\ell,B}^{\Phi,\Psi}$ over $\mathfrak V_\ell$. Consequently, we have that 
$
\mathcal J_{\ell,B}^{\Phi,\Psi}(\pi^{\rm gr})
>
\inf_{\pi\in\Pi_\ell}
\mathcal J_{\ell,B}^{\Phi,\Psi}(\pi).
$
\end{theorem}

Theorem~\ref{thm:greedy-bellman-certificate} above converts the global nonoptimality into a verifiable \textit{Bellman certificate} by showing that if a greedy action fails to minimize the marginal continuation value induced by its own terminal law, greedy CART cannot be a local optimizer of the ensemble objective. 
We apply this certificate to the MSE-objective induced from Theorem~\ref{thm:mse} to show that greedy stabilization and ensemble optimality can \textit{disagree}. The \textit{counterexample} below is deliberately simple and isolates the structural source of such nonalignment: the local split rule optimizes one-step stabilization, whereas the forest objective is a terminal-law functional that also contains two-tree interaction terms that are not captured by the local greedy split rule.

\begin{proposition}[Counterexample to CART's global optimality]\label{prop:counterexample}
Consider the sparse linear midpoint model with
$d=6, S=\{1,2\}, m=4, \ell=2, \beta_1=\beta_2=1$, and $\sigma_0^2=0$.
Then for each integer $B\ge 15$, the terminal law induced by $\pi^{\rm gr}$ is not a
local minimizer, and thus not a global minimizer of $\mathcal J_{\ell,B}^{\Phi,\Psi}(\pi)$ over the terminal laws
induced by admissible informative-respecting policies.
\end{proposition}

\subsection{MSE bounds for the exploratory benchmark} \label{sec:benchmse}

We now evaluate the exact MSE risk of the exploratory benchmark $\pi_{\ex}$ from Section~\ref{sec:policyfamilies}. Under $\pi_{\ex}$, the informative-time process is drift-neutral and the raw split-count vector $N_{\ell}$ is multinomial with probability vector $\pi = (\pi_1,\cdots, \pi_d)^\top$, where
$$
\pi_j=\frac{q}{s},\quad j\in S,
\qquad
\pi_j=\frac{1-q}{d-s},\quad j\in S^c.
$$
That is, for each depth $\ell$, we have that 
$
N_\ell\sim \mathrm{Multinomial}(\ell,\pi)$. Let
$
N_\ell'$ be an independent copy of $N_\ell$.
Substituting this law into Theorem~\ref{thm:mse} and applying Proposition~\ref{prop:poisson-functionals} yield the following result.

\begin{theorem}[Benchmark MSE] \label{thm:bench-mse}
The depth-$\ell$ exploratory benchmark forest with $B$ trees satisfies that 
\begin{align}
\operatorname{MSE}_{\mathrm{RF}}
\asymp\;
&\frac{1}{B}\Big(\sum_{j\in S}\beta_j^2\Big)
\left(1-\frac{3q}{4s}\right)^\ell
\left(1+\frac{2^\ell}{n_0}\right)
+\frac{1}{B}\sigma_0^2\frac{2^\ell}{n_0}
\notag\\
&+\frac{B-1}{B}\Big(\sum_{j\in S}\beta_j^2\Big)
\left(1-\frac{q}{2s}\right)^{2\ell}
F_{\ell,2^{-2}}
\!\left(
\frac{q}{2s-q}
\right)
\left(1+\frac{2^\ell}{n_0}\right)
\notag\\
&+\frac{B-1}{B}\sigma_0^2\frac{2^\ell}{n_0}
L_{\ell,d,2^{-1/2}}(\pi)
+R_{\mathrm{RF}},
\label{eq:bench-mse}
\end{align}
where $R_{\mathrm{RF}}\lesssim (1-2^{-\ell})^{n_0}$.
\end{theorem}

Theorem~\ref{thm:bench-mse} above gives the matching risk evaluation for the drift-neutral exploratory benchmark, and is the statistical counterpart of Proposition~\ref{prop:benchmark}. Together, Theorems~\ref{thm:cart-mse} and~\ref{thm:bench-mse} translate the policy-level comparison in Sections~\ref{sec:results} and~\ref{sec:policyfamilies} into MSE terms; see also Table~\ref{tab:comparison}. Population CART and the exploratory benchmark share the same informative opportunity rate $q$ and thus the same first-order split-allocation vector $\pi(\gamma)$. Their \textit{difference} is \textit{second-order}: greedy population CART stabilizes the informative block, whereas the benchmark leaves it in a diffusive multinomial regime. Such distinction propagates directly into the MSE formulas. Since the exploratory benchmark has no stabilizing drift, the informative block is not compressed, and the cross-tree overlap term retains the full $d$-dimensional multinomial structure $L_{\ell,d,2^{-1/2}}(\pi)$. In contrast, greedy population CART compresses exponentially the informative fluctuations around their balanced configuration, so the informative block behaves as a single aggregated coordinate in the overlap functional and the variance proxy depends on the reduced $(d-s+1)$-dimensional structure $L_{\ell,d-s+1,2^{-1/2}}(\pi^{\pop})$. Consequently, although the two policies have the same first-order allocation limit, they induce different variance geometries. The population CART accelerates bias reduction and reduces the effective dimension of the overlap functional through stabilization, whereas the exploratory benchmark preserves more cross-tree \textit{decorrelation} since its informative block remains diffusive and the full $d$-dimensional overlap structure survives.

\begin{table}[ht]
\TABLE
{Structural comparison of count dynamics and risk scaling under population CART and exploratory benchmark.
\label{tab:comparison}}
{
\begin{tabular}{
>{\raggedright\arraybackslash}p{0.29\linewidth}
>{\raggedright\arraybackslash}p{0.31\linewidth}
>{\raggedright\arraybackslash}p{0.31\linewidth}}
\hline
\up
& Population CART
& Exploratory benchmark \\
\hline
\up
First-order limit
& $N_t/t\to \pi(\gamma)$
& Same limit $\pi(\gamma)$ \\
Imbalance drift
& $\E[\Delta_{n,J_{n+1}}\mid\F_n]\le -c_\star W_n$
& $\E[\Delta_{n,J_{n+1}}\mid\F_n]=0$ \\
Exponential order of $W_n$
& $\exists \lambda>0:\sup_n \E e^{\lambda W_n}<\infty$
& $\forall \lambda>0:\sup_n \E e^{\lambda W_n}=\infty$ \\
Single-tree bias
& $\bigl(1-q+q4^{-1/s}\bigr)^\ell$
& $\left(1-(3q)/(4s)\right)^\ell$ \\
Cross-tree bias
& $\asymp \bigl(1-q+q2^{-1/s}\bigr)^{2\ell}\ell^{-1/2}$
& $\asymp \left(1-q/(2s)\right)^{2\ell}\ell^{-1/2}$ \\
Cross-tree variance
& $\asymp 2^\ell \ell^{-(d-s)/2}/n_0$
& $\asymp 2^\ell \ell^{-(d-1)/2}/n_0\down$ \\
\hline
\end{tabular}
}
{}
\end{table}

At the policy level, policies with larger contraction coefficient $\kappa$ sacrifice some cross-tree diversity in exchange for faster informative stabilization and better bias decay. Policies with smaller $\kappa$ preserve more diversity but do not benefit from the same second-order compression. In the current feature-subsampled setting, practical design must consider \textit{both levers} $q$ and $\kappa$.

\section{Design map and regime guidance} \label{sec:practical}

The \textit{mismatch} between the local and global objectives showed in Section~\ref{sec:stat} turns randomized tree construction into a \textit{design problem}. The theory we have developed identifies two control levers: the informative opportunity rate $q$ and the stabilization strength $\kappa$. We will employ these two levers to organize regime-dependent guidance for choosing and tuning randomized tree policies.

\subsection{Two control levers and their practical implication} \label{sec:practicallevers}

The statistical comparison in Section~\ref{sec:stat} reveals that the two levers $q$ and $\kappa$ enter at different orders. The practical implication is not that one should always prefer more greed or always prefer more exploration. Instead, randomized tree growth should be organized in a \textit{two-lever space} indexed by $(q,\kappa)$; see Table~\ref{tab:regimemap}. 
This section will translate this design map into practical guidance.

\begin{table}[ht]
\TABLE
{Practical regime map suggested by the comparison in
Section~\ref{sec:stat}.
\label{tab:regimemap}}
{
\footnotesize
\begin{tabular}{
>{\raggedright\arraybackslash}p{0.2\linewidth}
>{\raggedright\arraybackslash}p{0.22\linewidth}
>{\raggedright\arraybackslash}p{0.22\linewidth}
>{\raggedright\arraybackslash}p{0.3\linewidth}}
\hline
\up
Regime
& Dominant concern
& Theory-suggested direction
& Theory-based rationale \\
\hline
\up
Bias-dominated / strong-signal
& Approximation error dominates
& Larger $\kappa$ once $q$ is adequate
& Greedy stabilization accelerates informative compression and improves the bias terms. \\
Variance-dominated / low-signal
& Cross-tree diversity dominates
& Reduce $\kappa$ while keeping $q$ adequate
& More exploratory policies retain more diffusive informative behavior and can improve the variance term. \\
Sparse high-dimensional / weak exposure
& Informative opportunities are too rare
& Increase $q$ first
& If informative variables rarely enter the candidate set, exploitation cannot act often enough. \\
Uncertain or mixed regime
& Relative importance of bias and variance is unclear
& Tune both levers jointly
& Intermediate policies interpolate between greedy compression and exploratory decorrelation. \down\\
\hline
\end{tabular}
}
{}
\end{table}

When approximation is the main difficulty, our theory points toward larger $\kappa$ for better bias reduction. Once informative opportunities are already sufficiently frequent (sufficiently large $q$), a more CART-like within-mask policy is favored. When the main challenge is excessive cross-tree similarity, our theory points in the opposite direction by sacrificing some greediness to improve total risk even though the bias terms worsen. Our theory thus suggests that the empirical success of random forests should not be understood only through decorrelation by feature subsampling. It should instead be understood through the \textit{interaction} of these two levers. The \textit{$\alpha$-mixture family} preserves the same informative opportunity mechanism and changes only the strength of exploitation on informative time. For this reason, it is the \textit{natural policy family} for practical tuning between the greedy and exploratory extremes.

Many practical random forests variants move both levers simultaneously: they expose fewer features and they inject more randomness into split choice. In contrast, our theory suggests \textit{distinguishing} these operations. 
A simple \textit{tuning workflow} suggested by our theory is:
\begin{enumerate}
 \item[1)] choose a grid for the candidate-feature size, the practical proxy for $q$;
 \item[2)] choose a grid for split-policy softness, the practical proxy for $\kappa$;
 \item[3)] compare performance over that two-dimensional grid;
 \item[4)] inspect whether the preferred region is high-$\kappa$, low-$\kappa$, or intermediate once $q$ is large enough to avoid an opportunity bottleneck.
\end{enumerate}

A natural \textit{extension} is to let exploitation intensity increase with depth: more exploration near the root, where tree-to-tree dependence is especially consequential and large, and more greediness deeper in the tree, where local refinement matters more for smaller bias. The theory in this paper does not yet prove optimality of such a rule, but it does explain why this direction is natural.
Another natural \textit{extension} is to let the amount of greediness depend on the local separation among the top candidate feature impurity decrease. When the best candidate strongly dominates, aggressive exploitation is easier to justify. When several candidates are nearly tied, a softer, less greedy rule may preserve diversity at a relatively small cost in the approximation error. Again, this is a \textit{theory-motivated policy} idea rather than a theorem of the current paper.



\subsection{Implementable score-window split policies}
\label{subsec:empirical-score-window}

The preceding theory motivates mixture policies. Since the
informative set $S$ is unknown in applications, these policies, however, should not be interpreted as
practical algorithms. We now introduce an \textit{implementable analog} that uses only
the empirical CART scores on the realized candidate set. We adopt an honest forest construction. Let $\mathcal D_{\rm split}$ be the sample used to construct the tree partition. For a cell
$C$ and 
candidate coordinate $j$, define the empirical midpoint impurity decrease as 
$$
\widehat G(C,j)
:=
\widehat{\Var}_{\mathcal D_{\rm split}}(Y\mid X\in C)
-
\widehat p_L
\widehat{\Var}_{\mathcal D_{\rm split}}(Y\mid X\in C_{j,L})
-
\widehat p_R
\widehat{\Var}_{\mathcal D_{\rm split}}(Y\mid X\in C_{j,R}),
$$
where $C_{j,L}$ and $C_{j,R}$ are the two children obtained by splitting $C$ at the midpoint
of coordinate $j$, $\widehat p_L,\widehat p_R$ are the empirical child proportions within
$C$, and $\widehat{\Var}$ stands for the empirical variance. If a split violates the minimum leaf-size requirement, we set $\widehat G(C,j)=-\infty$.

We fix a window parameter $w\in[0,\infty)$. Given a realized candidate set $u$, denote by 
$
\widehat G_{\max}(C,u):=\max_{k\in u}\widehat G(C,k).
$
Define the empirical score-window action set
$$
\widehat A_w(C,u)
:=
\begin{cases}
\left\{
j\in u:
\widehat G(C,j)\ge 2^{-2w}\widehat G_{\max}(C,u)
\right\},
& \widehat G_{\max}(C,u)>0,\\[0.5em]
u,
& \widehat G_{\max}(C,u)\le 0.
\end{cases}
$$
The empirical score-window policy is given by 
$$
\widehat\pi^{\rm sw}_w(j\mid C,u,\mathcal D_{\rm split})
=
\frac{\mathbf 1\{j\in \widehat A_w(C,u)\}}
{|\widehat A_w(C,u)|},
\qquad j\in u.
\label{eq:empirical-score-window-policy}
$$
Hence, the case of $w=0$ recovers greedy empirical CART with random tie-breaking, while the case of $w = \infty$
corresponds to the exploration benchmark. The rule employs only the
realized candidate set and empirical split scores, and thus is practically implementable.


Under the sparse linear midpoint model, whenever $u\cap S\neq\varnothing$, policy $\pi^{\mathrm{sw}}$ using population-level quantities excludes the noninformative coordinates and satisfies that 
\begin{equation*}
A_w(C,u)
=
\Bigl\{
j\in u\cap S:
N_j-\theta_j
\le
\min_{k\in u\cap S}(N_k-\theta_k)+w
\Bigr\}.
\end{equation*}
This is exactly a randomized window around the
smallest shifted-count rule. 

We next examine whether the \textit{two-lever design message} persists for the empirical rule $\hat\pi^{\mathrm{sw}}_w$
 \eqref{eq:empirical-score-window-policy}. 
In the simulations, we consider a sparse linear midpoint model with ambient dimensionality 
$d=100$ and 
sparsity level $s=5$, where only the first five coordinates are informative and all nonzero coefficients are set equal to $1$.
Denote by $\boldsymbol{\beta} = (\beta_1,\ldots, \beta_s)^\top = \mathbf 1_s$ the coefficient vector of informative features.
The partitioning sample
$\mathcal D_{\rm split}$
and estimation sample
$\mathcal D_{\rm est}$
each have size $n_0 = 500$. Trees are grown using midpoint splits up to depth $\ell = 5$, with minimum leaf size equal to $5$. At each node, the candidate-set size is
$
m=\lceil \gamma d\rceil
$,
and the split coordinate is sampled according to the empirical score-window policy
$
\widehat\pi_w^{\rm sw}
$.
Each forest contains $B=200$ trees generated from the same sample split, while tree-level randomness is induced through independent feature subsampling and within-window randomized split selection. 
We consider both low- and high-SNR regimes, where
the signal-to-noise ratio $\mathrm{SNR}=\|\boldsymbol{\beta}\|_2/\sigma_0$
is controlled through the noise standard deviation $\sigma$. 
Prediction performance is evaluated on an additional independent test sample of size $n' = 100$. For each pair $(\gamma,w)$, we compute the empirical prediction error
$
\widehat{\mathrm{MSE}}(\gamma,w)
$
and report the Monte Carlo averages over repeated simulation runs.

\begin{figure}[ht]
\FIGURE
{\includegraphics[width=0.95\textwidth]{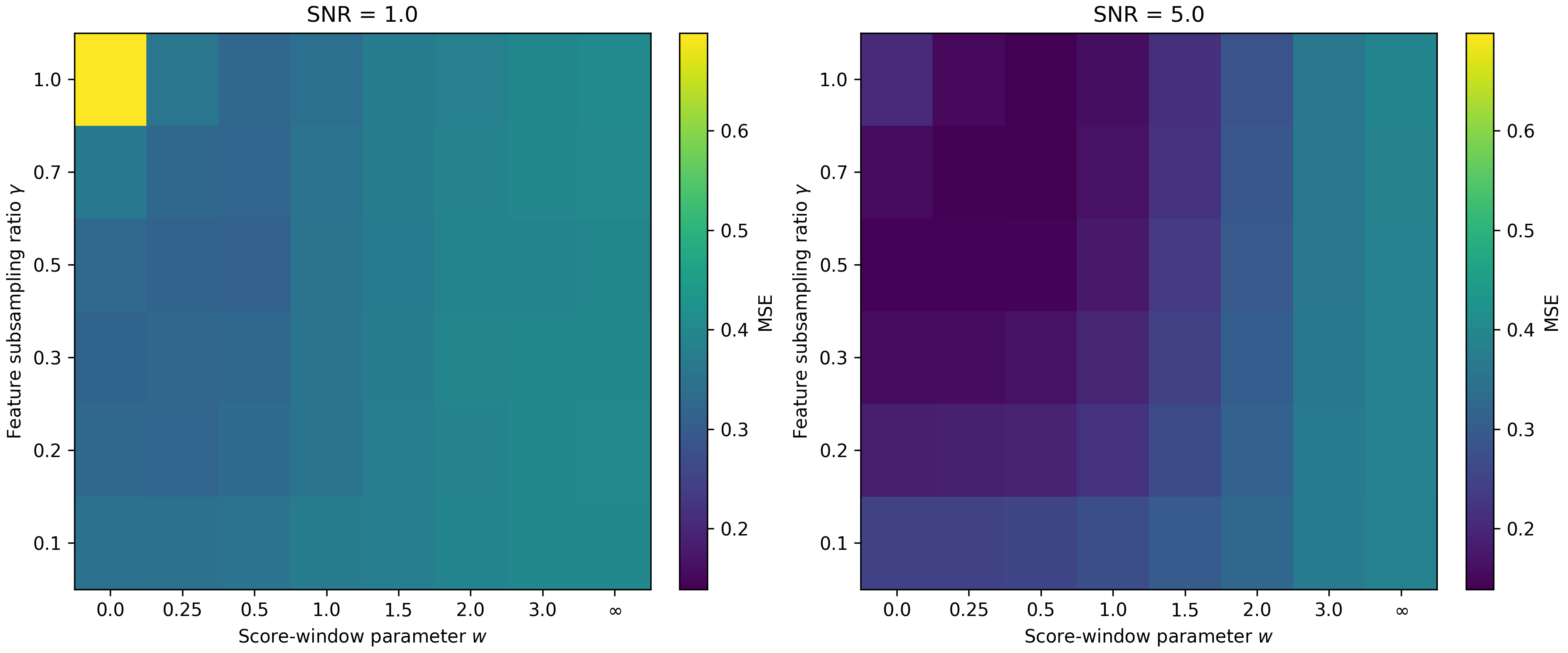}}
{Heatmaps of the empirical prediction error
$\widehat{\mathrm{MSE}}(\gamma,w)$
for honest forests under low- and high-SNR regimes.
Darker regions indicate smaller prediction error.
\label{fig:empirical-mse-heatmap}}
{}
\end{figure}

Figure~\ref{fig:empirical-mse-heatmap} depicts the heatmaps of
$
\widehat{\mathrm{MSE}}(\gamma,w)
$
under low- and high-SNR regimes. The horizontal axis varies the score-window parameter $w$, while the vertical axis varies the feature subsampling ratio $\gamma$. 
Color intensity represents the Monte Carlo average prediction MSE evaluated on an independent test sample. 
A clear regime-dependent pattern emerges from Figure~\ref{fig:empirical-mse-heatmap}.
When the feature subsampling level $\gamma$ is too small, prediction performance deteriorates uniformly across all score-window levels $w$, 
indicating that insufficient informative exposure cannot be compensated by more aggressive exploitation. 
This supports the theoretical interpretation of $\gamma$ as the \textit{primary} opportunity-generation lever. 
Once $\gamma$ becomes sufficiently large, the effect of the split-softness parameter $w$ becomes \textit{visible}. In the high-SNR regime, smaller score windows are consistently preferred, which is consistent with the bias-dominated picture predicted by the theory: 
stronger exploitation improves the informative compression and reduces the approximation error. 
\textit{In contrast}, under lower SNR, intermediate score windows become competitive and in some regions, outperform the fully greedy rule, suggesting that softer split randomization can improve overall prediction risk by mitigating excessive cross-tree dependence. 
The preferred region therefore does \textit{not} concentrate at either extreme, but instead forms an \textit{interior regime} balancing informative exploitation and exploratory decorrelation, consistent with the two-lever design map developed in Section~\ref{sec:practicallevers}.


These practical recommendations should be read at the level of \textit{mechanism}, not as universal defaults for each forest implementation. The current proofs are carried out under a sparse regression model and midpoint trees, and the exact Section~\ref{sec:stat} rates are model-specific. What the theory contributes more broadly is a structured language: the statistical consequences of the two levers can \textit{differ} at the first and second orders.

\section{Discussions} \label{new.sec.discu}

We have in this paper developed a stochastic-control perspective, named CART-ROSA, for understanding the mechanism behind the
feature-subsampled population CART random forests \citep{breiman2001random}.
Instead of approaching the problem from the algorithmic/statistical perspective, we formulate the tree growing process as a finite-horizon controlled stochastic system with random
feasible action sets.
Within this formulation, feature subsampling generates stochastic opportunity
sets, CART splitting acts as an adaptive masked-action policy, the informative
split-count vector evolves as a controlled state process, and the forest MSE
becomes a terminal-law objective induced by the resulting recursive dynamics.
Such formulation de-blackboxes a core mechanism of CART random forests by
separating two design levers:
the informative-opportunity generation mechanism induced by feature
subsampling, and the stabilization mechanism induced by the within-mask split
allocation policy.

Under a sparse regression midpoint model with uniformly distributed covariates, the induced split-count dynamics admit a
self-similar controlled stochastic structure, allowing the interaction between
greedy splitting and feature subsampling to be analyzed \textit{explicitly} through
stochastic approximation techniques.
The resulting analysis reveals a qualitative distinction between greedy CART
policy and pure exploration policy.
 Although sharing the same first-order equilibrium split
 allocation, their induced fluctuation geometries are fundamentally different.
 Pure exploration policy generates a diffusion-like fluctuation regime, whereas greedy population CART induces a
 self-stabilizing mechanism through a negative-drift effect along with an
 associated exponential compression phenomenon.
 Viewed through the stochastic-control formulation, greedy CART is therefore
 not just a local impurity-reduction rule; it also acts as a stabilizing
 allocation policy for the recursive split-allocation dynamics.

Our theory also suggests that greedy CART and exploration policies can have substantially different terminal geometries and
cross-tree interactions.
This reveals an important distinction between local split stabilization and
global ensemble performance:
a policy that stabilizes the recursive state dynamics of individual trees need
not optimize the terminal ensemble objective induced by the forest.
From such perspective, feature subsampling plays a role \textit{beyond} regularization
or decorrelation alone.
The subsampling mechanism controls simultaneously the informative-opportunity
arrival rate and the stabilization strength of the induced recursive
dynamics, leading naturally to an exploration--exploitation interpretation of
random forests construction.


The scope of our theory is deliberately limited. The \textit{sharpest} results are established for population CART under a sparse linear midpoint model. This model is not intended to capture
all features of Breiman's empirical random forests. Rather, it isolates a core
mechanism: the interaction between random feature opportunities, adaptive split
allocation, terminal cell geometry, and ensemble risk. The empirical CART setting
adds sampling noise, data-dependent split thresholds, finite-sample constraints,
and possible model misspecification. These features can alter both the state
dynamics and the induced terminal law, and they remain outside the current main theorem
package and left for future investigations.


Another important direction is to extend the current framework beyond the
 sparse regression setting with uniformly distributed covariates.
The midpoint model and uniformly distributed covariates help isolate the interaction between random opportunity
generation and adaptive split allocation in a particularly transparent form.
It would be interesting to understand whether similar stabilization and
terminal-law phenomena persist under broader model and covariate distribution settings, where the induced state dynamics and feasible-action
geometry become substantially more complicated.

\bibliographystyle{informs2014}
\bibliography{ref, RF}

@article{tan2024statcomp,
      title={Statistical-Computational Trade-offs for Greedy Recursive Partitioning Estimators}, 
      author={Yan Shuo Tan and Jason M. Klusowski and Krishnakumar Balasubramanian},
      journal = {Manuscript},
      year={2024},
      eprint={2411.04394},
      archivePrefix={arXiv},
      primaryClass={stat.ML},
      url={https://arxiv.org/abs/2411.04394},
      note = {arXiv:2411.04394}
}

@article{le2023survey,
  title={A Survey on the Impact of Hyperparameters on Random Forest Performance using Multiple Accelerometer Datasets},
  author={Le, Hong-Lam and Le, Thanh-Tuoi and Tran, Doan-Hieu and Van Chau, Dinh and others},
  journal={International Journal for Computers \& Their Applications},
  volume={30},
  number={4},
  pages={351--361},
  year={2023}
}

@inproceedings{bernard2009influence,
  title={Influence of hyperparameters on random forest accuracy},
  author={Bernard, Simon and Heutte, Laurent and Adam, S{\'e}bastien},
  booktitle={Multiple Classifier Systems: 8th International Workshop, MCS 2009, Reykjavik, Iceland, June 10-12, 2009. Proceedings 8},
  pages={171--180},
  year={2009},
  organization={Springer}
}

@phdthesis{zhou2022random,
  title={Random Forests and Regularization},
  author={Zhou, Siyu},
  year={2022},
  school={University of Pittsburgh}
}

@Article{zhou2023trees,
  author    = {Zhou, Siyu and Mentch, Lucas},
  journal   = {Statistical Analysis and Data Mining: The ASA Data Science Journal},
  title     = {Trees, forests, chickens, and eggs: when and why to prune trees in a random forest},
  year      = {2023},
  number    = {1},
  pages     = {45--64},
  volume    = {16},
  publisher = {Wiley Online Library},
}

@article{Buehlmann2002,
	author = {Peter B{\"u}hlmann and Bin Yu},
	doi = {10.1214/aos/1031689014},
	journal = {The Annals of Statistics},
	number = {4},
	pages = {927--961},
	publisher = {Institute of Mathematical Statistics},
	title = {{Analyzing bagging}},
	url = {https://doi.org/10.1214/aos/1031689014},
	volume = {30},
	year = {2002}}

@article{Scornet2015,
	author = {Erwan Scornet and G{\'e}rard Biau and Jean-Philippe Vert},
	doi = {10.1214/15-AOS1321},
	journal = {The Annals of Statistics},
	number = {4},
	pages = {1716--1741},
	publisher = {Institute of Mathematical Statistics},
	title = {{Consistency of random forests}},
	url = {https://doi.org/10.1214/15-AOS1321},
	volume = {43},
	year = {2015}}

@article{Athey2019,
	author = {Susan Athey and Julie Tibshirani and Stefan Wager},
	doi = {10.1214/18-AOS1709},
	journal = {The Annals of Statistics},
	number = {2},
	pages = {1148--1178},
	publisher = {Institute of Mathematical Statistics},
	title = {{Generalized random forests}},
	url = {https://doi.org/10.1214/18-AOS1709},
	volume = {47},
	year = {2019}}

@article{Mourtada2020,
	author = {Jaouad Mourtada and St{\'e}phane Ga{\"\i}ffas and Erwan Scornet},
	doi = {10.1214/19-AOS1886},
	journal = {The Annals of Statistics},
	number = {4},
	pages = {2253--2276},
	publisher = {Institute of Mathematical Statistics},
	title = {{Minimax optimal rates for Mondrian trees and forests}},
	url = {https://doi.org/10.1214/19-AOS1886},
	volume = {48},
	year = {2020}}

@article{cattaneo2024inferencemondrianrandomforests,
	archiveprefix = {arXiv},
	author = {Matias D. Cattaneo and Jason M. Klusowski and William G. Underwood},
	eprint = {2310.09702},
        eprinttype = {arxiv},
	primaryclass = {math.ST},
	title = {Inference with Mondrian Random Forests},
    journal = {Manuscript},
	url = {https://arxiv.org/abs/2310.09702},
	year = {2024},
        note = {arXiv:2310.09702}}

@article{biau12a,
	author = {G{{\'e}}rard Biau},
	journal = {Journal of Machine Learning Research},
	number = {38},
	pages = {1063--1095},
	title = {Analysis of a Random Forests Model},
	url = {http://jmlr.org/papers/v13/biau12a.html},
	volume = {13},
	year = {2012}}

@inproceedings{syrgkanis20a,
	author = {Syrgkanis, Vasilis and Zampetakis, Manolis},
	booktitle = {Proceedings of Thirty Third Conference on Learning Theory},
	editor = {Abernethy, Jacob and Agarwal, Shivani},
	pages = {3453--3454},
	series = {Proceedings of Machine Learning Research},
	title = {Estimation and Inference with Trees and Forests in High Dimensions},
	url = {https://proceedings.mlr.press/v125/syrgkanis20a.html},
	volume = {125},
	year = {2020}}

@inproceedings{klusowski21b,
	author = {Klusowski, Jason},
	booktitle = {Proceedings of The 24th International Conference on Artificial Intelligence and Statistics},
	editor = {Banerjee, Arindam and Fukumizu, Kenji},
	pages = {757--765},
	series = {Proceedings of Machine Learning Research},
	title = {Sharp Analysis of a Simple Model for Random Forests},
	url = {https://proceedings.mlr.press/v130/klusowski21b.html},
	volume = {130},
	year = {2021}}

@article{Lin2006,
	author = {Yi Lin and Yongho Jeon},
	doi = {10.1198/016214505000001230},
	eprint = {https://doi.org/10.1198/016214505000001230},
	journal = {Journal of the American Statistical Association},
	number = {474},
	pages = {578--590},
	publisher = {ASA Website},
	title = {Random Forests and Adaptive Nearest Neighbors},
	url = {https://doi.org/10.1198/016214505000001230},
	volume = {101},
	year = {2006}}

@article{Klusowski2024,
	author = {Jason M. Klusowski and Peter M. Tian},
	doi = {10.1080/01621459.2022.2126782},
	eprint = {https://doi.org/10.1080/01621459.2022.2126782},
	journal = {Journal of the American Statistical Association},
	number = {545},
	pages = {525--537},
	publisher = {ASA Website},
	title = {Large Scale Prediction with Decision Trees},
	url = {https://doi.org/10.1080/01621459.2022.2126782},
	volume = {119},
	year = {2024}}

@article{Wager2018,
	author = {Stefan Wager and Susan Athey},
	doi = {10.1080/01621459.2017.1319839},
	eprint = {https://doi.org/10.1080/01621459.2017.1319839},
	journal = {Journal of the American Statistical Association},
	number = {523},
	pages = {1228--1242},
	publisher = {ASA Website},
	title = {Estimation and Inference of Heterogeneous Treatment Effects using Random Forests},
	url = {https://doi.org/10.1080/01621459.2017.1319839},
	volume = {113},
	year = {2018}}

@article{Scornet2016,
	author = {Erwan Scornet},
	doi = {https://doi.org/10.1016/j.jmva.2015.06.009},
	issn = {0047-259X},
	journal = {Journal of Multivariate Analysis},
	pages = {72--83},
	title = {On the asymptotics of random forests},
	url = {https://www.sciencedirect.com/science/article/pii/S0047259X15001542},
	volume = {146},
	year = {2016}}

@article{Probst2018,
	author = {Philipp Probst and Anne-Laure Boulesteix},
	journal = {Journal of Machine Learning Research},
	number = {181},
	pages = {1--18},
	title = {To Tune or Not to Tune the Number of Trees in Random Forest},
	url = {http://jmlr.org/papers/v18/17-269.html},
	volume = {18},
	year = {2018}}

@article{delgado14a,
	author = {Manuel Fern{{\'a}}ndez-Delgado and Eva Cernadas and Sen{{\'e}}n Barro and Dinani Amorim},
	journal = {Journal of Machine Learning Research},
	number = {90},
	pages = {3133--3181},
	title = {Do we Need Hundreds of Classifiers to Solve Real World Classification Problems?},
	url = {http://jmlr.org/papers/v15/delgado14a.html},
	volume = {15},
	year = {2014}}

@article{liu2024randomizationreducebiasvariance,
	archiveprefix = {arXiv},
	author = {Brian Liu and Rahul Mazumder},
	eprint = {2402.12668},
        eprinttype= {arxiv},
    journal={Manuscript},
	primaryclass = {stat.ML},
	title = {Randomization Can Reduce Both Bias and Variance: A Case Study in Random Forests},
	url = {https://arxiv.org/abs/2402.12668},
	year = {2024},
        note = {arXiv:2402.12668}}

@article{curth2024randomforestsworkunderstanding,
	archiveprefix = {arXiv},
	author = {Alicia Curth and Alan Jeffares and Mihaela van der Schaar},
	eprint = {2402.01502},
        eprinttype= {arxiv},
	primaryclass = {stat.ML},
    journal = {Manuscript},
	title = {Why do Random Forests Work? Understanding Tree Ensembles as Self-Regularizing Adaptive Smoothers},
	url = {https://arxiv.org/abs/2402.01502},
	year = {2024},
        note = {arXiv: 2402.01502}}

@article{MentchZhou2020,
  author  = {Lucas Mentch and Siyu Zhou},
  title   = {Randomization as Regularization:  A Degrees of Freedom Explanation for Random Forest Success},
  journal = {Journal of Machine Learning Research},
  year    = {2020},
  volume  = {21},
  number  = {171},
  pages   = {1--36},
  url     = {http://jmlr.org/papers/v21/19-905.html}
}

@article{MFL2024,
  title={Exogenous randomness empowering random forests},
  author={Mei, Tianxing and Fan, Yingying and Lv, Jinchi},
  journal = {Manuscript},
  url = {https://arxiv.org/abs/2411.07554},
  note = {arXiv:2411.07554},
  year={2024}
}

@book{Butler2007,
  author    = {Ronald W. Butler},
  title     = {Saddlepoint Approximations with Applications},
  publisher = {Cambridge University Press},
  year      = {2007},
  address   = {Cambridge},
  series    = {Cambridge Series in Statistical and Probabilistic Mathematics},
  volume    = {22}
}

@incollection{RobbinsSiegmund1971,
  author    = {Robbins, Herbert and Siegmund, David},
  title     = {A convergence theorem for non-negative almost supermartingales and some applications},
  booktitle = {Optimizing Methods in Statistics},
  editor    = {Rustagi, J. S.},
  pages     = {233--257},
  publisher = {Academic Press},
  address   = {New York},
  year      = {1971}
}

@book{marshall2011inequalities,
  author    = {Marshall, A. W. and Olkin, I. and Arnold, B. C.},
  title     = {Inequalities: Theory of Majorization and Its Applications},
  edition   = {2},
  publisher = {Springer},
  address   = {New York},
  year      = {2011}
}

@book{Borkar2023,
  author       = {Borkar, Vivek S.},
  title        = {Stochastic Approximation: A Dynamical Systems Viewpoint},
  series       = {Texts and Readings in Mathematics},
  edition      = {2},
  publisher    = {Springer Singapore},
  year         = {2024},
  doi          = {10.1007/978-981-99-8277-6},
  note         = {Second Edition; Hindustan Book Agency},
}

@article{Fan22AOS,
author = {Chien-Ming Chi and Patrick Vossler and Yingying Fan and Jinchi Lv},
title = {{Asymptotic properties of high-dimensional random forests}},
volume = {50},
journal = {The Annals of Statistics},
number = {6},
publisher = {Institute of Mathematical Statistics},
pages = {3415--3438},
year = {2022},
doi = {10.1214/22-AOS2234},
URL = {https://doi.org/10.1214/22-AOS2234}
}

@article{klusowski2019analyzing,
  title={Analyzing {CART}},
  author={Klusowski, Jason M},
  journal={arXiv preprint arXiv:1906.10086},
  year={2019}
}

@book{kushneryin2003,
  title={Stochastic Approximation and Recursive Algorithms and Applications},
  author={Kushner, Harold J and Yin, G George},
  year={2003},
  publisher={Springer}
}

@article{breiman2001random,
  title={Random forests},
  author={Breiman, Leo},
  journal={Machine Learning},
  volume={45},
  number={1},
  pages={5--32},
  year={2001},
  publisher={Springer}
}

@book{bertsekas2012dynamic,
  title={Dynamic Programming and Optimal Control: Volume I},
  author={Bertsekas, Dimitri},
  volume={4},
  year={2012},
  publisher={Athena Scientific}
}

@book{powell2007approximate,
  title={Approximate Dynamic Programming: Solving the Curses of Dimensionality},
  author={Powell, Warren B},
  volume={703},
  year={2007},
  publisher={John Wiley \& Sons}
}

\newpage
\setcounter{page}{1}
\setcounter{section}{0}
\renewcommand{\theequation}{A.\arabic{equation}}
\setcounter{equation}{0}
	
\begin{center}{\bf \Large Supplementary Material to ``CART Random Forests as Sequential Allocation over Random Opportunity Sets: A Stochastic-Control Theory of Ensemble Risk''}
		
\bigskip
		
Tianxing Mei, Yingying Fan, Mingming Leng and Jinchi Lv
\smallskip
\end{center}

\noindent This Supplementary Material contains all the proofs of the main results, some key lemmas, and additional technical details.

\begin{APPENDICES}
\section{Technical preparation and results}

\subsection{Definitions of partitioning estimator and random partitioning ensemble}
\label{app:defntreerf}

We formally define the mean squared error (MSE) risk of the greedy CART tree and forest. To do so, let us first introduce some necessary notation. A partitioning estimator is specified by a pair
$(\mathcal P,\mathcal D)$, where $\mathcal P=\{P_j\}$ is a partition of
$C_0 =[0,1]^d$ and
$\mathcal D=\{(y_i,\boldsymbol X_i)\}_{i=1}^{n_{tr}}$ is an independent and identically distributed (i.i.d.) evaluation sample drawn from \eqref{eq:model}.
With such notation, the resulting \textit{partitioning estimator} takes the form
\begin{equation}\label{eq:partest}
\hat{\mu}_{\text{tree}}(\boldsymbol x;\mathcal P,\mathcal D)
=
\frac{\sum_{i=1}^{n_{tr}} y_i \mathbf 1\{\boldsymbol X_i\in \mathcal P_{\boldsymbol x}\}}
{\sum_{i=1}^{n_{tr}} \mathbf 1\{\boldsymbol X_i\in \mathcal P_{\boldsymbol x}\}},
\end{equation}
where $\mathcal P_{\boldsymbol x}$ represents the unique cell of $\mathcal P$
containing $\boldsymbol x$ and we adopt the convention of $0/0=0$.

The partition $\mathcal P$ itself may be random.
As such, let us write $\mathcal P=\mathcal P(\mathcal D,\Theta)$, where randomness through
$\mathcal D$ corresponds to the endogenous, data-driven variability, and $\Theta$ denotes the exogenous randomness independent of $\mathcal D$.
Given independent seeds $\Theta_1,\ldots,\Theta_B$, we can obtain independent
partitions $\mathcal P^{(b)}=\mathcal P(\mathcal D,\Theta_b)$,
$b=1,\ldots,B$.
The corresponding \textit{random partitioning ensemble} estimator is then defined by aggregation
\begin{equation}\label{eq:ensest}
\hat{\mu}_{\text{RF}}(\boldsymbol x;\{\mathcal P^{(b)}\},\mathcal D)
=
\frac{1}{B}\sum_{b=1}^B
\hat{\mu}_{\text{tree}}(\boldsymbol x;\mathcal P^{(b)},\mathcal D).
\end{equation}

\subsection{General marginal-value Bellman equations for terminal-law ensemble design}\label{sec:general-bellman}


In this subsection, we develop the Bellman optimality system for the terminal-law ensemble-design problem \eqref{eq:ensemble-design-problem}. We show that any local optimal terminal law $\nu^\star$ induces a linear terminal-cost control problem with marginal terminal cost $\Gamma_{\nu^\star}^{\Phi,\Psi}$; see \eqref{eq:marginal-terminal-cost} for definition.
Theorem~\ref{thm:marginal-value-bellman} below shows that $\Gamma_{\nu}^{\Phi,\Psi}(\mathbf n)$ is the marginal ensemble terminal cost; that is, for any
$\widetilde\nu\in\mathfrak V_\ell$,
$$
\left.
\frac{d}{d\varepsilon}
\mathcal J_{\ell,B}^{\Phi,\Psi}
\bigl((1-\varepsilon)\nu+\varepsilon\widetilde\nu\bigr)
\right|_{\varepsilon=0+}
=
\int_{\mathcal N_{\ell,d}}
\Gamma_\nu^{\Phi,\Psi}(\mathbf n)\,
d(\widetilde\nu-\nu)(\mathbf n).
$$
Hence, $\Gamma_\nu^{\Phi,\Psi}$ is the marginal system-level cost of placing additional terminal probability mass at state $\mathbf n$. If $\Psi$ is symmetric, it holds that 
$$
\Gamma_{\nu}^{\Phi,\Psi}(\mathbf n)
=
\frac{1}{B}\Phi(\mathbf n)
+
\frac{2(B-1)}{B}
\int_{\mathcal N_{\ell,d}}\Psi(\mathbf n,\mathbf n')\,d\nu(\mathbf n').
$$

Since $\mathcal N_{\ell,d}$ is finite, we identify terminal laws with their probability mass functions. Denote by 
$
\mathcal U_m:=\{u\subseteq[d]: |u|=m\}$ and
$p_m(u):=\binom{d}{m}^{-1}$ for $u\in\mathcal U_m.
$
Recall that $\mathcal A(u)$ represents the admissible action set. 
For any finite-valued terminal cost
$h:\mathcal N_{\ell,d}\to\mathbb R$, 
we define 
the \textit{random-opportunity Bellman recursion}
$$
\mathcal V_\ell^h(\mathbf n):=h(\mathbf n), \qquad
 \mathbf n\in\mathcal N_{\ell,d},
$$
and for $t=0, \ldots, \ell-1$,
\begin{align}
\mathcal V_t^h(\mathbf n)
:=
\sum_{u\in\mathcal U_m}p_m(u)
\min_{j\in\mathcal A(u)}
\mathcal V_{t+1}^h(\mathbf n+e_j),
\qquad \mathbf n\in\mathcal N_{t,d}.
\label{eq:linear-terminal-bellman}
\end{align}

\begin{lemma}
\label{lem:cost-to-go-bellman}
For any $t = 0, \ldots, \ell - 1$, 
and every $N_t=\mathbf n \in \mathcal{N}_{t,d}$, 
it holds that 
$$
\mathcal V_{t}^h(\mathbf n)
=
\inf_{\pi_{(t+1):\ell}}
\mathbb E_\pi
\left[
h(N_\ell)
\mid
N_{t}=\mathbf n
\right],
$$
where the infimum is taken over all admissible nonanticipative split policies from depths $t+2,\ldots,\ell$. 
In particular,  $
\mathcal V_0^h(\mathbf 0)
=
\inf_{\pi\in\Pi_\ell}
\mathbb E_\pi[h(N_\ell)].
$
\end{lemma}


Lemma~\ref{lem:cost-to-go-bellman} shows that $\mathcal V_t^h$ is the Bellman value function of the corresponding linear terminal-cost control problem. 



\begin{theorem}[Bellman optimality conditions for marginal terminal cost]
\label{thm:marginal-value-bellman}
Fix $\ell$, $B$, $\Phi$, and $\Psi$, and assume that $\Phi$ and
$\Psi$ are finite-valued. Then it holds that 
\begin{enumerate}

\item[(i)] If $\nu^\star\in\mathfrak V_\ell$ is a local minimizer of
$\mathcal J_{\ell,B}^{\Phi,\Psi}$ over $\mathfrak V_\ell$, we have that 
\begin{align}\label{eq:first-order-terminal-law-condition}
\nu^\star
\in
\arg\min_{\nu\in\mathfrak V_\ell}
\int_{\mathcal N_{\ell,d}}
\Gamma_{\nu^\star}^{\Phi,\Psi}(\mathbf n)\,d\nu(\mathbf n).
\end{align}

\item[(ii)] Let $\pi^\star\in\Pi_\ell$ satisfy that $\nu_{\pi^\star,\ell}=\nu^\star$, and assume that $\nu^\star$ is a local
minimizer of $\mathcal J_{\ell,B}^{\Phi,\Psi}$ over $\mathfrak V_\ell$.
Then $\pi^\star$ minimizes of the linear terminal-cost objective
$$
\mathbb E_{\pi^{\star}}
\big[
\Gamma_{\nu^\star}^{\Phi,\Psi}(N_\ell)
\big] \leq 
\mathbb E_\pi
\big[
\Gamma_{\nu^\star}^{\Phi,\Psi}(N_\ell)
\big], \qquad \forall \pi \in \Pi_{\ell}.
$$
Consequently, we have that for each $t=0,\ldots,\ell-1$,
\begin{align}
\operatorname{supp}
\pi^\star_{t+1}(\cdot\mid H_t,U_{t+1})
\subseteq
\arg\min_{j\in\mathcal A(U_{t+1})}
\mathcal V_{t+1}^{\Gamma_{\nu^\star}^{\Phi,\Psi}}(N_t+e_j),
\qquad
\P_{\pi^\star}\text{-a.s.}
\label{eq:marginal-bellman-support}
\end{align}
Here, $\operatorname{supp} \mu $ denotes the support of a probability measure $\mu$. 

\item[(iii)] Conversely, assume that $\pi^\star\in\Pi_\ell$ induces
$\nu^\star=\nu_{\pi^\star,\ell}$ and satisfies
\eqref{eq:marginal-bellman-support} for all $t=0,\ldots,\ell-1$. If
$\mathcal J_{\ell,B}^{\Phi,\Psi}$ is convex on $\mathfrak V_\ell$, we have that 
$$
\mathcal J_{\ell,B}^{\Phi,\Psi}(\nu^\star)
=
\inf_{\nu\in\mathfrak V_\ell}
\mathcal J_{\ell,B}^{\Phi,\Psi}(\nu),
\text{ and }
\mathcal J_{\ell,B}^{\Phi,\Psi}(\pi^\star)
=
\inf_{\pi\in\Pi_\ell}
\mathcal J_{\ell,B}^{\Phi,\Psi}(\pi).
$$
\end{enumerate}
Moreover, the convexity condition on $\mathcal J_{\ell,B}^{\Phi,\Psi}$ holds whenever the
symmetrized kernel
$
Q^\Psi(\mathbf n,\mathbf n')
:=
\frac{1}{2}\{\Psi(\mathbf n,\mathbf n')+\Psi(\mathbf n',\mathbf n)\}
$
satisfies that 
$$
\sum_{\mathbf n,\mathbf n'\in\mathcal N_{\ell,d}}
a(\mathbf n)Q^\Psi(\mathbf n,\mathbf n')a(\mathbf n')
\ge 0
$$
for each signed vector
$
a\in
\operatorname{span}
\{\nu-\widetilde\nu:\nu,\widetilde\nu\in\mathfrak V_\ell\}.
$
\end{theorem}

Theorem~\ref{thm:marginal-value-bellman} provides the optimality system for the ensemble-design problem \eqref{eq:ensemble-design-problem} and clarifies the technical conditions under which that system is valid. 
The finite-valued requirement on $\Phi$ and $\Psi$ ensures that all first variations, marginal terminal costs, and Bellman values are well-defined on the finite terminal state space. 
The positive-semidefinite condition on $Q^\Psi$ is precisely the convexity condition for the quadratic interaction term along feasible terminal-law directions. 
The finite-valued condition is employed to obtain the first-order Bellman necessity, whereas the convexity condition is needed only when converting that necessity into a global optimality certificate. 

Given a candidate terminal law $\nu$, the marginal terminal cost $\Gamma_{\nu}^{\Phi,\Psi}$ defines a linear terminal-cost control problem, and the Bellman recursion \eqref{eq:linear-terminal-bellman} identifies the associated optimal split policy. Therefore, optimality of ensemble design requires a consistent condition: the Bellman policy must reproduce the same terminal law $\nu$ that generated the marginal terminal cost. 

In contrast, greedy CART optimizes a one-step impurity score rather than the marginal continuation value induced by $\Gamma_{\nu}^{\Phi,\Psi}$. 
Therefore, the local CART split rule and the terminal-law ensemble objective are governed by different Bellman equations. Such distinction is the main  reason why a locally greedy tree can be suboptimal for a forest-level objective.

\subsection{Verification of Assumption~\ref{assm:etareq} for $s=2$ and $s=3$ under sufficiently large $\gamma$}
\label{subsec:eta-s2-s3}

In this subsection, we verify that Assumption~\ref{assm:etareq} can be satisfied for the population CART under the sparse linear model with $s=2$ and $s=3$, by taking the subsampling ratio $\gamma=m/d$ sufficiently large.
Throughout the section, we will use the simplified expression
\begin{equation}\label{eq:cstar-clean-v1}
c_\star
=
\frac{1}{s(s-1)^{3/2}}
\sum_{k=2}^s (k-1)\P(K=k\mid K\ge1)
\end{equation}
with $K\sim{\rm Hypergeometric}(d,s,m)$. Notice that $\kappa(\pi_{\gr}) \ge c_\star$. Then it suffices to show that $c_\star > (1 - 1/s)\eta_{\rm req}/4 $.


\noindent\textbf{Case $s=2$.}
For this case, we have $K\in\{0,1,2\}$ and \eqref{eq:cstar-clean-v1} reduces to
$$
c_\star=\frac12\P(K=2\mid K\ge1).
$$
As $\gamma = m /d \to1$, 
it holds that   
$\P(K=2\mid K\ge1)\to1$ and thus $c_\star \to 1/2$.

Meanwhile, we have that $\eta_{\rm req}=\max\{2\log2,\log2,{\log2}/{\sqrt{2}}\}=2\log2$ and 
$$
\frac{1 - (1/s)}{4} \eta_{\rm req} =
\frac{(1 -1/2) (2\log2)}{4}
=
\frac{\log2}{4}
< 1/2.
$$
Hence, it follows that
\[
\kappa(\pi_{\gr}) \ge c_\star > \frac{1 - (1/s)}{4}\eta_{\rm req}
\]
for all sufficiently large $\gamma$, and thus Assumption \ref{assm:etareq} holds.

\medskip
\noindent\textbf{Case $s=3$.}
For this case, we have $K\in\{0,1,2,3\}$ and \eqref{eq:cstar-clean-v1} reads
$$
c_\star=\frac{1}{3\cdot 2^{3/2}}
\Big(\P(K=2\mid K\ge1)+2\P(K=3\mid K\ge1)\Big).
$$
As $\gamma\to1$, it holds that $\P(K=3\mid K\ge1)\to1$ and $\P(K=2\mid K\ge1)\to0$,
so $c_\star\to {1}/{3\cdot 2^{1/2}}$.

Meanwhile, we have that $\eta_{\rm req}= \max\{2\log2,
\sqrt{3}\log2 / 2\} 
= 2\log 2$.
Then it follows that 
$$
\frac{1 - (1/s)}{4} \eta_{\rm req} =
\frac{(1-1/3 )(2\log2)}{4}
=
\frac{\log2}{3} < \frac{1}{3\sqrt{2}}.
$$
Thus, we can obtain that 
\[
\kappa(\pi_{\gr}) \ge c_\star > \frac{1 - (1/s)}{4}\eta_{\rm req}
\]
for all sufficiently large $\gamma$, so Assumption \ref{assm:etareq} also holds.

\section{Proofs of main results} \label{new.sec.mainproofs}

\subsection{Proof of Theorem \ref{thm:firstorder}} \label{proof.thm:firstorder}


The proof of this theorem will separate the stochastic clock from the balancing dynamics. The main technical point is that the greedy informative-time rule creates a negative drift for the imbalance $W_n$, but that drift must still be converted into an almost-sure ratio limit. To this end, we exploit a Robbins--Siegmund argument in informative time and then compose the resulting limit with the binomial informative clock $M_t$.

We begin with the informative-time component. Recall from Section~\ref{new.sec.oppotime} 
that $Z_n$ is the informative-time counting process, $\Delta_{n,j}=Z_{n,j}-n/s \1_s$ is the informative imbalance, and $V_n=\sum_{j=1}^s \Delta_{n,j}^2$ and $W_n=\sqrt{V_n}$ are the associated Lyapunov functions. The first claim is that 
\begin{equation}\label{eq:Wn-over-n-to-0-both}
\frac{W_n}{n}\xrightarrow{\mathrm{a.s.}}0
\qquad\text{or equivalently,}\qquad
\max_{1\le i,j\le s}\Big|\frac{Z_{n,i}}{n}-\frac{Z_{n,j}}{n}\Big|\xrightarrow{\mathrm{a.s.}}0.
\end{equation}
To establish such claim, let us start from the exact one-step identity in
Proposition~\ref{prop:cart-onestep}
$$
V_{n+1}=V_n+2\Delta_{n,J_{n+1}}+\Big(1-\frac1s\Big).
$$
Taking the conditional expectation given $\cF_n$ and applying the negative drift
bound in Lemma~\ref{lem:negdrift}, we can deduce that 
\begin{align}
\E[V_{n+1}\mid \cF_n]
&=
V_n+2\E[\Delta_{n,J_{n+1}}\mid \cF_n]+\Big(1-\frac1s\Big)\notag\\
&\le
V_n-2c_\star W_n+\Big(1-\frac1s\Big).
\label{eq:V_RS}
\end{align}

Define the nonnegative process
$Y_n:=\frac{V_n}{(n+1)^2}$ for $n\ge0$.
We aim to derive a Robbins--Siegmund-type recursion for $(Y_n)$. To obtain the desired recursion, we divide \eqref{eq:V_RS} by $(n+2)^2$ and observe that
$$
\frac{V_n}{(n+2)^2}
=
\Big(\frac{n+1}{n+2}\Big)^2 \frac{V_n}{(n+1)^2}
=
\Big(\frac{n+1}{n+2}\Big)^2 Y_n
 \le 
Y_n.
$$
It follows that for all $n\ge0$,
\begin{equation}\label{eq:Y_RS_raw}
\E[Y_{n+1}\mid \cF_n]
=
\E \left[\frac{V_{n+1}}{(n+2)^2}\middle|\cF_n\right]
\le
Y_n
-\frac{2c_\star}{(n+2)^2}W_n
+\frac{1-1/s}{(n+2)^2}.
\end{equation}
The inequality above has the form that is needed for the following Robbins--Siegmund almost-supermartingale theorem in \cite{RobbinsSiegmund1971}.

\begin{lemma}[Robbins--Siegmund framework]\label{lem:RS}
Let $(Y_n)_{n\ge0}$ be a nonnegative adapted process satisfying that 
$$
\E[Y_{n+1}\mid \cF_n]
\le
Y_n + \xi_n - \eta_n,
$$
where $(\xi_n)$ and $(\eta_n)$ are nonnegative sequences with $\sum_n \xi_n<\infty$ almost surely.
Then $(Y_n)$ converges almost surely to a finite random limit, and
$\sum_n \eta_n<\infty$ almost surely.
\end{lemma}

Lemma~\ref{lem:RS} above is the mechanism that converts the negative drift in \eqref{eq:Y_RS_raw} into both convergence of the normalized Lyapunov process and summability of the drift term. Inequality~\eqref{eq:Y_RS_raw} is precisely of the Robbins--Siegmund form in
Lemma~\ref{lem:RS}, with the choices of 
$$
\xi_n := \frac{1-1/s}{(n+2)^2},
\qquad
\eta_n := \frac{2c_\star}{(n+2)^2}W_n.
$$
Since $\sum_{n\ge0}\xi_n<\infty$, an application of Lemma~\ref{lem:RS} gives that $(Y_n)$
converges almost surely to a finite random limit $Y_\infty$, and further 
\begin{equation}\label{eq:sum-eta-finite}
\sum_{n=0}^\infty \frac{W_n}{(n+2)^2}<\infty
\qquad\text{almost surely.}
\end{equation}

We next convert the summability conclusion into informative-time equalization. The key identities are
$$
Y_n := \frac{V_n}{(n+1)^2}
\quad\text{and the summability}\quad
\sum_{n=0}^\infty \frac{W_n}{(n+2)^2}<\infty.
$$
Since $W_n=\sqrt{V_n}$, it holds that 
$$
\sqrt{Y_n}=\frac{W_n}{n+1}.
$$
Moreover, we have that for all $n\ge0$,
$$
\frac{W_n}{(n+2)^2}
=
\frac{n+1}{(n+2)^2}\cdot \frac{W_n}{n+1}
=
\frac{n+1}{(n+2)^2}\sqrt{Y_n}
\asymp \frac{\sqrt{Y_n}}{n+1},
$$
and thus the summability $\sum_{n\ge0} W_n/(n+2)^2<\infty$ entails that 
\begin{equation}\label{eq:sqrtY-harmonic-summable}
\sum_{n=0}^\infty \frac{\sqrt{Y_n}}{n+1}<\infty
\qquad\text{almost surely.}
\end{equation}

On the other hand, by invoking the Robbins--Siegmund argument established earlier,
$Y_n\to Y_\infty$ almost surely for some finite random variable $Y_\infty\ge0$.
We claim that $Y_\infty=0$ almost surely.
Assume \textit{on the contrary} that $\P(Y_\infty>0)>0$.
Then there exist some $\varepsilon>0$ and a (random) index $n_0$ such that
$\sqrt{Y_n}\ge \varepsilon$ for all $n\ge n_0$ on an event of positive probability.
On such event, it holds that 
$$
\sum_{n=0}^\infty \frac{\sqrt{Y_n}}{n+1}
 \ge 
\sum_{n=n_0}^\infty \frac{\varepsilon}{n+1}
 = \infty,
$$
contradicting \eqref{eq:sqrtY-harmonic-summable}.
Hence, it follows that $Y_\infty=0$ almost surely, and thus
\begin{equation}\label{eq:Wn-over-n-to-0}
\frac{W_n}{n+1}=\sqrt{Y_n}\xrightarrow{\mathrm{a.s.}}0.
\end{equation}

Further, using $\sum_{j=1}^s \Delta_{n,j}=0$ and bound
$\max_{1\le j\le s}|\Delta_{n,j}|\le \|\Delta_n\|_2=W_n$, we can show that 
$$
\max_{1\le j\le s}\Big|\frac{Z_{n,j}}{n}-\frac1s\Big|
=
\max_{1\le j\le s}\Big|\frac{\Delta_{n,j}}{n}\Big|
\le
\frac{W_n}{n}
\xrightarrow{\mathrm{a.s.}}0.
$$
This establishes the desired informative-time equalization
\begin{equation}\label{eq:Zn-limit}
\frac{Z_{n,j}}{n}\xrightarrow{\mathrm{a.s.}}\frac1s,
\qquad j=1,\dots,s.
\end{equation}

It remains to transfer the informative-time limit to the original depth-time process through the clock $M_t$.
Let us fix $j\in S$. In view of \eqref{eq:clock-link}, it holds that for all $t\ge1$,
$$
N_{t,j}=Z_{M_t,j}.
$$
On event $\{M_t\to\infty\}$ (which holds almost surely since $q>0$),
we can write
$$
\frac{N_{t,j}}{t}
=
\frac{Z_{M_t,j}}{M_t}\cdot \frac{M_t}{t}.
$$
Recall that $M_t\sim\mathrm{Binomial}(t,q)$ and $M_t/t\to q$ almost surely.
With \eqref{eq:Zn-limit} applied along the (random) subsequence $n=M_t$, it follows that 
$$
\frac{Z_{M_t,j}}{M_t}\xrightarrow{\mathrm{a.s.}}\frac1s.
$$
Therefore, we can obtain that 
$$
\frac{N_{t,j}}{t}
=
\frac{Z_{M_t,j}}{M_t}\cdot \frac{M_t}{t}
\xrightarrow{\mathrm{a.s.}}
\frac1s\cdot q
=\frac{q}{s},
\qquad j\in S,
$$
which establishes the first-order limit on the informative block $S$. This completes the proof of Theorem \ref{thm:firstorder}.

\subsection{Proof of Theorem \ref{thm:expcomp}} \label{proof.thm:expcomp}



The proof of this theorem will turn the negative drift of the quadratic imbalance into a uniform exponential moment bound for $W_n$. The technical challenge is that the drift is useful only when $W_n$ is large, while the increments are controlled only through a deterministic bound. We thus combine a large-state contraction, Hoeffding's lemma for bounded increments, and a small-state envelope to obtain a global geometric drift for $e^{\eta W_n}$.

The exponential-rate condition in Theorem~\ref{thm:expcomp} is the analytic source of Assumption~\ref{assm:etareq}. The admissible range for $\eta$ balances the negative drift coefficient $c_\star$ against the Hoeffding penalty caused by bounded increments of size $a$. Assumption~\ref{assm:etareq} later selects the particular exponential rates needed by the MSE functionals and requires those rates to lie within this stable range.

Let us first recall that
$$
a = \sqrt{1 - \frac{1}{s}}, \qquad b = a^2 = 1 - \frac{1}{s}.
$$
We begin with identifying the drift inequality for $V_n$ that drives the entire argument. An application of Proposition~\ref{prop:cart-onestep} leads to 
$$
V_{n+1}-V_n
=
2\Delta_{n,J_{n+1}}+\Bigl(1-\frac1s\Bigr).
$$
Taking the conditional expectation given $\cF_n$ and resorting to Lemma~\ref{lem:negdrift}, we can show that 
$$
\E[V_{n+1}-V_n\mid \cF_n]
=
2\E[\Delta_{n,J_{n+1}}\mid \cF_n]+\Bigl(1-\frac1s\Bigr)
\le
-2c_\star W_n + b.
$$
Hence, it holds that for all $n\ge0$,
\begin{equation}\label{eq:V-drift-plus}
\E[V_{n+1}-V_n\mid \cF_n]\le -2c_\star W_n + b.
\end{equation}

We next focus on the behavior of $W_n$. Observe that
$$
W_{n+1}-W_n
=
\frac{V_{n+1}-V_n}{W_{n+1}+W_n}.
$$
On event $\{W_n>0\}$, we have that $W_{n+1}+W_n\ge W_n$ and thus
$$
(W_{n+1}-W_n)\1\{W_n>0\}
\le
\frac{V_{n+1}-V_n}{W_n}\1\{W_n>0\}.
$$
Since $W_n \in \cF_n$, taking the $\cF_n$-conditional expectation and applying \eqref{eq:V-drift-plus} give that 
$$
\E[W_{n+1}-W_n\mid \cF_n]\1\{W_n>0\}
\le
\Bigl(-2c_\star+\frac{b}{W_n}\Bigr)\1\{W_n>0\}.
$$
Since $0 < \eta < 4c_\star/a^2$, there exists a $\delta\in(0,1)$ such that 
$$
0 < \eta < \frac{2(2 - \delta) c_\star}{a^2}.
$$

Let us set a deterministic threshold
$R:=b/{\delta c_\star}$.
Then on event $\{W_n\ge R\}$, it holds that $b/W_n\le \delta c_\star$, and thus 
\begin{equation}\label{eq:W-drift-large}
\E[W_{n+1}-W_n\mid \cF_n]\le -\mu
\qquad \text{on }\{W_n\ge R\},
\end{equation}
where $\mu:= (2-\delta)c_\star>0$. Denote by $\Delta W_n:=W_{n+1}-W_n$. Equation~\eqref{eq:bounded-jump} in Proposition~\ref{prop:cart-onestep} implies that 
$|\Delta W_n| = |W_{n+1}-W_n|\le a$ almost surely.
We will bound the conditional moment generating function of $\Delta W_n$ using only bounded increments. This step is where the bounded-jump estimate enters; for any $\eta>0$, we have that 
\begin{equation}
\E \left[e^{\eta \Delta W_n}\mid \cF_n\right]
=
\exp \Big(\eta\E[\Delta W_n\mid \cF_n]\Big)
\E \left[\exp \Big(\eta\big(\Delta W_n-\E[\Delta W_n\mid \cF_n]\big)\Big)\middle|\cF_n\right].
\label{eq:mgf-factor}
\end{equation}
Since $\Delta W_n\in[-a,a]$ almost surely, the centered increment
$\Delta W_n-\E[\Delta W_n\mid \cF_n]$ is supported on an interval of length at most
$2a$. Then an application of the conditional form of Hoeffding's lemma results in 
\begin{equation}\label{eq:hoeffding}
\E \left[\exp \Big(\eta\big(\Delta W_n-\E[\Delta W_n\mid \cF_n]\big)\Big)\middle|\cF_n\right]
\le
\exp \Big(\frac{\eta^2 a^2}{2}\Big).
\end{equation}
Combining \eqref{eq:mgf-factor}--\eqref{eq:hoeffding}, we can obtain the conditional moment generating function bound
\begin{equation}\label{eq:mgf-bound}
\E \left[e^{\eta \Delta W_n}\mid \cF_n\right]
\le
\exp \Big(\eta\E[\Delta W_n\mid \cF_n]+\frac{\eta^2 a^2}{2}\Big).
\end{equation}

Multiplying \eqref{eq:mgf-bound} by $e^{\eta W_n}$ yields that 
$$
\E[e^{\eta W_{n+1}}\mid \cF_n]
=
e^{\eta W_n}\E[e^{\eta \Delta W_n}\mid \cF_n]
\le
e^{\eta W_n}\exp \Big(\eta\E[\Delta W_n\mid \cF_n]+\frac{\eta^2 a^2}{2}\Big).
$$
On the large set $\{W_n\ge R\}$, it follows from \eqref{eq:W-drift-large} that $\E[\Delta W_n\mid \cF_n]\le -\mu$. Hence, we have that on event $\{W_n\ge R\}$,
\begin{equation}\label{eq:large-contract-B}
\E[e^{\eta W_{n+1}}\mid \cF_n]
\le
\rho e^{\eta W_n},
\qquad
\rho:=\exp \Big(-\eta\mu+\frac{\eta^2 a^2}{2}\Big).
\end{equation}
Recall that $\eta$ satisfies that 
$$
0<\eta<\frac{2\mu}{a^2}
=
\frac{2(2-\delta)c_\star}{a^2}.
$$
Consequently, it holds that $-\eta\mu+\eta^2 a^2/2<0$ and $\rho\in(0,1)$.

Meanwhile, on the complementary set $\{W_n<R\}$, the jump bound \eqref{eq:bounded-jump} entails that $W_{n+1}\le W_n+a<R+a$, and thus
\begin{equation}\label{eq:small-bound-B}
e^{\eta W_{n+1}}\1\{W_n<R\}\le e^{\eta(R+a)}.
\end{equation}
Combining the contraction on the large set with the deterministic envelope on the small set yields the global exponential drift that for all $n\ge0$,
\begin{equation}\label{eq:global-exp-drift-B}
\E[e^{\eta W_{n+1}}\mid \cF_n]
\le
\rho e^{\eta W_n} + B,
\qquad
B:=e^{\eta(R+a)}.
\end{equation}
Taking expectations in \eqref{eq:global-exp-drift-B}, we can show that 
$$
\E e^{\eta W_{n+1}} \le \rho\E e^{\eta W_n}+B.
$$
Iterating this process gives that 
$$
\E e^{\eta W_n}
\le
\rho^n\E e^{\eta W_0} + B\sum_{k=0}^{n-1}\rho^k
\le
\E e^{\eta W_0}+\frac{B}{1-\rho}.
$$
Since $Z_0\equiv0$ implies that $V_0=0$ and $W_0=0$, it holds that $\E e^{\eta W_0}=1$.
Therefore, we can obtain that 
$$
\sup_{n\ge0}\E e^{\eta W_n}
\le
1+\frac{B}{1-\rho}
<\infty,
$$
which establishes \eqref{eq:expcomp}. This concludes the proof of Theorem \ref{thm:expcomp}.













\subsection{Proof of Theorem~\ref{thm:cart-mse}} \label{proof.thm:cart-mse}


The proof of this theorem will translate the abstract risk representation into explicit population-CART rates. The main difficulty is that the risk terms depend on raw terminal counts, whereas CART balances the informative-time counts only up to stochastic fluctuations. We will isolate the binomial clock, control the informative-time error by exponential compression, and then evaluate the remaining clock and noninformative terms through exact binomial and multinomial identities.

More precisely, we start from the general formula in Theorem~\ref{thm:mse}. We then show how the three functionals involving the informative block can be uniformly controlled using Theorem~\ref{thm:expcomp}, Assumption~\ref{assm:etareq}, and the exact identities established in Proposition~\ref{prop:poisson-functionals}. Throughout the proof, constants denoted as $C$ may change from line to line and depend only on $(d,s,\gamma)$, but not on $\ell$ and $n$.

Let us first recall the clock representation in \eqref{eq:clock-link}. For each $j\in S$, it holds that 
$
N_{\ell,j} = Z_{M_\ell,j}
$,
where $M_\ell$ is the informative clock and $(Z_n)_{n\ge0}$ is the
informative-time counting process.
Writing
$$
Z_{n,j} = \frac{n}{s} + \Delta_{n,j},
\qquad
W_n := \|\Delta_n\|_2,
$$
we have the deterministic bound
$
|\Delta_{n,j}| \le W_n,
\ j=1,\dots,s.
$
Under Assumption~\ref{assm:etareq}, an application of Theorem~\ref{thm:expcomp} gives that 
\begin{equation}\label{eq:app-exp-moment}
C:= \sup_{n\ge0}\E\exp\big(\eta_{\rm req} W_n\big) < \infty,
\end{equation}
where
$
\eta_{\rm req} = \max\{2\log 2,\ \tfrac{\log 2}{2}\sqrt{s}\}.
$
Consequently, we have that for each $\lambda\in(0,\eta_{\rm req}]$,
\begin{equation}\label{eq:app-exp-moment-lambda}
\sup_{n\ge0}\E e^{\lambda W_n} < \infty.
\end{equation}
Meanwhile, the informative clock $M_\ell$ depends only on the subsampling events
$\{U_t\}$ and is independent of the informative-time dynamics $(Z_n)$.
Hence, $M_\ell$ is independent of $(W_n)$, and we can freely condition on
$M_\ell$ in the remaining proof.

Here, Assumption~\ref{assm:etareq} is used as a moment-threshold condition instead of as a new structural assumption on the tree. It guarantees that Theorem~\ref{thm:expcomp} is available at every exponential rate generated by the single-tree bias, cross-tree bias, and overlap terms, so the informative-time imbalance errors can be absorbed into constants independent of depth $\ell$.

\smallskip
\noindent\textbf{Step 1: Single-tree bias term.} The single-tree bias contribution in Theorem~\ref{thm:mse} is the following expectation over terminal informative counts
$$
\sum_{j\in S}\beta_j^2\E\bigl[2^{-2N_{\ell,j}}\bigr].
$$
Let us fix $j\in S$.
Using the informative-time embedding, we can deduce that 
$$
2^{-2N_{\ell,j}}
=
2^{-2Z_{M_\ell,j}}
=
2^{-2(M_\ell/s+\Delta_{M_\ell,j})}
=
2^{-2M_\ell/s}2^{-2\Delta_{M_\ell,j}}.
$$
It follows from $2^{-2x}=\exp(-(2\log 2)x)$ and $|\Delta_{M_\ell,j}|\le W_{M_\ell}$ that 
$$
2^{-2N_{\ell,j}}
\le
2^{-2M_\ell/s}
\exp\big((2\log 2)W_{M_\ell}\big).
$$
Taking expectations and conditioning on $M_\ell$ yield that 
$$
\E\bigl[2^{-2N_{\ell,j}}\bigr]
\le
\E \left[
2^{-2M_\ell/s}
\E\bigl(e^{(2\log 2)W_{M_\ell}}\mid M_\ell\bigr)
\right]
\le
C\E\bigl[2^{-2M_\ell/s}\bigr],
$$
where constant $C<\infty$ follows from
\eqref{eq:app-exp-moment}. 

Since $N_{\ell,j}$ and $M_{\ell}/s$ are exchangeable in the above argument, it also holds that  
$$
2^{-2M_\ell/s}
\le
2^{-2N_{\ell,j}}
\exp\big((2\log 2)W_{M_\ell}\big),
$$
and thus 
$$
\E\bigl[2^{-2N_{\ell,j}}\bigr] \ge \frac{1}{C} \E\bigl[2^{-2M_\ell/s}\bigr]
$$
with the same constant $C$ mentioned before.
Since $M_\ell\sim\mathrm{Binomial}(\ell,q)$, an application of 
Lemma~\ref{lem:binom-pgf} with the choices of $p=q$ and $r= 2^{-2/s}$ leads to 
$$
\E\bigl[2^{-2M_\ell/s}\bigr]
=
\bigl(1-q+q4^{-1/s}\bigr)^\ell.
$$
This establishes the stated bound for the single-tree bias term.

\smallskip
\noindent\textbf{Step 2: Cross-tree bias term.} The cross-tree bias contribution in Theorem~\ref{thm:mse} has the similar two-tree form 
$$
\sum_{j\in S}\beta_j^2
\E\bigl[2^{-2\max\{N_{\ell,j},N'_{\ell,j}\}}\bigr],
$$
where $N'_\ell$ represents an independent copy of $N_\ell$.
For $j\in S$, denote by 
$$
N_{\ell,j} = \frac{M_\ell}{s}+\Delta_{M_\ell,j}
\qquad
N'_{\ell,j} = \frac{M'_\ell}{s}+\Delta'_{M'_\ell,j},
$$
with $(M_\ell,\Delta)$ independent of $(M'_\ell,\Delta')$.
Using the identity
$
\max(x,y)=\tfrac12(x+y+|x-y|)
$,
we can show that 
$$
2^{-2\max\{N_{\ell,j},N'_{\ell,j}\}}
=
2^{-(N_{\ell,j}+N'_{\ell,j})}
2^{-|N_{\ell,j}-N'_{\ell,j}|}.
$$

For the first factor, it holds that 
\begin{align*}
2^{-(N_{\ell,j}+N'_{\ell,j})}
&=
2^{-(M_\ell+M'_\ell)/s}
\exp\bigl(-(\log 2)(\Delta_{M_\ell,j}+\Delta'_{M'_\ell,j})\bigr)
\\
&\le
2^{-(M_\ell+M'_\ell)/s}
\exp\bigl((\log 2)(W_{M_\ell}+W'_{M'_\ell})\bigr).
\end{align*}
For the second factor, we have that 
\begin{align*}
|N_{\ell,j}-N'_{\ell,j}|
&\le
\frac{|M_\ell-M'_\ell|}{s}
+
|\Delta_{M_\ell,j}|+|\Delta'_{M'_\ell,j}|\\
&\le
\frac{|M_\ell-M'_\ell|}{s}
+
W_{M_\ell}+W'_{M'_\ell},
\end{align*}
and thus
$$
2^{-|N_{\ell,j}-N'_{\ell,j}|}
\le
2^{-|M_\ell-M'_\ell|/s}
\exp\bigl((\log 2)(W_{M_\ell}+W'_{M'_\ell})\bigr).
$$

Combining the two bounds above gives that 
\begin{align*}
2^{-2\max\{N_{\ell,j},N'_{\ell,j}\}}
& \le
2^{-(M_\ell+M'_\ell)/s}
2^{-|M_\ell-M'_\ell|/s}
\exp\bigl(2(\log 2)(W_{M_\ell}+W'_{M'_\ell})\bigr)\\
& \le 2^{-2\max\{M_\ell, M'_\ell\}/s} 
\exp\bigl(2(\log 2)(W_{M_\ell}+W'_{M'_\ell})\bigr).
\end{align*}
Taking the expectations conditional on $M_\ell$ and using
\eqref{eq:app-exp-moment-lambda} with $\lambda=2\log 2$, we can deduce that 
$$
\E\bigl[2^{-2\max\{N_{\ell,j},N'_{\ell,j}\}}\bigr]
\le
C
\E \left[
2^{-2\max\{M_\ell, M'_\ell\}/s}
\right],
$$
where constant $C$ follows from \eqref{eq:app-exp-moment}. 
Thanks to the exchangeability of $N_{\ell,j}$ and $M_\ell$ in the above argument, we also have that 
$$
\E\bigl[2^{-2\max\{N_{\ell,j},N'_{\ell,j}\}}\bigr]
\ge
\frac{1}{C}
\E \left[
2^{-2\max\{M_\ell, M'_\ell\}/s}
\right].
$$

The remaining expectation depends only on the binomial clocks
$M_\ell,M'_\ell\stackrel{\rm i.i.d.}{\sim}{\rm Binomial}(\ell,q)$.
Applying part (i) of Proposition~\ref{prop:poisson-functionals} with the choices of $p=q$ and $r=2^{-2/s}$ yields the exact clock representation
$$
\E\bigl[2^{-2\max\{M_\ell,M'_\ell\}/s}\bigr]
=
\bigl(1-q+q2^{-1/s}\bigr)^{2\ell}
F_{\ell,2^{-2/s}}\left(\frac{q 2^{-1/s}}{1-q + q 2^{-1/s}}\right),
$$
where $F_{\ell,r}$ is the Poisson-kernel attenuation functional defined in \eqref{eq:Flr}.
This establishes the stated cross-tree bias term.

\smallskip
\noindent\textbf{Step 3: Cross-tree variance term.} The remaining term is the cross-tree variance contribution in Theorem~\ref{thm:mse}
$$
\E \left[
2^{\sum_{j=1}^d \min\{N_{\ell,j},N'_{\ell,j}\}}
\right].
$$
It follows from identities 
$$
\sum_{j=1}^d \min\{x_j,y_j\}
=
\frac12\sum_{j=1}^d(x_j+y_j)
-
\frac12\|x-y\|_1
$$
and $\sum_{j=1}^d N_{\ell,j}=\sum_{j=1}^d N'_{\ell,j}=\ell$ that 
\begin{equation}\label{eq:app-cv-basic}
2^{\sum_{j=1}^d \min\{N_{\ell,j},N'_{\ell,j}\}}
=
2^\ell2^{-\frac12\|N_\ell-N'_\ell\|_1}.
\end{equation}
Let us set 
$$
r:=2^{-1/2}\in(0,1),
\qquad\text{so that}\qquad
2^{-\frac12\|N_\ell-N'_\ell\|_1}=r^{\|N_\ell-N'_\ell\|_1}.
$$

We next decompose $N_t$ on $S$ into an equilibrium part and an informative fluctuation.
Let us define the $(d-s+1)$-dimensional aggregated count vectors
$$
\widetilde N_\ell
:=
\bigl(M_\ell,\ (N_{\ell,j})_{j\in S^c}\bigr),
\qquad
\widetilde N'_\ell
:=
\bigl(M'_\ell,\ (N'_{\ell,j})_{j\in S^c}\bigr),
$$
and introduce the corresponding ``equilibrium'' full $d$-vectors
$$
\bar N_\ell
:=
\Bigl(\underbrace{M_\ell/s,\dots,M_\ell/s}_{s\ \text{times}},\ (N_{\ell,j})_{j\in S^c}\Bigr),
\qquad
\bar N'_\ell
:=
\Bigl(\underbrace{M'_\ell/s,\dots,M'_\ell/s}_{s\ \text{times}},\ (N'_{\ell,j})_{j\in S^c}\Bigr).
$$
For each $j\in S$, in light of \eqref{eq:clock-link} we have that 
$$
N_{\ell,j}
=
Z_{M_\ell,j}
=
\frac{M_\ell}{s}+\Delta_{M_\ell,j},
\qquad
N'_{\ell,j}
=
Z'_{M'_\ell,j}
=
\frac{M'_\ell}{s}+\Delta'_{M'_\ell,j}.
$$
Hence, we can write that 
$$
N_\ell=\bar N_\ell+E_\ell,
\qquad
N'_\ell=\bar N'_\ell+E'_\ell,
$$
where the error vectors satisfy
$$
(E_{\ell,j})_{j\in S}=(\Delta_{M_\ell,j})_{j\in S},
\quad
(E_{\ell,j})_{j\in S^c}\equiv 0,
\qquad
(E'_{\ell,j})_{j\in S}=(\Delta'_{M'_\ell,j})_{j\in S},
\quad
(E'_{\ell,j})_{j\in S^c}\equiv 0.
$$

With the aid of the triangle inequality, we can deduce that 
$$
\|N_\ell-N'_\ell\|_1
=
\|(\bar N_\ell-\bar N'_\ell)+(E_\ell-E'_\ell)\|_1
\ge
\|\bar N_\ell-\bar N'_\ell\|_1-\|E_\ell\|_1-\|E'_\ell\|_1.
$$
Since $r\in(0,1)$, the map $x\mapsto r^x$ is decreasing, so it holds that 
\begin{align*}
r^{\|N_\ell-N'_\ell\|_1}
&\le
r^{\|\bar N_\ell-\bar N'_\ell\|_1-\|E_\ell\|_1-\|E'_\ell\|_1}
\\
&=
r^{\|\bar N_\ell-\bar N'_\ell\|_1}
r^{-\|E_\ell\|_1}r^{-\|E'_\ell\|_1}
\\
&=
r^{\|\bar N_\ell-\bar N'_\ell\|_1}
\exp \bigl((-\log r)(\|E_\ell\|_1+\|E'_\ell\|_1)\bigr).
\end{align*}
Moreover, we have that $\|E_\ell\|_1=\sum_{j\in S}|\Delta_{M_\ell,j}|
\le \sqrt{s}\|\Delta_{M_\ell}\|_2=\sqrt{s}W_{M_\ell}$, and similarly
$\|E'_\ell\|_1\le \sqrt{s}W'_{M'_\ell}$.
It follows that 
\begin{equation}\label{eq:app-cv-perturb}
r^{\|N_\ell-N'_\ell\|_1}
\le
r^{\|\bar N_\ell-\bar N'_\ell\|_1}
\exp \Bigl((-\log r)\sqrt{s}(W_{M_\ell}+W'_{M'_\ell})\Bigr).
\end{equation}
Since $-\log r = -\log(2^{-1/2})=\tfrac12\log 2$, the exponent coefficient is equal to 
$\tfrac{\log 2}{2}\sqrt{s}$, which is dominated by $\eta_{\rm req}$ by definition.
Thus, \eqref{eq:app-exp-moment-lambda} implies that
$$
\sup_{\ell\ge1}\E\exp \Bigl((-\log r)\sqrt{s}W_{M_\ell}\Bigr)<\infty,
\qquad
\sup_{\ell\ge1}\E\exp \Bigl((-\log r)\sqrt{s}W'_{M'_\ell}\Bigr)<\infty.
$$
Note that $(N_{\ell,j})_{j\in S^c}$ is independent of the informative-time process $(W_n)$, and thus conditional on $M_{\ell}$, $\bar{N}_{\ell}$ is independent of $W_{M_{\ell}}$. Taking the expectations conditional on $M_\ell$ and $M'_{\ell}$, we can deduce from \eqref{eq:app-cv-perturb} that
\begin{equation}\label{eq:app-cv-reduce}
\E\bigl[r^{\|N_\ell-N'_\ell\|_1}\bigr]
\le
C\E\bigl[r^{\|\bar N_\ell-\bar N'_\ell\|_1}\bigr],
\end{equation}
where constant $C<\infty$ follows from \eqref{eq:app-exp-moment}. Using the same argument, with the roles of $N_{\ell,j}$ and $M_\ell$ exchanged, we also have that 
\begin{equation}\label{eq:app-cv-reduce-1}
\E\bigl[r^{\|\bar N_\ell-\bar N'_\ell\|_1}\bigr]
\le
\frac{1}{C}\E\bigl[r^{\|N_\ell-N'_\ell\|_1}\bigr].
\end{equation}
Hence, it suffices to focus on the non-asymptotic behavior of the reduced vector $\widetilde N_\ell$.

Notice that by construction, 
$
\widetilde N_\ell=(M_\ell,(N_{\ell,j})_{j\in S^c})
$
is multinomial with parameter vector
$$
\pi^{\mathrm{pop}}
=
\Bigl(
q,\
\underbrace{\tfrac{1-q}{d-s},\dots,\tfrac{1-q}{d-s}}_{d-s}
\Bigr) \in \Delta^{d-s+1},
$$
and $\widetilde N'_\ell$ is an independent copy.
In light of $\|\bar N_\ell-\bar N'_\ell\|_1=\|\widetilde N_\ell-\widetilde N'_\ell\|_1$
(since the informative block is replaced with $M_\ell/s$ replicated $s$ times), it holds that 
$$
\E\bigl[r^{\|\bar N_\ell-\bar N'_\ell\|_1}\bigr]
=
\E\bigl[r^{\|\widetilde N_\ell-\widetilde N'_\ell\|_1}\bigr].
$$
An application of part (ii) in Proposition~\ref{prop:poisson-functionals} with
the choices of $\boldsymbol p=\pi^{\mathrm{pop}} \in \Delta^{d-s+1}$ and $r=2^{-1/2}$ yields the Poisson-kernel representation
$$
\E \left[
2^{\sum_{j=1}^{d-s+1}\min\{\widetilde N_{\ell,j},\widetilde N'_{\ell,j}\}}
\right]
=
2^\ell{ L_{\ell,d-s+1, 2^{-1/2}}(\pi^{\mathrm{pop}})}.
$$
Consequently, combining this with \eqref{eq:app-cv-basic}, \eqref{eq:app-cv-reduce}, and \eqref{eq:app-cv-reduce-1}, we can obtain that 
$$
\E \left[
2^{\sum_{j=1}^d \min\{N_{\ell,j},N'_{\ell,j}\}}
\right]
=
2^\ell\E\bigl[r^{\|N_\ell-N'_\ell\|_1}\bigr]
\le
C2^\ell{ L_{\ell,d-s+1,2^{-1/2}}(\pi^{\mathrm{pop}})}
$$
and
$$
\E \left[
2^{\sum_{j=1}^d \min\{N_{\ell,j},N'_{\ell,j}\}}
\right]
=
2^\ell\E\bigl[r^{\|N_\ell-N'_\ell\|_1}\bigr]
\ge
\frac{1}{C}2^\ell{ L_{\ell,d-s+1,2^{-1/2}}(\pi^{\mathrm{pop}})}
$$
which is the desired control of the cross-tree variance functional.

Therefore, substituting the bounds obtained in Steps~1--3 above into the general MSE
decomposition \eqref{eq:rf-mse} of Theorem~\ref{thm:mse} gives the stated
non-asymptotic MSE bound (\ref{eq:cart-mse}) for the population CART. This completes the proof of Theorem~\ref{thm:cart-mse}.

\subsection{Proof of Theorem \ref{thm:greedy-bellman-certificate}}

The proof of this theorem employs the Bellman certificate in Theorem~\ref{thm:marginal-value-bellman}. If greedy \textit{were} locally optimal,  every action it selects with positive probability would have to minimize the Bellman continuation value for the marginal terminal cost. Event $E$ violates exactly this support condition by providing a feasible action with strictly smaller continuation cost.

Let
$
h^{\rm gr}:=\Gamma_{\nu^{\rm gr}}^{\Phi,\Psi}.
$
Assume, \textit{toward contradiction}, that $\nu^{\rm gr}$ is a local minimizer of
$\mathcal J_{\ell,B}^{\Phi,\Psi}$ over $\mathfrak V_\ell$. 
Since $\pi^{\rm gr}$ induces $\nu^{\rm gr}$, 
it follows from Theorem~\ref{thm:marginal-value-bellman}(ii) that $\pi^{\rm gr}$ is optimal for the linear terminal-cost problem with terminal cost $h^{\rm gr}$.
In particular, it satisfies that 
\[
\operatorname{supp}
\pi^{\rm gr}_{t+1}(\cdot\mid \mathcal H_t,U_{t+1})
\subseteq
\arg\min_{a\in\mathcal A(U_{t+1})}
\mathcal V_{t+1}^{h^{\rm gr}}(N_t+e_a),
\qquad
\P_{\pi^{\rm gr}}\text{-a.s.}
\]
As in the proof of Lemma \ref{lem:linear-terminal-bellman}, we define
$$
D_t^{\rm gr}
:=
\sum_{a\in\mathcal A(U_{t+1})}
\pi^{\rm gr}_{t+1}(a\mid H_t,U_{t+1})
\mathcal V_{t+1}^{h^{\rm gr}}(N_t+e_a)
-
\min_{a\in\mathcal A(U_{t+1})}
\mathcal V_{t+1}^{h^{\rm gr}}(N_t+e_a).
$$
By construction, it holds that 
$
D_t^{\rm gr}\ge0.
$
Moreover, the Bellman support condition for an optimal policy in the linear
terminal-cost problem implies that 
$$
D_t^{\rm gr}=0
\qquad
\P_{\pi^{\rm gr}}\text{-a.s.}
$$

However, on event $E$, greedy rule assigns positive probability to an action
$j$ for which there exists another feasible action $k$ satisfying that 
$$
\mathcal V_{t+1}^{h^{\rm gr}}(N_t+e_k)
<
\mathcal V_{t+1}^{h^{\rm gr}}(N_t+e_j).
$$
Hence, action $j$ is not a minimizer of
$
a\mapsto
\mathcal V_{t+1}^{h^{\rm gr}}(N_t+e_a)
$
over $\mathcal A(U_{t+1})$. Since greedy puts positive probability on this
nonminimizing action, the convex combination defining the first term in
$D_t^{\rm gr}$ is strictly larger than the minimum. It follows that  
$D_t^{\rm gr}>0$ on event $E$.
Since $P_{\pi^{\rm gr}}(E)>0$, this contradicts
$D_t^{\rm gr}=0$ almost surely. Consequently, $\nu^{\rm gr}$ cannot be a local
minimizer of $\mathcal J_{\ell,B}^{\Phi,\Psi}$ over $\mathfrak V_\ell$.

Finally, since a global minimizer is necessarily a local minimizer, 
$\nu^{\rm gr}$ cannot be a global minimizer either.
This concludes the proof of Theorem \ref{thm:greedy-bellman-certificate}.

\subsection{Proof of Theorem~\ref{thm:marginal-value-bellman}}

The proof of this theorem will turn the nonlinear terminal-law objective into a linearized control problem. The main technical point is that the terminal-law feasible set is a compact convex polytope, so local variations along feasible line segments are sufficient. Specifically, the proof relies on the following Lemmas. 
\begin{lemma}[Bellman principle for linear terminal costs]
\label{lem:linear-terminal-bellman}
If $\pi^\star\in\Pi_\ell$ is optimal for the linear terminal-cost
problem $\inf_{\pi\in \Pi_{\ell}}\E_{\pi}[h(N_{\ell})]$, we have that for each $t=0,\ldots,\ell-1$,
$$
\operatorname{supp}
\pi^\star_{t+1}(\cdot\mid H_t,U_{t+1})
\subseteq
\arg\min_{j\in\mathcal A(U_{t+1})}
\mathcal V_{t+1}^h(N_t+e_j),
\qquad \P_{\pi^\star}\text{-a.s.}
$$
\end{lemma}

Lemma~\ref{lem:linear-terminal-bellman} above is the finite-horizon control foundation of Theorem~\ref{thm:marginal-value-bellman}. It converts a terminal marginal cost into explicit state-mask action restrictions.

\begin{lemma}\label{lem:compact-convex}
The feasible terminal-law set
$\mathfrak V_\ell=\{\nu_{\pi,\ell}:\pi\in\Pi_\ell\}$
is a compact convex polytope.
\end{lemma}

Lemma~\ref{lem:compact-convex} above ensures that the first-order variations along feasible terminal-law directions are well-defined and 
the local terminal-law optimality can be tested by feasible line segments. 

We are now ready to prove Theorem \ref{thm:marginal-value-bellman} below.

\paragraph{Proof of part \emph{(i)}.}  Let $\nu,\widetilde\nu\in\mathfrak V_\ell$, and define the signed measure
$
\eta:=\widetilde\nu-\nu.
$
Since $\mathfrak V_\ell$ is convex, it holds that 
$$
\nu_\varepsilon:=(1-\varepsilon)\nu+\varepsilon\widetilde\nu
=
\nu+\varepsilon\eta
\in\mathfrak V_\ell,
\qquad \varepsilon\in[0,1].
$$
Using \eqref{eq:terminal-law-objective}, we can derive the following first-variation expansion
\begin{align*}
\mathcal J_{\ell,B}^{\Phi,\Psi}(\nu_\varepsilon)
&-
\mathcal J_{\ell,B}^{\Phi,\Psi}(\nu)
=
\frac{\varepsilon}{B}
\int_{\mathcal N_{\ell,d}}\Phi(\mathbf n)\,d\eta(\mathbf n)+
\frac{B-1}{B}
\bigg[
\varepsilon
\int_{\mathcal N_{\ell,d}}\int_{\mathcal N_{\ell,d}}
\Psi(\mathbf n,\mathbf n')\,d\eta(\mathbf n)\,d\nu(\mathbf n')\\
&\quad+
\varepsilon
\int_{\mathcal N_{\ell,d}}\int_{\mathcal N_{\ell,d}}
\Psi(\mathbf n,\mathbf n')\,d\nu(\mathbf n)\,d\eta(\mathbf n')+
\varepsilon^2
\int_{\mathcal N_{\ell,d}}\int_{\mathcal N_{\ell,d}}
\Psi(\mathbf n,\mathbf n')\,d\eta(\mathbf n)\,d\eta(\mathbf n')
\bigg].
\end{align*}
Interchanging the dummy variables in the second first-order double integral leads to 
$$
\int\int \Psi(\mathbf n,\mathbf n')\,d\nu(\mathbf n)\,d\eta(\mathbf n')
=
\int\int \Psi(\mathbf n',\mathbf n)\,d\nu(\mathbf n')\,d\eta(\mathbf n).
$$
Then it follows that 
\begin{align}
\mathcal J_{\ell,B}^{\Phi,\Psi}(\nu_\varepsilon)
-
\mathcal J_{\ell,B}^{\Phi,\Psi}(\nu)
&=
\varepsilon
\int_{\mathcal N_{\ell,d}}
\Gamma_{\nu}^{\Phi,\Psi}(\mathbf n)\,d\eta(\mathbf n)
\nonumber\\
&\quad+
\varepsilon^2
\frac{B-1}{B}
\int_{\mathcal N_{\ell,d}}\int_{\mathcal N_{\ell,d}}
\Psi(\mathbf n,\mathbf n')\,d\eta(\mathbf n)\,d\eta(\mathbf n').
\label{eq:first-variation-expansion}
\end{align}
Hence, it holds that 
$$
\left.
\frac{d}{d\varepsilon}
\mathcal J_{\ell,B}^{\Phi,\Psi}
\bigl((1-\varepsilon)\nu+\varepsilon\widetilde\nu\bigr)
\right|_{\varepsilon=0+}
=
\int_{\mathcal N_{\ell,d}}
\Gamma_{\nu}^{\Phi,\Psi}(\mathbf n)\,d(\widetilde\nu-\nu)(\mathbf n).
$$

Now assume that $\nu^\star$ is a local minimizer of
$\mathcal J_{\ell,B}^{\Phi,\Psi}$ over $\mathfrak V_\ell$. For any
$\nu\in\mathfrak V_\ell$, the line segment
$
\nu_\varepsilon=(1-\varepsilon)\nu^\star+\varepsilon\nu
$
lies in $\mathfrak V_\ell$. For all sufficiently small $\varepsilon>0$,
the local optimality entails that 
$$
\mathcal J_{\ell,B}^{\Phi,\Psi}(\nu_\varepsilon)
\ge
\mathcal J_{\ell,B}^{\Phi,\Psi}(\nu^\star).
$$
Taking the right derivative at $\varepsilon=0$ and using
\eqref{eq:first-variation-expansion}, we can deduce that 
$$
\int_{\mathcal N_{\ell,d}}
\Gamma_{\nu^\star}^{\Phi,\Psi}(\mathbf n)\,d(\nu-\nu^\star)(\mathbf n)
\ge 0,
\qquad \nu\in\mathfrak V_\ell.
$$
Equivalently, it holds that 
$$
\int_{\mathcal N_{\ell,d}}
\Gamma_{\nu^\star}^{\Phi,\Psi}(\mathbf n)\,d\nu^\star(\mathbf n)
\le
\int_{\mathcal N_{\ell,d}}
\Gamma_{\nu^\star}^{\Phi,\Psi}(\mathbf n)\,d\nu(\mathbf n),
\qquad \nu\in\mathfrak V_\ell.
$$
This establishes part \emph{(i)}.

\paragraph{Proof of part \emph{(ii)}.}
Let us take
$
h=\Gamma_{\nu^\star}^{\Phi,\Psi}
$ in Lemma~\ref{lem:linear-terminal-bellman}. 
If $\pi^\star\in\Pi_\ell$ satisfies that 
$\nu_{\pi^\star,\ell}=\nu^\star$, it follows from part (i) that 
$$
\mathbb E_{\pi^\star}
\left[
\Gamma_{\nu^\star}^{\Phi,\Psi}(N_\ell)
\right]
=
\inf_{\nu\in\mathfrak V_\ell}
\int_{\mathcal N_{\ell,d}}
\Gamma_{\nu^\star}^{\Phi,\Psi}(\mathbf n)\,d\nu(\mathbf n) 
=  \inf_{\pi\in\Pi_\ell}
\mathbb E_\pi
\left[
\Gamma_{\nu^\star}^{\Phi,\Psi}(N_\ell)
\right],
$$
since each $\nu\in\mathfrak V_\ell$ is induced by some policy in
$\Pi_\ell$. 
Consequently, $\pi^\star$ is optimal for the linear terminal-cost problem with
terminal cost $h=\Gamma_{\nu^\star}^{\Phi,\Psi}$. An application of Lemma~\ref{lem:linear-terminal-bellman} gives that 
$$
\operatorname{supp}
\pi^\star_{t+1}(\cdot\mid \mathcal H_t,U_{t+1})
\subseteq
\arg\min_{j\in\mathcal A(U_{t+1})}
\mathcal V_{t+1}^{\Gamma_{\nu^\star}^{\Phi,\Psi}}(N_t+e_j),
\qquad \P_{\pi^\star}\text{-a.s.}
$$
This establishes the marginal-value Bellman condition.

\paragraph{Proof of part \emph{(iii)}.}
Assume that $\pi^\star$ satisfies \eqref{eq:marginal-bellman-support} for
all $t=0,\ldots,\ell-1$, and let
$\nu^\star=\nu_{\pi^\star,\ell}$. In view of the Bellman characterization in part \emph{(ii)}, 
policy $\pi^\star$ must be optimal for the linear terminal-cost problem
$\inf_{\pi\in \Pi_\ell} \E_\pi [\Gamma_{\nu^\star}^{\Phi,\Psi}(N_\ell)]$. 
It then follows that
$$
\int_{\mathcal N_{\ell,d}}
\Gamma_{\nu^\star}^{\Phi,\Psi}(\mathbf n)\,d(\nu-\nu^\star)(\mathbf n)
\ge0,
\qquad \nu\in\mathfrak V_\ell.
$$

Since $\mathcal J_{\ell,B}^{\Phi,\Psi}$ is convex on
$\mathfrak V_\ell$, its first variation gives a global supporting
hyperplane on $\mathfrak V_\ell$. In fact, let us fix
$\nu\in\mathfrak V_\ell$ and set
$
\eta:=\nu-\nu^\star.
$
For $\theta\in[0,1]$, denote by 
$
\nu_\theta:=\nu^\star+\theta\eta
=
(1-\theta)\nu^\star+\theta\nu.
$
Since $\mathfrak V_\ell$ is convex, we have that $\nu_\theta\in\mathfrak V_\ell$ for all
$\theta\in[0,1]$. Let
$
g(\theta):=\mathcal J_{\ell,B}^{\Phi,\Psi}(\nu_\theta).
$
Then the convexity of $\mathcal J_{\ell,B}^{\Phi,\Psi}$ on $\mathfrak V_\ell$
implies that $g$ is convex on $[0,1]$. It follows that 
$$
g(1)-g(0)\ge g'_+(0).
$$
By the same calculation in part \emph{(i)}, we have 
$$
g'_+(0)
=
\int_{\mathcal N_{\ell,d}}
\Gamma_{\nu^\star}^{\Phi,\Psi}(\mathbf n)\,d(\nu-\nu^\star)(\mathbf n) \geq 0,
$$
which implies
$$
\mathcal J_{\ell,B}^{\Phi,\Psi}(\nu)
\ge
\mathcal J_{\ell,B}^{\Phi,\Psi}(\nu^\star),
\qquad \nu\in\mathfrak V_\ell.
$$
As a result, $\nu^\star$ is globally optimal over $\mathfrak V_\ell$. Hence,
policy $\pi^\star$ is globally optimal over $\Pi_\ell$. This establishes part \emph{(iii)}.

It remains to verify the stated sufficient condition for convexity. Notice that
$$
\mathcal J_{\ell,B}^{\Phi,\Psi}(\nu)
=
\frac{1}{B}\int \Phi(\mathbf n)\,d\nu(\mathbf n)
+
\frac{B-1}{B}
\int\int Q^\Psi(\mathbf n,\mathbf n')\,d\nu(\mathbf n)\,d\nu(\mathbf n').
$$
For any $\nu_0,\nu_1\in\mathfrak V_\ell$, let
$
a:=\nu_1-\nu_0,
$
and
$
\nu_\theta:=(1-\theta)\nu_0+\theta\nu_1$,
$ \theta\in[0,1].
$ 
According to the discussion in part \emph{(i)}, the second derivative along this line segment is given by 
$$
\frac{d^2}{d\theta^2}
\mathcal J_{\ell,B}^{\Phi,\Psi}(\nu_\theta)
=
2\frac{B-1}{B}
\sum_{\mathbf n,\mathbf n'\in\mathcal N_{\ell,d}}
a(\mathbf n)Q^\Psi(\mathbf n,\mathbf n')a(\mathbf n').
$$
By the positive semi-definiteness condition,
this second derivative is nonnegative for every pair of $\nu_0$ and $\nu_1$. Therefore, $\mathcal J_{\ell,B}^{\Phi,\Psi}$ is convex on $\mathfrak V_\ell$. 
This completes the proof of Theorem~\ref{thm:marginal-value-bellman}.

\section{Additional proofs and technical details} \label{new.sec.addtechdeta}

\subsection{Proof of Proposition~\ref{prop:cart-onestep}} \label{proof.prop:cart-onestep}

The proof of this proposition will be some direct algebraic calculations, but it is important because it isolates the two ingredients used later: an exact quadratic increment identity and a deterministic bound on the one-step movement of $W_n$. The identity comes from the centered update of $\Delta_n$, while the bound follows from the nonexpansiveness of the Euclidean norm.

At informative time $n$, the vector of informative counts updates as
$$
Z_{n+1} = Z_n + e_{J_{n+1}},
$$
where $J_{n+1}\in\{1,\dots,s\}$ is the coordinate selected by the greedy rule.
It follows from the definition that 
$$
\Delta_n = Z_n - \frac{n}{s}\mathbf 1,
\qquad
\Delta_{n+1}
=
Z_{n+1}-\frac{n+1}{s}\mathbf 1
=
\Delta_n + e_{J_{n+1}}-\frac1s\mathbf 1.
$$
Squaring the centered update and expanding give that 
\begin{align*}
V_{n+1}
&=
\|\Delta_{n+1}\|_2^2
=
\|\Delta_n\|_2^2
+2\langle\Delta_n,e_{J_{n+1}}-\tfrac1s\mathbf 1\rangle
+\|e_{J_{n+1}}-\tfrac1s\mathbf 1\|_2^2.
\end{align*}
Since $\sum_{j=1}^s\Delta_{n,j}=0$, it holds that $\langle\Delta_n,\mathbf 1\rangle=0$, and thus
$$
\langle\Delta_n,e_{J_{n+1}}-\tfrac1s\mathbf 1\rangle
=
\Delta_{n,J_{n+1}}.
$$
Moreover, we can deduce that 
$$
\|e_{J_{n+1}}-\tfrac1s\mathbf 1\|_2^2
=
1-\frac1s.
$$
This establishes \eqref{eq:one-step}.

For the increment bound, it follows from the inequality
$|\|x\|_2-\|y\|_2|\le\|x-y\|_2$ that 
$$
|W_{n+1}-W_n|
\le
\|\Delta_{n+1}-\Delta_n\|_2
=
\|e_{J_{n+1}}-\tfrac1s\mathbf 1\|_2
=
\sqrt{1-\frac1s},
$$
which yields \eqref{eq:bounded-jump}.
This completes the proof of Proposition~\ref{prop:cart-onestep}.

\subsection{Proof of Proposition \ref{prop:majorization-local}} \label{proof.prop:majorization-local}


Let us write $z:=\widetilde Z_n$, $A:=\mathcal A_{n+1}$, and
$
c:=\frac{n+1}{s}.
$
Then it holds that 
$
X_n(j)=z+e_j-c\mathbf 1_s.
$
Since subtracting the same multiple of $\mathbf 1_s$ from two vectors
preserves majorization, it suffices to compare
$
z+e_i$
and $
z+e_j.
$
Denote by $M_n(A):=\argmin_{k\in A} z_k$.
Fix $i\in M_n(A)$ and $j\in A$.  The case $i=j$ is trival, so we assume 
$i\neq j$. Since $i\in M_n(A)$, we have
$
z_i\le z_j.
$
Let
$
y:=z+e_j$ 
and $
x:=z+e_i.
$
Then 
$
y_j-y_i=z_j+1-z_i\ge1.
$
Let $P_{ij}$ be the permutation matrix that swaps coordinates $i$ and $j$,
and define
$
\lambda:=1/d\in(0,1].
$
Then it holds that 
$
x=(1-\lambda)y+\lambda P_{ij}y.
$
Indeed, for coordinate $i$,
$$
(1-\lambda)y_i+\lambda y_j
=
y_i+\lambda(y_j-y_i)
=
z_i+1
=
x_i,
$$
and for coordinate $j$,
$$
(1-\lambda)y_j+\lambda y_i
=
y_j-\lambda(y_j-y_i)
=
z_j
=
x_j.
$$
All other coordinates are unchanged. Hence, $x\prec y$ and thus 
$
X_n(i)\prec X_n(j).
$
The Schur-convexity then gives 
$
\Psi(X_n(i))\le\Psi(X_n(j)).
$

It remains to establish the exact argmin identity for symmetric strictly
Schur-convex $\Psi$. If $i,k\in M_n(A)$, it holds that $z_i=z_k$, so
$X_n(i)$ and $X_n(k)$ are permutations of one another. By symmetry, it follows that  
$
\Psi(X_n(i))=\Psi(X_n(k)).
$
Now let us take $j\notin M_n(A)$, and choose any $i\in M_n(A)$. Then we have 
$
z_i<z_j.
$
The preceding argument implies 
$
X_n(i)\prec X_n(j).
$
Moreover, $X_n(i)$ is not a permutation of $X_n(j)$ since $z_i < z_j$. Since $\Psi$ is symmetric strictly
Schur-convex, it holds that 
$
\Psi(X_n(i))<\Psi(X_n(j)).
$
Therefore, exactly the coordinates with minimal $\widetilde Z_{n,j}$ minimize
$\Psi(X_n(j))$ over $j\in A$, which yields $$
\argmin_{j\in A}\Psi(X_n(j))
=
\argmin_{j\in A}\widetilde Z_{n,j}.
$$
This concludes the proof of Proposition~\ref{prop:majorization-local}.



\subsection{Proof of Proposition \ref{prop:benchmark}} \label{proof.prop:benchmark}

The proof of this proposition will highlight that the benchmark policy has no stabilizing feedback, in contrast to the population CART. The law of large numbers gives the first-order balance, but the coordinate imbalances remain martingale fluctuations of order $\sqrt n$, which is enough to make every positive exponential moment of $W_n$ diverge with $n$.

Recall that $Z_n$ is a sum of i.i.d. random vectors, so the convergence $Z_n/n \to 1/s \1$ follows directly from the law of large numbers. It remains to establish $\E_{\pi_{\ex}}\bigl[\Delta_{n,J_{n+1}}\mid \F_n\bigr]=0$ and Equation~\eqref{eq:no-exp-compression}. 
Observe that for the benchmark process, it holds that 
$$
\E \left[\Delta_{n+1,j}-\Delta_{n,j}\mid \mathcal F_n\right]
=
\E \left[\1\{J_{n+1}=j\}\right]-\frac1s
=
\frac1s-\frac1s
=
0,
\qquad j\in S.
$$
Hence, each coordinate imbalance is a martingale. Consequently, we have that 
$$
\E \left[\Delta_{n,J_{n+1}}\mid \mathcal F_n\right]
=
\sum_{j\in S}\frac1s\Delta_{n,j}
=
0,
$$
which proves $\E_{\pi_{\ex}}\bigl[\Delta_{n,J_{n+1}}\mid \F_n\bigr]=0$.

Let us recall that
$$
Z_n = (Z_{n,1},\dots,Z_{n,s})
\sim
\mathrm{Multinomial} \left(n,\frac1s,\dots,\frac1s\right).
$$
In particular, it holds that 
$$
Z_{n,1}\sim \mathrm{Binomial} \left(n,\frac1s\right).
$$
Meanwhile, we also have that 
$$
W_n
=
\sqrt{\sum_{j=1}^s (Z_{n,j}-n/s)^2}
 \ge 
|Z_{n,1}-n/s|.
$$
Thus, it follows that for any $\eta>0$,
\begin{equation}\label{eq:bench-lower}
\E e^{\eta W_n}
 \ge 
\E e^{\eta |Z_{n,1}-n/s|}.
\end{equation}

Denote by $X_n:=Z_{n,1}-n/s$.
Since $Z_{n,1}$ is binomial, its moment generating function is explicit:
for any $\lambda\in\R$,
$$
\E e^{\lambda Z_{n,1}}
=
\Bigl(1-\tfrac1s+\tfrac1s e^\lambda\Bigr)^n.
$$
Then we can deduce that 
$$
\E e^{\lambda X_n}
=
\exp \left(-\frac{\lambda n}{s}\right)
\Bigl(1-\tfrac1s+\tfrac1s e^\lambda\Bigr)^n
=
\exp \bigl(n\psi(\lambda)\bigr),
$$
where
$$
\psi(\lambda)
:=
\log \Bigl(1-\tfrac1s+\tfrac1s e^\lambda\Bigr)
-\frac{\lambda}{s}.
$$
A direct calculation shows that $\psi(0)=\psi'(0)=0$ and
$$
\psi''(0)=\frac{s-1}{s^2}>0.
$$
Hence, there exist some $\lambda_0>0$ and $c>0$ such that
$$
\psi(\lambda)\ge c\lambda^2,
\qquad |\lambda|\le\lambda_0.
$$

Let us fix any $\eta>0$ and choose $\lambda=\eta\wedge\lambda_0$.
With the aid of the inequality
$$
e^{\eta|x|}\ge \tfrac12\bigl(e^{\lambda x}+e^{-\lambda x}\bigr),
$$
we can obtain that 
$$
\E e^{\eta |X_n|}
 \ge 
\frac12\Bigl(\E e^{\lambda X_n}+\E e^{-\lambda X_n}\Bigr)
\ge
\frac12 \exp \bigl(n\psi(\lambda)\bigr)
 \ge 
\frac12 \exp(c\lambda^2 n).
$$
Thus, combining this with \eqref{eq:bench-lower} gives that 
$$
\E e^{\eta W_n}
 \ge 
\frac12\exp(c\lambda^2 n)
\to \infty,\qquad\text{as}~n\to\infty.
$$
Since $\eta>0$ is arbitrary, we establish Equation~\eqref{eq:no-exp-compression}. This completes the proof of Proposition \ref{prop:benchmark}.

\subsection{Proof of Proposition~\ref{prop:poisson-functionals}} \label{proof.prop:poisson-functionals}

The proof of this proposition will convert the relevant binomial and multinomial exponential functionals into the Fourier--Poisson-kernel integrals, exploiting the auxiliary lemmas presented in Appendix~\ref{sec:auxiliary}. These representations separate the dominant exponential factor from a polynomial attenuation term, which is then evaluated by the Laplace method near the phase-aligned maximizer. In particular, the interior-probability conditions in Proposition~\ref{prop:poisson-functionals} are needed for these Laplace asymptotics instead of for the exact integral identities themselves. When parameters such as $\alpha$ and $\bp$ lie in the interior, the Fourier integral has a unique nondegenerate maximizer after the phase invariance is removed, so Lemma~\ref{lem:laplace-multi} yields the stated polynomial order. The boundary cases may still admit exact integral representations, but their polynomial prefactors can change since the local curvature degenerates or the maximizer lies on a lower-dimensional face. As a basic example, if $N\sim\mathrm{Binomial}(\ell,p)$, we have that for any $r\in(0,1)$, an application of Lemma~\ref{lem:binom-pgf} with $z:=r$ gives that $\mathbb E[r^N]=(1-p+pr)^\ell$.

\subsubsection{Proof of Equation~\eqref{eq:max-binomial-functional}}
\label{app:proof-lem2}

Let $N_\ell,N_\ell'\stackrel{\mathrm{i.i.d.}}{\sim}\mathrm{Binomial}(\ell,p)$ and fix $r\in(0,1)$.
With the aid of the elementary identity $\max\{x,y\}=\frac{x+y+|x-y|}{2}$, we can show that 
\begin{equation}\label{eq:lem2-start-r}
r^{\max\{N_\ell,N_\ell'\}}
=
r^{(N_\ell+N_\ell')/2}r^{|N_\ell-N_\ell'|/2}.
\end{equation}
An application of Lemma~\ref{lem:poisson-fourier} with parameter $\sqrt r\in(0,1)$ leads to 
$$
r^{|N_\ell-N_\ell'|/2}
=
(\sqrt r)^{|N_\ell-N_\ell'|}
=
\int_{-\pi}^{\pi}e^{\mathrm i\theta(N_\ell-N_\ell')}\mu_{\sqrt r}(\theta)d\theta .
$$
Substituting this into \eqref{eq:lem2-start-r} and applying Fubini's theorem, we can deduce that 
\begin{align}
\mathbb E \left[r^{\max\{N_\ell,N_\ell'\}}\right]
&=
\int_{-\pi}^{\pi}
\mathbb E \left[r^{(N_\ell+N_\ell')/2}e^{\mathrm i\theta(N_\ell-N_\ell')}\right]
\mu_{\sqrt r}(\theta)d\theta \notag\\
&=
\int_{-\pi}^{\pi}
\mathbb E \left[r^{N_\ell/2}e^{\mathrm i\theta N_\ell}\right]
\mathbb E \left[r^{N_\ell'/2}e^{-\mathrm i\theta N_\ell'}\right]
\mu_{\sqrt r}(\theta)d\theta \notag\\
&=
\int_{-\pi}^{\pi}
\Bigl|
\mathbb E \left[(\sqrt re^{\mathrm i\theta})^{N_\ell}\right]
\Bigr|^2
\mu_{\sqrt r}(\theta)d\theta ,
\label{eq:lem2-factor-r}
\end{align}
where we have used that the second expectation is the complex conjugate of the first.
By resorting to Lemma~\ref{lem:binom-pgf} with $z=\sqrt re^{\mathrm i\theta}$, it holds that 
$$
\mathbb E \left[(\sqrt re^{\mathrm i\theta})^{N_\ell}\right]
=
\left(1-p+p\sqrt re^{\mathrm i\theta}\right)^{\ell}.
$$
Then plugging this into \eqref{eq:lem2-factor-r} gives that 
\begin{equation}\label{eq:lem2-int-raw-r}
\mathbb E \left[r^{\max\{N_\ell,N_\ell'\}}\right]
=
\int_{-\pi}^{\pi}
\left|1-p+p\sqrt re^{\mathrm i\theta}\right|^{2\ell}
\mu_{\sqrt r}(\theta)d\theta .
\end{equation}

\medskip

Denote by 
$$
\alpha:=\frac{p\sqrt{r}}{1-p+p\sqrt{r}}\in(0,1).
$$
A direct algebraic rearrangement yields that for each $\theta\in[-\pi,\pi]$,
$$
1-p+p\sqrt{r}e^{\mathrm i\theta}
=
(1-p+p\sqrt{r})\bigl((1-\alpha)+\alpha e^{\mathrm i\theta}\bigr).
$$
Consequently, it holds that 
$$
\left|1-p+p \sqrt{r}e^{\mathrm i\theta}\right|^{2\ell}
=
(1-p+p\sqrt{r})^{2\ell}
\left|(1-\alpha)+\alpha e^{\mathrm i\theta}\right|^{2\ell}.
$$
Inserting this factorization into \eqref{eq:lem2-int-raw-r} (and reparametrizing the
Poisson-kernel measure accordingly), we can obtain the asserted representation
$$
\mathbb E \left[r^{\max\{N_\ell,N_\ell'\}}\right]
=
(1-p+p\sqrt{r})^{2\ell}F_{\ell,r}(\alpha),
$$
where
$$
F_{\ell,r}(\alpha)
:=
\mathbb E_{\Xi\sim\mu_{\sqrt{r}}}
\Bigl[\bigl|(1-\alpha)+\alpha e^{\mathrm i\Xi}\bigr|^{2\ell}\Bigr],
\qquad \alpha\in[0,1].
$$

\medskip

We next establish the polynomial order of $F_{\ell,r}(\alpha)$ for fixed $\alpha\in(0,1)$ and $r\in(0,1)$.
Let us write
\begin{equation}\label{eq:lem2-modulus}
\bigl|(1-\alpha)+\alpha e^{\mathrm i\theta}\bigr|^{2}
=
(1-\alpha)^2+\alpha^2+2\alpha(1-\alpha)\cos\theta
=
1-2\alpha(1-\alpha)\bigl(1-\cos\theta\bigr).
\end{equation}
Denote by $c_\alpha:=2\alpha(1-\alpha)\in(0,1/2]$, and define on $D:=(-\pi,\pi)$
$$
g(\theta)
:=
\log \Big(1-c_\alpha\bigl(1-\cos\theta\bigr)\Big),
\qquad
h(\theta):=\mu_{\sqrt{r}}(\theta).
$$
Then we have that 
\begin{equation}\label{eq:lem2-laplace-form}
F_{\ell,r}(\alpha)
=
\int_{-\pi}^{\pi}\exp\big(\ell g(\theta)\big)h(\theta)d\theta.
\end{equation}

We will verify the conditions of Lemma~\ref{lem:laplace-multi} with $m=1$, $D=(-\pi,\pi)$, and $x_0=0$.
First, notice that $g(0)=0$ and $g(\theta)<0$ for $\theta\neq 0$, and thus $0$ is the unique global maximizer of $g$ on $D$.
Moreover, $g$ is $C^3$ on $D$, $\nabla g(0)=g'(0)=0$, and a Taylor expansion of function $1-\cos\theta$ at $0$ yields that 
$$
g(\theta)=-\frac{c_\alpha}{2}\theta^2+O(\theta^4),
\qquad \theta\to0,
$$
so the Hessian $\mathrm{Hess}g(0)=g''(0)=-c_\alpha<0$ is negative. 
In particular, there exist some constants $c_0,c_1>0$ and $r_0\in(0,\pi)$ such that for all $|\theta|\le r_0$,
$$
-c_1\theta^2\le g(\theta)\le -c_0\theta^2.
$$
Since $g$ is continuous and strictly negative on compact set $\{|\theta|\in[r_0,\pi)\}$, there exists some $\delta_0>0$ such that
$\sup_{|\theta|\ge r_0}g(\theta)\le -\delta_0$.
Further, the Poisson-kernel density $\mu_{\sqrt{r}}$ is continuous and strictly positive on $[-\pi,\pi]$, and thus $h$ is continuous on $D$ with $h(0)>0$.

Therefore, by invoking Lemma~\ref{lem:laplace-multi} to \eqref{eq:lem2-laplace-form}, we can obtain that 
$$
F_{\ell,r}(\alpha)
\sim
\exp \big(\ell g(0)\big)
\left(\frac{2\pi}{\ell}\right)^{1/2}
\frac{h(0)}{\sqrt{-\mathrm{Hess}g(0)}}
=
\left(\frac{2\pi}{\ell}\right)^{1/2}
\frac{\mu_{\sqrt{r}}(0)}{\sqrt{c_\alpha}},
\qquad \ell\to\infty.
$$
In particular, it follows that for each fixed $\alpha\in(0,1)$ and $r\in(0,1)$,
$$
F_{\ell,r}(\alpha)\asymp \ell^{-1/2},
\qquad \ell\to\infty,
$$
which completes the proof of Equation~\eqref{eq:max-binomial-functional}.

\subsubsection{Proof of Equation~\eqref{eq:min-multinomial-functional}}
\label{app:proof-lem3}

Let $N,N'\stackrel{\mathrm{i.i.d.}}{\sim}\mathrm{Multinomial}(\ell,\bp)$ be independent with 
$\bp=(p_1,\dots,p_d)\in \mathfrak S_{d-1}$, and fix $r\in(0,1)$. It follows from the identity $\min\{x,y\}=\frac{x+y}{2}-\frac{|x-y|}{2}$ and the facts
$\sum_{j=1}^d N_j=\sum_{j=1}^d N'_j=\ell$ that 
\begin{align}
\sum_{j=1}^d \min\{N_j,N'_j\}
&=
\frac12\sum_{j=1}^d (N_j+N'_j)-\frac12\sum_{j=1}^d |N_j-N'_j|
=
\ell-\frac12\|N-N'\|_1.
\label{eq:min-to-l1-r}
\end{align}
Hence, it holds that 
\begin{equation}
\label{eq:mult-start-r}
\mathbb E \left[r^{\sum_{j=1}^d \min\{N_j,N'_j\}}\right]
=
r^\ell\mathbb E \left[r^{-\frac12\|N-N'\|_1}\right]
=
r^\ell\mathbb E \left[(\sqrt r)^{-\|N-N'\|_1}\right].
\end{equation}

Let us first derive a Poisson-kernel representation of the expectation on the right-hand side of \eqref{eq:mult-start-r} above.
We can apply Lemma~\ref{lem:l1-poisson} with $\rho=\sqrt r$ and $X=N-N'$ to represent the
$\ell_1$-damping factor
\begin{equation}
\mathbb E \left[(\sqrt r)^{\|N-N'\|_1}\right]
=
\int_{[-\pi,\pi]^d}
\mathbb E \left[e^{\mathrm i\langle\theta,N-N'\rangle}\right]
\prod_{j=1}^d \mu_{\sqrt r}(\theta_j)d\theta_j.
\label{eq:l1-poisson-apply-r}
\end{equation}
In view of the independence of $N$ and $N'$, we have that 
$$
\mathbb E \left[e^{\mathrm i\langle\theta,N-N'\rangle}\right]
=
\mathbb E \left[e^{\mathrm i\langle\theta,N\rangle}\right]
\mathbb E \left[e^{-\mathrm i\langle\theta,N'\rangle}\right]
=
\left|\mathbb E \left[e^{\mathrm i\langle\theta,N\rangle}\right]\right|^2.
$$
An application of Lemma~\ref{lem:mult-cf} leads to 
$$
\left|\mathbb E \left[e^{\mathrm i\langle\theta,N\rangle}\right]\right|^2
=
\left|\sum_{j=1}^d p_j e^{\mathrm i\theta_j}\right|^{2\ell}.
$$
Substituting this into \eqref{eq:l1-poisson-apply-r} and then into \eqref{eq:mult-start-r} gives that 
\begin{equation}
\mathbb E \left[r^{\sum_{j=1}^d \min\{N_j,N'_j\}}\right]=
r^\ell
\int_{[-\pi,\pi]^d}
\left|\sum_{j=1}^d p_j e^{\mathrm i\theta_j}\right|^{2\ell}
\prod_{j=1}^d \mu_{\sqrt r}(\theta_j)d\theta_j 
= r^\ell L_{\ell,d,r}(\bp),
\end{equation}
where $L_{\ell,d,r}$ is exactly the functional stated in \eqref{eq:Lldr}.

We next investigate the asymptotic order of $L_{\ell,d,r}(\bp)$ through the multivariate Laplace approximation.
Assume that $\bp$ lies in the interior of $\mathfrak S_{d-1}$, so that $p_j>0$ for all $j$.
Let us write
$$
S(\bp,\bm\theta):=\sum_{j=1}^d p_j e^{\mathrm i\theta_j},
\qquad \bm\theta\in[-\pi,\pi]^d,
$$
so that
$$
L_{\ell,d,r}(\bp)
=
\int_{[-\pi,\pi]^d} |S(\bp,\bm\theta)|^{2\ell}
\prod_{j=1}^d \mu_{\sqrt r}(\theta_j)d\theta_j.
$$
Since $|S(\bp,\bm\theta)|\le \sum_{j=1}^d p_j = 1$ with equality \textit{if and only if} $\theta_1=\cdots=\theta_d\ (\mathrm{mod}\ 2\pi)$, the integrand is maximized
precisely along the diagonal $\{\theta_1=\cdots=\theta_d\}$.
To remove this global phase invariance, we fix the differences
$$
\varphi_j:=\theta_{j+1}-\theta_1,\qquad j=1,\dots,d-1,
$$
and write $\varphi=(\varphi_1,\dots,\varphi_{d-1})\in[-\pi,\pi]^{d-1}$.
Then it holds that 
$$
|S(\bp,\bm\theta)|
=
\left|p_1 + \sum_{j=2}^d p_j e^{\mathrm i\varphi_{j-1}}\right|
=: \exp \Big(\frac12g(\varphi)\Big),
$$
where we define
$$
g(\varphi)
:=
\log \left|p_1 + \sum_{j=2}^d p_j e^{\mathrm i\varphi_{j-1}}\right|^{2},
\qquad \varphi\in(-\pi,\pi)^{d-1}.
$$

Moreover, the Poisson-kernel density $\mu_{\sqrt r}$ is continuous and strictly
positive on $[-\pi,\pi]$, and thus the product density
$$
h(\varphi)
:=
\int_{-\pi}^{\pi}
\prod_{j=1}^d \mu_{\sqrt r}(\theta_j)d\theta_1
\quad\text{(with }\theta_{j+1}=\theta_1+\varphi_j\text{ for }j\ge1\text{)}
$$
is continuous and satisfies that $h(0)>0$.
Consequently, $L_{\ell,d,r}(\bp)$ can be written in the form
\begin{equation}
\label{eq:Lell-laplace-form}
L_{\ell,d,r}(\bp)
=
\int_{D} \exp \big(\ell g(\varphi)\big)h(\varphi)d\varphi,
\qquad D:=(-\pi,\pi)^{d-1}.
\end{equation}

We now verify the conditions of Lemma~\ref{lem:laplace-multi} for \eqref{eq:Lell-laplace-form}
with dimensionality $m=d-1$ and maximizer $\varphi_0=0$.
Note that $\varphi=0$ is the unique maximizer of $g$ on $D$ since $|p_1+\sum_{j=2}^d p_j e^{\mathrm i\varphi_{j-1}}|\le 1$ with equality \textit{if and only if} all
$\varphi_j=0$.
A direct Taylor expansion of function $e^{\mathrm i u}$ at $u=0$ then yields that 
$$
p_1 + \sum_{j=2}^d p_j e^{\mathrm i\varphi_{j-1}}
=
1 + \mathrm i\sum_{j=2}^d p_j \varphi_{j-1}
-\frac12\sum_{j=2}^d p_j \varphi_{j-1}^2 + O(\|\varphi\|_2^3),
$$
and thus 
$$
\left|p_1 + \sum_{j=2}^d p_j e^{\mathrm i\varphi_{j-1}}\right|^{2}
=
1
-\sum_{j=2}^d p_j \varphi_{j-1}^2
+\left(\sum_{j=2}^d p_j \varphi_{j-1}\right)^{2}
+O(\|\varphi\|_2^3).
$$
Since all $p_j>0$ and $\sum_{j=2}^d p_j \varphi_{j-1}^2-(\sum_{j=2}^d p_j \varphi_{j-1})^2
= \mathrm{Var}_{J\sim\bp}(\varphi_{J-1})$ is a strictly positive quadratic form on
$\R^{d-1}$, it follows that $g$ is twice differentiable on $D$ with
$g(0)=0$, $\nabla g(0)=0$, and $\mathrm{Hess}g(0)$ negative. 
Consequently, we can apply Lemma~\ref{lem:laplace-multi} to \eqref{eq:Lell-laplace-form} and obtain that 
$$
L_{\ell,d,r}(\bp)
 \asymp 
\ell^{-(d-1)/2},
\qquad \ell\to\infty
$$
with constants depending on $(\bp,d,r)$. Thus, we establish the stated order in Equation~\eqref{eq:min-multinomial-functional}. This concludes the proof of Proposition~\ref{prop:poisson-functionals}.

\subsection{Proof of Proposition \ref{prop:counterexample}}
The proof of this proposition will compute the terminal law of greedy CART and test it against the marginal-value Bellman condition. The key ingredient is to identify a positive-probability event on which greedy moves toward the balanced terminal state even though the forest-level marginal cost prefers an imbalanced alternative. A small perturbation on that event then gives a strict local descent.

The configuration in this proposition has $\gamma=2/3$, so that $m=\lceil \gamma d\rceil=4$. Let $\pi_{\rm gr}$ be the feature-subsampled population CART policy with uniform
tie-breaking among maximizers. Denote by 
$
Z_2:=(N_{2,1},N_{2,2})
$
the informative terminal split-count vector at depth $\ell = 2$. Note that, in the zero-noise case with $\sigma_0^2 = 0$, the leading MSE terms depend only on the informative count vector $Z_2$.
For a probability law $\eta$ on
$$
\mathcal Z_{\le 2}:=\{z\in \mathbb Z_+^2:z_1+z_2\le 2\},
$$
we define the normalized zero-noise leading forest-MSE objective
$$
\mathcal J_B(\eta)
:=
\frac1B\,\mathbb E_{\eta}\{\phi(Z)\}
+
\frac{B-1}{B}\,\mathbb E_{\eta\otimes \eta}\{\psi_C(Z,Z')\},
$$
where $Z,Z'\stackrel{\rm iid}{\sim}\eta$, and
$$
\phi(z):=2^{-2z_1}+2^{-2z_2}, \quad
\psi_C(z,z')
:=
2^{-2\max\{z_1,z_1'\}}
+
2^{-2\max\{z_2,z_2'\}}.
$$

Then the informative-opportunity probability is 
$$
q:=\mathbb P(U_t\cap S\neq \varnothing)=\frac{14}{15}<1
$$
and
$$
\mathbb P(U_t\cap S=\{1,2\})
=
\frac{\binom{4}{2}}{\binom{6}{4}}
=
\frac6{15}
=
\frac25.
$$
Since $\beta_1=\beta_2$, the population CART gain among informative coordinates is
maximized by an exposed informative coordinate with smallest current informative count.
Note that ties are broken uniformly. Hence, from $Z_0=(0,0)$, we have 
$$
\mathbb P_{\pi^{\rm gr}}(Z_1=(1,0))
=
\mathbb P_{\pi^{\rm gr}}(Z_1=(0,1))
= 
\frac4{15}+\frac12 \cdot \frac6{15}
=
\frac7{15}
$$
and
$$
\mathbb P_{\pi^{\rm gr}}(Z_1=(0,0))
=
\frac1{15}.
$$

A second conditioning step yields the greedy terminal law
$
\eta^{\rm gr}:=\mathcal L_{\pi^{\rm gr}}(Z_2)
$
on $\mathcal Z_{\le 2}$
$$
\begin{array}{c|cccccc}
z
& (0,0) & (1,0) & (0,1) & (2,0) & (0,2) & (1,1)\\
\hline
\eta^{\rm gr}(z)
& \frac1{225}
& \frac{14}{225}
& \frac{14}{225}
& \frac{28}{225}
& \frac{28}{225}
& \frac{140}{225}
\end{array}.
$$
Indeed, if $Z_1=(1,0)$, greedy moves to $(1,1)$ whenever coordinate $2$ is exposed,
and moves to $(2,0)$ only when coordinate $1$ is exposed but coordinate $2$ is not.
The case of $Z_1=(0,1)$ is symmetric.

Denote by $a=(2,0)$ and $b=(1,1)$.
For a terminal law $\eta$, we define the marginal terminal cost associated with
$\mathcal J_B$ as 
$$
\Gamma_\eta^B(z)
:=
\frac1B\,\phi(z)
+
2\frac{B-1}{B}\,
\mathbb E_{\eta}\{\psi_C(z,Z')\}.
$$
It holds that 
$$
\phi(a)-\phi(b)
=
\left(2^{-4}+1\right)-\left(2^{-2}+2^{-2}\right)
=
\frac{17}{16}-\frac12
=
\frac9{16}.
$$
We also have
$$
\mathbb E_{\eta^{\rm gr}}\{\psi_C(a,Z')-\psi_C(b,Z')\}
=
-\frac1{48}.
$$
In fact, noting that
$$
\begin{array}{c|cccccc}
z
& (0,0) & (1,0) & (0,1) & (2,0) & (0,2) & (1,1)\\
\hline
\psi_C(a,z)-\psi_C(b,z)
& \frac9{16}
& \frac9{16}
& -\frac3{16}
& \frac{12}{16}
& -\frac3{16}
& -\frac3{16}
\end{array},
$$
we immediately have 
$$
\begin{aligned}
&\mathbb E_{\eta^{\rm gr}}\{\psi_C(a,Z')-\psi_C(b,Z')\}\\
&\quad =
\frac{1}{225\cdot 16}
\left[
9+14\cdot 9+14(-3)+28\cdot 12+28(-3)+140(-3)
\right]\\
&\quad =
-\frac{75}{3600}
=
-\frac1{48}.
\end{aligned}
$$
and thus
$$
\begin{aligned}
\Gamma_{\eta^{\rm gr}}^B(a)-\Gamma_{\eta^{\rm gr}}^B(b)
&=
\frac1B\cdot \frac9{16}
+
2\frac{B-1}{B}\left(-\frac1{48}\right)\\
&=
\frac{29-2B}{48B}.
\end{aligned}
$$
For each integer $B\ge 15$, we can obtain that 
$
\frac{29-2B}{48B}<0,
$ 
so it holds that 
$
\Gamma_{\eta^{\rm gr}}^B(a)<\Gamma_{\eta^{\rm gr}}^B(b).
$

Let us now consider the positive-probability event
$$
E:=\{Z_1=(1,0),\ U_2\cap S=\{1,2\}\}.
$$
In light of the independence of the second mask from the past, it holds that 
$$
\mathbb P_{\pi^{\rm gr}}(E)
=
\frac7{15}\cdot \frac6{15}
=
\frac{14}{75}>0.
$$
On event $E$, greedy CART chooses coordinate $2$ since coordinate $2$, leading to terminal informative count
$
b=(1,1).
$
The alternative coordinate $1$ is, feasible and leads to
$
a=(2,0).
$
Since $\ell=2$, this is the final decision epoch. Consequently, the relevant continuation costs are
exactly $\Gamma_{\eta^{\rm gr}}^B(a)$ and $\Gamma_{\eta^{\rm gr}}^B(b)$, and the feasible
alternative $a$ has strictly smaller marginal terminal cost.

To make the descent explicit, let us fix $\varepsilon\in(0,1)$, and define a policy
$\pi^\varepsilon$ that agrees with $\pi_{\rm gr}$ except on event $E$. On event $E$,
policy $\pi^\varepsilon$ chooses coordinate $1$ with probability $\varepsilon$, and coordinate $2$
with probability $1-\varepsilon$. Denote by 
$$
\theta_\varepsilon
:=
\varepsilon\,\mathbb P_{\pi^{\rm gr}}(E)
=
\varepsilon\frac{14}{75}.
$$
Then the induced informative terminal law is
$
\eta^\varepsilon
=
\eta^{\rm gr}
+
\theta_\varepsilon(\delta_a-\delta_b).
$
From the symmetry of $\psi_C$, we have the exact expansion
$$
\mathcal J_B(\eta^\varepsilon)-\mathcal J_B(\eta^{\rm gr})
=
\theta_\varepsilon
\left\{
\Gamma_{\eta^{\rm gr}}^B(a)-\Gamma_{\eta^{\rm gr}}^B(b)
\right\}
+
\theta_\varepsilon^2
\frac{B-1}{B}
\left[
\psi_C(a,a)-2\psi_C(a,b)+\psi_C(b,b)
\right].
$$
Moreover, it holds that
$$
\psi_C(a,a)=\frac{17}{16},\qquad
\psi_C(a,b)=\frac5{16},\qquad
\psi_C(b,b)=\frac12,
$$
and thus 
$$
\psi_C(a,a)-2\psi_C(a,b)+\psi_C(b,b)
=
\frac{15}{16}.
$$

Therefore, we can obtain that 
$$
\mathcal J_B(\eta^\varepsilon)-\mathcal J_B(\eta^{\rm gr})
=
\theta_\varepsilon\frac{29-2B}{48B}
+
\theta_\varepsilon^2\frac{B-1}{B}\frac{15}{16}.
$$
When $B\ge 15$, the linear coefficient is strictly negative, while the quadratic term is $O(\theta^2_\varepsilon)$. 
Thus, for all sufficiently small
$\varepsilon>0$,
$$
\mathcal J_B(\eta^\varepsilon)<\mathcal J_B(\eta^{\rm gr}).
$$
Consequently, the greedy terminal law is not a local minimizer. In particular, it cannot be a global
minimizer. 
This completes the proof of Proposition~\ref{prop:counterexample}.

\subsection{Proof of Lemma~\ref{lem:negdrift}} \label{proof.lem:negdrift}


The proof of this lemma will extract the negative drift from the possibility that an informative mask exposes more than one informative coordinate. When a low-count coordinate is exposed, the greedy rule is pulled toward it; when it is not exposed, averaging over the remaining candidates still yields a controllable upper bound. A geometric inequality then relates the minimum coordinate imbalance to $W_n$, converting such local preference into a linear drift bound. Assumption~\ref{ass:nondeg} is precisely the nondegeneracy condition that renders the drift constant $c_\star$ strictly positive. It rules out the degenerate case when an informative opportunity never exposes more than one informative coordinate, since then the greedy rule would have no chance to compare informative imbalances and the negative drift in Lemma~\ref{lem:negdrift} would disappear.

Denote by $T_n = \inf\{t\ge 1: M_t = n \}$, $\mathcal{A}_{n+1} = |U_{T_{n+1}} \cap S|$, and $K = K_{T_{n+1}} = |\mathcal{A}_{n+1}|$ the number of informative coordinates in the $(n+1)$th informative split. 
It is easy to see that $K \sim \text{Hypergeometric}(d,s,m)$ and is independent of the $\sigma$-algebra $\cF_n$ containing information from the previous $n$ informative splits. If $K=1$, the informative coordinate is chosen uniformly from $S$ and
$$
\E[\Delta_{n,J_{n+1}}\mid \cF_n,K=1]
=
\frac1s\sum_{j=1}^s\Delta_{n,j}
=
0.
$$

If $K = k \geq 2$, 
let $m_n=\min_j\Delta_{n,j}$ and $j^\star$ an index attaining this minimum.
Observe that conditional on $\cF_n$ and $K=k$, the candidate informative set $\mathcal{A}_{n+1}$ is a uniformly random
$k$-subset of $S$. Hence,
if $j^\star\in \mathcal{A}_{n+1}$, we have that by the greedy rule,
$$
\Delta_{n,J_{n+1}}=\min_{j\in \mathcal{A}_{n+1}} \Delta_{n,j}=m_n.
$$
If $j^\star\notin \mathcal{A}_{n+1}$, it holds that for any realization of $\mathcal A_{n+1}$,
$$
\min_{j\in \mathcal{A}_{n+1}}\Delta_{n,j}
\le
\frac1k\sum_{j\in \mathcal{A}_{n+1}}\Delta_{n,j}.
$$
Taking the conditional expectation over $\mathcal{A}_{n+1}$ given $K=k$ and $j^\star\notin \mathcal{A}_{n+1}$,
and using the uniformity of $\mathcal{A}_{n+1}$ over subsets of $S\setminus\{j^\star\}$, we can deduce that 
$$
\E \left[\min_{j\in \mathcal{A}_{n+1}}\Delta_{n,j}\mid \cF_n,K=k,j^\star\notin \mathcal{A}_{n+1}\right]
\le 
\frac{1}{s-1}\sum_{j\neq j^\star}\Delta_{n,j}
=
\frac{-m_n}{s-1}.
$$
Combining the above two cases yields that for any $k\geq 2$,
\begin{equation}\label{eq:cond-exp-k-2}
\E[\Delta_{n,J_{n+1}}\mid \cF_n,K=k]
\le
\frac{k}{s}m_n + \Bigl(1-\frac{k}{s}\Bigr)\frac{-m_n}{s-1}
=
m_n\frac{k-1}{s-1}.
\end{equation}

Since all coordinates of $\Delta_{n}$ satisfy that $\Delta_{n,j}\ge m_n$ and their sum is zero,
there must exist at least one coordinate whose value is greater than $-m_n/(s-1)$. In other words, it holds that 
$$
\max_j |\Delta_{n,j}| \ge -\frac{m_n}{s-1}.
$$
Meanwhile, let us consider the extreme configuration achieving the maximal variance under
these constraints: one coordinate is equal to $m_n$ and the remaining $s-1$
coordinates are equal to $-m_n/(s-1)$.
For such case, we have that 
$$
V_n=\sum_{j=1}^s \Delta_{n,j}^2
=
m_n^2 + (s-1)\Bigl(\frac{-m_n}{s-1}\Bigr)^2
=
m_n^2\Bigl(1+\frac{1}{s-1}\Bigr)
=
\frac{s}{s-1}m_n^2.
$$
Rearranging terms leads to 
$$
m_n^2 \ge \frac{s-1}{s}V_n
\qquad\text{or}\qquad
m_n\le -\sqrt{\frac{V_n}{s(s-1)}}.
$$
Hence, combining this with (\ref{eq:cond-exp-k-2}), we can obtain that for $k\ge2$,
$$
\E[\Delta_{n,J_{n+1}}\mid \cF_n,K=k]
\le
-\frac{k-1}{s(s-1)^{3/2}}\sqrt{V_n}.
$$

Further, averaging over $K$ conditional on $K\ge1$ gives that 
$$
\E[\Delta_{n,J_{n+1}}\mid \cF_n]
\le
-c_\star\sqrt{V_n}
$$
with
\begin{align*}
c_\star
&=
\frac{1}{s(s-1)^{3/2}}
\sum_{k=2}^s (k-1)\P(K=k\mid K\ge1)\\
&=
\frac{1}{s(s-1)^{3/2}}
\sum_{k=1}^s (k-1)_+\P(K=k\mid K\ge1)\\
&=
\frac{1}{s(s-1)^{3/2}}
\E \big[(K-1)_+\mid K\ge1\big],
\end{align*}
The positivity of $c_\star$ follows from Assumption~\ref{ass:nondeg}.
This completes the proof of Lemma~\ref{lem:negdrift}.

\subsection{Proof of Lemma~\ref{lem:cost-to-go-bellman}}
The proof of this lemma will involve a standard backward-induction argument, but the random mask will require one additional conditioning step. After the mask is observed, the current decision is deterministic over a finite action set. Before it is observed, the value is the average over all mask realizations.

Specifially, we aim to establish the result by backward induction on $r$.
At terminal time $r=\ell$, no decisions remain. If $N_\ell=\mathbf n$, the
terminal cost is already fixed and equal to  $h(\mathbf n)$. Hence, it holds that 
$$
\inf
\mathbb E_\pi[h(N_\ell)\mid N_\ell= \mathbf n]
=
h(\mathbf n)
=
\mathcal V_\ell^h(\mathbf n).
$$
This shows the base case.

Now assume that the claim holds at time $r+1$. Let us fix $n\in\mathcal N_{r,d}$.
Conditional on $N_r=\mathbf n$, the next mask $U_{r+1}$ has law $p_m$ on
$\mathcal U_m$, and is independent of the past. After observing
$U_{r+1}=u$, an admissible policy must choose an action
$a\in\mathcal A(u)$. If action $a$ is chosen, the next state is
$
N_{r+1}=\mathbf n+e_a.
$
In view of the induction hypothesis, the minimum expected terminal cost from this next
state is $\mathcal V_{r+1}^h(\mathbf n+e_a)$.
Thus, after observing $u$, the smallest achievable continuation cost is given by 
$$
\min_{a\in\mathcal A(u)}
\mathcal V_{r+1}^h(\mathbf n+e_a).
$$
Averaging over the next mask then gives that 
$$
\inf_{\pi_{(r+1):\ell}}
\mathbb E_\pi[h(N_\ell)\mid N_r=\mathbf n]
=
\sum_{u\in\mathcal U_m}p_m(u)
\min_{a\in\mathcal A(u)}
\mathcal V_{r+1}^h(\mathbf n+e_a)
=
\mathcal V_r^h(\mathbf n).
$$
The minimum is attainable since all state and action sets are finite, so a
minimizing selector exists for each pair $(\mathbf n,u)$. This completes the
backward induction.

Finally, 
taking $r=0$ and then $N_0 = \mathbf 0$, the above established result at time $0$ yields that 
$$
\mathcal V_{0}^h(\mathbf 0) = \inf_{\pi\in \Pi_{\ell}}\E_{\pi}[h(N_{\ell})],
$$
which gives the final claim. This concludes the proof of Lemma~\ref{lem:cost-to-go-bellman}.

\subsection{Proof of Lemma~\ref{lem:linear-terminal-bellman}}
The proof of this lemma will rely on a submartingale argument to show that an optimal policy cannot put positive mass on nonminimizing actions without creating a strict increase in the expected cost-to-go.



Let
$\pi\in\Pi_\ell$ be arbitrary, and define
$
M_t^h:=\mathcal V_t^h(N_t),
\qquad t=0,\ldots,\ell.
$
For $t=0,\ldots,\ell-1$, it holds that 
\begin{align*}
\mathbb E_\pi[M_{t+1}^h\mid \mathcal H_t,U_{t+1}]
&=
\sum_{j\in\mathcal A(U_{t+1})}
\pi_{t+1}(j\mid \mathcal H_t,U_{t+1})
\mathcal V_{t+1}^h(N_t+e_j)\ge
\min_{j\in\mathcal A(U_{t+1})}
\mathcal V_{t+1}^h(N_t+e_j).
\end{align*}
Notice that the right-hand side of the expression above is a convex combination of $\left\{\mathcal V_{t+1}^h(N_t + e_j): j \in \mathcal A(U_{t+1})\right\}$, and thus is lower bounded by the minimum.
Taking the conditional expectation given $\mathcal H_t$, and using the independence and
uniform law of $U_{t+1}$, we can deduce that 
$$
\mathbb E_\pi[M_{t+1}^h\mid \mathcal H_t]
\ge
\sum_{u\in\mathcal U_m}p_m(u)
\min_{j\in\mathcal A(u)}
\mathcal V_{t+1}^h(N_t+e_j)
=
\mathcal V_t^h(N_t)
=
M_t^h.
$$
Hence, $(M_t^h)_{t=0}^\ell$ is a submartingale under every policy
$\pi\in\Pi_\ell$.

Let $\pi^\star$ be optimal for the linear terminal-cost problem with cost
$h$. 
We define the nonnegative random variable
\begin{align*}
    D_t^h
&:=
\sum_{j\in\mathcal A(U_{t+1})}
\pi^\star_{t+1}(j\mid H_t,U_{t+1})
\mathcal V_{t+1}^h(N_t+e_j)
-
\min_{j\in\mathcal A(U_{t+1})}
\mathcal V_{t+1}^h(N_t+e_j) \\
& = \E_{\pi^*}[M_{t+1}^h|\mathcal H_t, U_{t+1}] - \min_{j\in\mathcal A(U_{t+1})}
\mathcal V_{t+1}^h(N_t+e_j).
\end{align*}
By construction, it holds that $D_t^h \geq 0$. Moreover, we have that 
\[
\E_{\pi^\star}[D_t^h \mid \mathcal H_t] = \E_{\pi^\star}[M_{t+1}^h \mid \mathcal H_t] - M_t^h.
\]
Taking expectations gives that 
$$
\mathbb E_{\pi^\star}[D_t^h]
=
\mathbb E_{\pi^\star}[M_{t+1}^h-M_t^h].
$$
Since $\pi^\star$ is optimal for the linear terminal-cost problem $\inf_{\pi\in \Pi_{\ell}}\E_{\pi}[h(N_{\ell})]$, an application of Lemma \ref{lem:cost-to-go-bellman} leads to 
$$
\mathbb E_{\pi^\star}[M_\ell^h]
=
\mathbb E_{\pi^\star}[h(N_\ell)]
=
\mathcal V_0^h(0)
=
M_0^h.
$$

Further, $(M_t^h)_{t=0}^\ell$ is a submartingale under every admissible
policy, so it holds that 
$$
\mathbb E_{\pi^\star}[M_{t+1}^h-M_t^h]\ge0,
\qquad t=0,\ldots,\ell-1.
$$
Summing over $t$ yields that 
$$
0
=
\mathbb E_{\pi^\star}[M_\ell^h-M_0^h]
=
\sum_{t=0}^{\ell-1}
\mathbb E_{\pi^\star}[M_{t+1}^h-M_t^h].
$$
Since each term in the sum is nonnegative, each term must be zero
$$
\mathbb E_{\pi^\star}[M_{t+1}^h-M_t^h]=0,
\qquad t=0,\ldots,\ell-1.
$$
Consequently, it follows that 
$$
\mathbb E_{\pi^\star}[D_t^h]=0,
\qquad t=0,\ldots,\ell-1.
$$ 
Since
$D_t^h\ge0$, we have that $D_t^h=0$, $\P_{\pi^\star}$-almost surely; that is, it holds that 
\begin{align*}
\mathbb E_{\pi^\star}
\left[
M_{t+1}^h
\mid H_t,U_{t+1}
\right]
=
\min_{j\in\mathcal A(U_{t+1})}
\mathcal V_{t+1}^h(N_t+e_j).
\end{align*}

Note that a convex
combination of finitely many real numbers is equal to their minimum if and only if
all values receiving positive weight are equal to that minimum. Therefore, we have 
$$
\operatorname{supp}
\pi^\star_{t+1}(\cdot\mid H_t,U_{t+1})
\subseteq
\arg\min_{j\in\mathcal A(U_{t+1})}
\mathcal V_{t+1}^h(N_t+e_j),
\qquad
\P_{\pi^\star}\text{-a.s.}
$$
This completes the proof of Lemma~\ref{lem:linear-terminal-bellman}.

\subsection{Proof of Lemma \ref{lem:compact-convex}}



We show that $\mathfrak V_\ell$ is a linear projection of a compact convex polytope $\mathcal P_\ell$ and thus is also a compact convex polytope. 

We first define the set $\mathcal P_\ell$. 
For $t=0,\ldots,\ell$ and $\mathbf n\in\mathcal N_{t,d}$, we define $
x_t(n)=\mathbb P(N_t=\mathbf n)$.
For $t=0,\ldots,\ell-1$, $\mathbf n\in\mathcal N_{t,d}$, $u\in\mathcal U_m$, and
$j\in\mathcal A(u)$, we further define
\[
y_t(n,u,j)
=
\mathbb P(N_t=\mathbf n,\ U_{t+1}=u,\ J_{t+1}=j).
\]
The set $\mathcal P_\ell$ is defined as the collection of all nonnegative arrays $\{(x_t),(y_t)\}$ satisfying that 
\begin{subequations}
\begin{align}
x_0(\mathbf 0)&=1, \label{eq:occ-initial-clean}\\
\sum_{j\in\mathcal A(u)} y_t(\mathbf n,u,j)
&=
p_m(u)x_t(\mathbf n),
&& t=0,\ldots,\ell-1,\quad \mathbf n\in\mathcal N_{t,d},\quad u\in\mathcal U_m,
\label{eq:occ-mask-clean}\\
x_{t+1}(\mathbf n')
&=
\sum_{\substack{
\mathbf n\in\mathcal N_{t,d},\ u\in\mathcal U_m,\ j\in\mathcal A(u):\\
\mathbf n+e_j=\mathbf n'}}
y_t(\mathbf n,u,j),
&& t=0,\ldots,\ell-1,\quad \mathbf n'\in\mathcal N_{t+1,d},
\label{eq:occ-flow-clean}\\
x_t(\mathbf n)&\ge 0,\qquad y_t(\mathbf n,u,j)\ge 0 .
\label{eq:occ-nonnegative-clean}
\end{align}
\end{subequations}

Next, we prove that the set of reachable terminal laws is exactly
\[
\mathfrak V_\ell
=
\left\{
\nu:\ \nu(\mathbf n)=x_\ell(\mathbf n),\ \mathbf n\in\mathcal N_{\ell,d},
\text{ for some }\{(x_t),(y_t)\}\in\mathcal P_\ell
\right\}.
\]
Thus, $\mathfrak V_\ell$ is the linear projection of $\mathcal P_\ell$ onto
the terminal state coordinates.

On the one hand, it is clear that every admissible policy generates a point of
$\mathcal P_\ell$. In fact, the initial condition \eqref{eq:occ-initial-clean} follows naturally
from $N_0=\mathbf 0$. 
Since $U_{t+1}$ is independent of the past and has law
$p_m$, we have that for each $\mathbf n$ and $u$,
\[
\sum_{j\in\mathcal A(u)} y_t(\mathbf n,u,j)
=
\mathbb P(N_t=\mathbf n,\ U_{t+1}=u)
=
p_m(u)x_t(\mathbf n),
\]
which gives \eqref{eq:occ-mask-clean}. Moreover, since the state transition is
deterministic and $N_{t+1}=N_t+e_{J_{t+1}}$, summing the joint probabilities of
all predecessor triples $(\mathbf n,u,j)$ satisfying that $\mathbf n+e_j=\mathbf n'$ yields \eqref{eq:occ-flow-clean}. The nonnegativity is immediate.

On the other hand, for any $\{(x_t),(y_t)\}\in\mathcal P_\ell$, there exists a Markov
randomized policy $\pi$ to realize $\{(x_t),(y_t)\}$. Indeed, for
$t=0,\ldots,\ell-1$, define
\[
\pi_{t+1}(j\mid \mathbf n,u)
=
\frac{y_t(\mathbf n,u,j)}{p_m(u)x_t(\mathbf n)},
\qquad j\in\mathcal A(u),
\]
whenever $p_m(u)x_t(\mathbf n)>0$. If $p_m(u)x_t(\mathbf n)=0$, let us define
$\pi_{t+1}(\cdot\mid \mathbf n,u)$ as any probability distribution on
$\mathcal A(u)$, which does not affect the induced law.
The constraint \eqref{eq:occ-mask-clean} guarantees that
$\pi_{t+1}(\cdot\mid \mathbf n,u)$ is a probability distribution whenever
$p_m(u)x_t(n)>0$. 

We verify by induction that this policy $\pi$ induces $\{(x_t),(y_t)\}$. 
At $t=0$, the conclusion naturally follows from
\eqref{eq:occ-initial-clean}. Assume that the claim holds at time $t$, so that
$\mathbb P(N_t=\mathbf n)=x_t(\mathbf n)$. If $p_m(u)x_t(\mathbf n)>0$, it holds that 
\[
\mathbb P(N_t=\mathbf n,\ U_{t+1}=u,\ J_{t+1}=j)
=
x_t(\mathbf n)p_m(u)\pi_{t+1}(j\mid \mathbf n,u)
=
y_t(\mathbf n,u,j).
\]
If $p_m(u)x_t(\mathbf n)=0$, both sides are zero, since \eqref{eq:occ-mask-clean} and nonnegativity imply that $y_t(\mathbf n,u,j)=0$ for all
$j\in\mathcal A(u)$. 
Applying \eqref{eq:occ-flow-clean} then
gives that $\mathbb P(N_{t+1}=\mathbf n')=x_{t+1}(\mathbf n')$. The induction is now complete.

It remains to show that $\mathcal P_\ell$ is a compact convex polytope. To see this, we notice that
the convexity is due to the fact that all constraints are linear. The set is closed since it is defined by finitely many linear equalities and inequalities. 
It is also
bounded; summing \eqref{eq:occ-flow-clean} over $\mathbf n'$ and using
\eqref{eq:occ-mask-clean}, we can show by induction that
\[
\sum_{\mathbf n\in\mathcal N_{t,d}} x_t(\mathbf n)=1,
\qquad t=0,\ldots,\ell.
\]
Hence, it follows that $0\le x_t(\mathbf n)\le 1$. Moreover, we have that 
\[
0\le y_t(\mathbf n,u,j)
\le
\sum_{a\in\mathcal A(u)}y_t(\mathbf n,u,a)
=
p_m(u)x_t(\mathbf n)
\le 1.
\]
Consequently, $\mathcal P_\ell$ is indeed a compact
convex polytope. 
This concludes the proof of Lemma \ref{lem:compact-convex}.





\section{Auxiliary lemmas with proofs}
\label{sec:auxiliary}

\subsection{Auxiliary lemmas}
\label{app:aux}

The following 
identities and concentration inequalities will be used repeatedly in the proof of Proposition~\ref{prop:poisson-functionals}.

\begin{lemma}[Poisson-kernel Fourier coefficients]
\label{lem:poisson-fourier}
Let $r\in(0,1)$ and define the Poisson kernel
$$
P_r(\theta)=\frac{1-r^2}{1-2r\cos\theta+r^2},\qquad \theta\in[-\pi,\pi],
$$
along with the probability density $\mu_r(\theta):=P_r(\theta)/(2\pi)$ on $[-\pi,\pi]$.
Then we have that for each $k\in\mathbb Z$,
$$
\int_{-\pi}^{\pi} e^{\mathrm i k \theta}\mu_r(\theta)d\theta = r^{|k|}.
$$
Equivalently, it holds that for any integer-valued random variable $X$,
$$
\mathbb E\big[r^{|X|}\big]
=
\int_{-\pi}^{\pi}\mathbb E\big[e^{\mathrm i\theta X}\big]\mu_r(\theta)d\theta,
$$
whenever the expectation exists.
\end{lemma}

Lemma~\ref{lem:poisson-fourier} above is the Fourier tool that transforms the absolute-value penalties into integrals of the characteristic functions.

\begin{lemma}[Binomial generating function]
\label{lem:binom-pgf}
Let $N\sim\mathrm{Binomial}(\ell,p)$ with $\ell\in\mathbb N$ and $p\in[0,1]$.
Then we have that for any $z\in\mathbb C$ (the complex plane),
$$
\mathbb E[z^N]=(1-p+pz)^\ell.
$$
\end{lemma}

Lemma~\ref{lem:binom-pgf} above provides the exact binomial factors that appear in the single-tree and cross-tree bias terms. In addition to Lemmas~\ref{lem:poisson-fourier} and \ref{lem:binom-pgf}, we will
exploit the 
representations in the following two lemmas.

\begin{lemma}[Multinomial characteristic function]
\label{lem:mult-cf}
Let $N=(N_1,\dots,N_d)\sim \mathrm{Multinomial}(\ell,\bp)$ with
$\bp=(p_1,\dots,p_d)\in\mathfrak S_{d-1}$ and $\ell\in\mathbb N$.
Then we have that for each $\theta=(\theta_1,\dots,\theta_d)\in\R^d$,
$$
\mathbb E \left[\exp \big(\mathrm i\langle\theta,N\rangle\big)\right]
=
\left(\sum_{j=1}^d p_j e^{\mathrm i\theta_j}\right)^\ell.
$$
\end{lemma}

Lemma~\ref{lem:mult-cf} above is the multinomial analog of Lemma~\ref{lem:binom-pgf}, and employed after the overlap functional is rewritten as an $\ell_1$-distance.

\begin{lemma}[Product Poisson-kernel representation]
\label{lem:l1-poisson}
Let $\rho\in(0,1)$ and $X=(X_1,\dots,X_d)$ a $\mathbb Z^d$-valued random vector.
Define $\|X\|_1=\sum_{j=1}^d |X_j|$.
Then we have that 
$$
\mathbb E \left[\rho^{\|X\|_1}\right]
=
\int_{[-\pi,\pi]^d}
\mathbb E \left[e^{\mathrm i\langle\theta,X\rangle}\right]
\prod_{j=1}^d \mu_\rho(\theta_j)d\theta_j,
$$
whenever the expectation exists.
\end{lemma}

Lemma~\ref{lem:l1-poisson} above reduces the product $\ell_1$-damping to a multidimensional Fourier integral, which is the starting point for the Poisson-kernel attenuation functional $L_{\ell,d,r}$.

\begin{lemma}[Multivariate Laplace approximation]\label{lem:laplace-multi}
Let $m\ge1$ and $D\subset\R^m$ an open region.
Let $g:D\to\R$ be three times continuously differentiable, and assume that 
\begin{itemize}
\item[1)] $g$ attains its unique global maximum at an interior point $x_0\in D$,
\item[2)] $\nabla g(x_0)=0$,
\item[3)] $\mathrm{Hess}g(x_0)$ is negative definite.
\end{itemize}
Let $h:D\to\R$ be continuous with $h(x_0)>0$.
Then we have that as $\ell\to\infty$,
\begin{equation}\label{eq:laplace-multi}
\int_D e^{\ell g(x)} h(x)dx
\sim
e^{\ell g(x_0)}
\left(\frac{2\pi}{\ell}\right)^{m/2}
\frac{h(x_0)}{\sqrt{\det \big(-\mathrm{Hess}g(x_0)\big)}}.
\end{equation}
In particular, it holds that 
$$
\int_D e^{\ell g(x)} h(x)dx
 \asymp 
e^{\ell g(x_0)}\ell^{-m/2}.
$$
\end{lemma}

Lemma~\ref{lem:laplace-multi} above converts the local quadratic curvature at the unique maximizer into the polynomial-order attenuation factors used in Proposition~\ref{prop:poisson-functionals}. Its technical conditions ensure that the integral is governed by a single interior quadratic peak. Specifically, the uniqueness excludes competing exponential contributions, the vanishing gradient and negative definite Hessian yield the Gaussian local form, and the condition of $h(x_0)>0$ prevents the leading prefactor from disappearing. This lemma follows from Proposition~3.2.1 together with Equations~(3.16) and~(3.23) in Section~3.2 of \cite{Butler2007}.




\subsection{Proofs of auxiliary lemmas} \label{proof.lem:coarsen-independence}

\noindent \textit{Proof of Lemma~\ref{lem:poisson-fourier}.}
The proof of this lemma will use the classical Fourier-series characterization of the Poisson kernel. Starting from the absolutely convergent expansion
\[
P_r(\theta)=1+2\sum_{m=1}^{\infty} r^m\cos(m\theta),
\qquad r\in(0,1),
\]
we can integrate term by term against $e^{\mathrm i k\theta}$ over $[-\pi,\pi]$. After division by $2\pi$, this leads to the stated coefficient identity for $\mu_r$. The probabilistic representation for $\mathbb E[r^{|X|}]$ then follows by conditioning on $X$ and applying Fubini's theorem. This completes the proof of Lemma~\ref{lem:poisson-fourier}.

\smallskip

\noindent \textit{Proof of Lemma~\ref{lem:binom-pgf}.}
The proof of this lemma will be a direct application of the binomial theorem. By definition, if $N\sim\mathrm{Binomial}(\ell,p)$, we can obtain that 
\[
\mathbb E[z^N]
=
\sum_{m=0}^{\ell} z^m \binom{\ell}{m}p^m(1-p)^{\ell-m}
=
\bigl((1-p)+pz\bigr)^\ell.
\]
This concludes the proof of Lemma~\ref{lem:binom-pgf}.

\smallskip

\noindent \textit{Proof of Lemma~\ref{lem:mult-cf}.}
The proof of this lemma will represent the multinomial vector as a sum of independent categorical increments and then factor the characteristic function. Let $(Y_1,\dots,Y_\ell)$ be i.i.d.\ categorical random variables taking values in $\{e_1,\dots,e_d\}$ with $\P(Y_1=e_j)=p_j$, so that $N=\sum_{t=1}^\ell Y_t$ has distribution $\mathrm{Multinomial}(\ell,\bp)$. By independence, it follows that 
$$
\mathbb E \left[e^{\mathrm i\langle\theta,N\rangle}\right]
=
\prod_{t=1}^\ell \mathbb E \left[e^{\mathrm i\langle\theta,Y_t\rangle}\right]
=
\left(\sum_{j=1}^d p_j e^{\mathrm i\theta_j}\right)^\ell,
$$ 
which yields the claimed characteristic function. This completes the proof of Lemma~\ref{lem:mult-cf}.

\smallskip

\noindent \textit{Proof of Lemma~\ref{lem:l1-poisson}.}
The proof of this lemma will apply the one-dimensional Poisson-kernel identity coordinatewise under the product measure. Specifically, since the coordinates are independent under the product measure, 
$$\prod_{j=1}^d \mu_r(\theta_j)d\theta_j$$ is a probability measure on $[-\pi,\pi]^d$. Conditioning on $X$ and applying Lemma~\ref{lem:poisson-fourier} to each coordinate, we can deduce that 
$$
\int_{[-\pi,\pi]^d} e^{\mathrm i\langle\theta,X\rangle}\prod_{j=1}^d \mu_r(\theta_j)d\theta_j
=
\prod_{j=1}^d \int_{-\pi}^{\pi} e^{\mathrm i\theta_j X_j}\mu_r(\theta_j)d\theta_j
=
\prod_{j=1}^d r^{|X_j|}
=
r^{\|X\|_1}.
$$
Therefore, taking expectations and applying Fubini's theorem yield the desired identity. This concludes the proof of Lemma~\ref{lem:l1-poisson}.




\end{APPENDICES}

\end{document}